\documentclass[10pt,twoside]{article}
\usepackage[margin=1in]{geometry}

\usepackage[utf8]{inputenc} 
\usepackage[T1]{fontenc}    
\usepackage[hidelinks]{hyperref}       
\usepackage{url}            
\usepackage{booktabs}       
\usepackage{amsfonts}       
\usepackage{nicefrac}       
\usepackage{microtype}      
\usepackage{xcolor}         
\usepackage{natbib}
\usepackage{graphicx}
\usepackage{mathrsfs}
\usepackage{amsfonts, amsmath,dsfont,amssymb,amsthm,stmaryrd,bbm}
\usepackage{mathtools}
\usepackage{caption}
\usepackage{subcaption}
\usepackage{enumitem}
\usepackage{wrapfig}
\usepackage{algorithm}
\usepackage{algorithmic}
\usepackage{float}
\usepackage{xspace}

\usepackage[textsize=tiny]{todonotes}

\newcommand{\octal}{\textsc{OCTAL}\xspace}

 \newtheorem{conj}{Conjecture}
 \newtheorem{thm}{Theorem}
 \newtheorem*{thmu}{Theorem}
 \newtheorem*{lemma*}{Lemma}
 \newtheorem{rmk}{Remark}
 
 \newtheorem{lemma}[conj]{Lemma}
 \newtheorem{claim}{Claim}

 \newtheorem{defn}[conj]{Definition}
 \newtheorem{coro}{Corollary}
 \newtheorem*{coro*}{Corollary}

\newcommand{\f}[1]{\boldsymbol{#1}}
\newcommand{\bb}[1]{\mathbb{#1}}
\newcommand{\fl}[1]{\mathbf{#1}}
\newcommand{\ca}[1]{\mathcal{#1}}

\newcommand{\s}[1]{\mathsf{#1}}

\title{Online Low Rank Matrix Completion}

\author{~~~Prateek Jain\footnote{P.Jain is with Google Research at Bangalore, INDIA (email: \texttt{prajain@google.com}).} ~~~Soumyabrata~Pal\footnote{S. Pal is with Google Research at Bangalore, INDIA (email: \texttt{soumyabrata@google.com}).}}

\begin{document}

\maketitle

\begin{abstract}
  We study the problem of {\em online} low-rank matrix completion with $\s{M}$ users, $\s{N}$ items and $\s{T}$ rounds. In each round, the algorithm recommends one item per user, for which it gets a (noisy) reward sampled from a low-rank user-item preference matrix. The goal is to design a method with sub-linear regret (in $\s{T}$) and nearly optimal dependence on $\s{M}$ and $\s{N}$. The problem can be easily mapped to the standard multi-armed bandit problem where each item is an {\em independent} arm, but that leads to poor regret as the correlation between arms and users is not exploited. On the other hand, exploiting the low-rank structure of reward matrix is challenging due to non-convexity of the low-rank manifold. We first demonstrate that the low-rank structure can be exploited  using a simple  explore-then-commit (ETC) approach that ensures a regret of $O(\s{polylog} (\s{M}+\s{N}) \s{T}^{2/3})$. That is, roughly  only $\s{polylog} (\s{M}+\s{N})$ item recommendations are required per user to get a non-trivial solution. We then improve our result for the  rank-$1$ setting which in itself is quite challenging and encapsulates some of the key issues. Here, we propose  \textsc{OCTAL}
(Online Collaborative filTering using iterAtive user cLustering) that guarantees nearly optimal regret of $O(\s{polylog} (\s{M}+\s{N}) \s{T}^{1/2})$. \octal is based on a novel technique of clustering users that allows  iterative elimination of items and leads to a nearly optimal minimax rate. 
\end{abstract}

\section{Introduction}
Collaborative filtering based on low-rank matrix completion/factorization techniques are the cornerstone of most modern recommendation systems \citep{Koren:2008}. Such systems model the underlying user-item affinity matrix as a low-rank matrix, use the acquired user-item recommendation data to estimate the low-rank matrix and subsequently, use the matrix estimate to recommend items for each user. Several existing works study this {\em offline} setting  \citep{candes2009exact, Montanari12, jain2013low, chen2019noisy, abbe2020entrywise}. 
However, typical recommendation systems are naturally {\em online} and interactive -- they recommend items to users and need to adapt quickly based on users' feedback. The goal of such systems is to quickly identify each user's preferred set of items, so it is necessary to identify the best items for each user instead of estimating the entire affinity matrix. Moreover, items/users are routinely added to the system, so it should be able to quickly adapt to new items/users by using only a small amount of recommendation feedback.

In this work, we study this problem of the online recommendation system. In particular, we study the online version of low-rank matrix completion with the goal of identifying top few items for each user using say only logarithmic many exploratory recommendation rounds for each user. 
In each round (out of $\s{T}$ rounds) we predict one item (out of $\s{N}$ items) for each user (out of $\s{M}$ users) and obtain feedback/reward for each of the predictions -- e.g. did the users view the recommended movie. The goal is to design a method that has asymptotically similar reward to a method that can pick {\em best} items for each user. As mentioned earlier, we are specifically interested in the setting where $\s{T}\ll \s{N}$ i.e. the number of recommendation feedback rounds is much smaller than the total number of items. 

Moreover, we assume that the expected reward matrix is low-rank. That is, if $\fl{R}_{ij}^{(t)}$ is the reward obtained in the $t^{\s{th}}$ round for predicting item-$j$ for user-$i$, then $\bb{E}\fl{R}_{ij}^{(t)}=\fl{P}_{ij}$, where $\fl{P}\in \mathbb{R}^{\s{M}\times \s{N}}$ is a low-rank matrix. A similar low rank reward matrix setting has been studied in {\em online} multi-dimensional learning problems  \citep{katariya2017stochastic,kveton2017stochastic,trinh2020solving}. But in these problems, the goal is to find the matrix entry with the highest reward. Instead, our goal is to recommend good items to {\em all} users which is a significantly harder challenge. 
A trivial approach is to ignore the underlying low-rank structure and solve the problem using standard multi-arm bandit methods. That is, model each \textit{(user, item)} pair as an {\em arm}. Naturally, that would imply exploration of almost all the items for {\em each} user, which is also reflected in the regret bound (averaged over users) of $O(\sqrt{\s{NT}})$ (Remark \ref{rmk:begin}). That is, as expected, the regret bound is vacuous when the number of recommendation rounds $\s{T}\ll \s{N}$.

In contrast, most of the existing online learning techniques that  leverage structure amongst arms assume a parametric form for the reward function and require the reward function to be convex in the parameters  \citep{shalev2011online,bubeck2011introduction}. Thus, due to the non-convexity of the manifold of low-rank matrices, such techniques do not apply in our case. While there are  some exciting recent approaches for {\em non-convex online learning} \citep{agarwal2019learning, suggala2020online}, they do not apply to the above mentioned problem. 

\noindent \textbf{Our Techniques and Contributions:} We first present a method based on the explore-then-commit (ETC) approach (Algorithm \ref{algo:p1}). For the first few rounds, the algorithm runs a pure exploration strategy by sampling random items for each user. We use the data obtained by pure exploration to learn a good estimate of the underlying reward matrix $\fl{P}$; this result requires a slight  modification of the standard matrix completion result in \citet{chen2019noisy}. We then run the exploitation rounds based on the current estimate. In particular, for the remaining rounds, the algorithm commits to the arm with the highest estimated reward for each user. 
With the ETC algorithm, we achieve a regret bound of $O(\s{polylog} (\s{M}+\s{N}) \s{T}^{2/3})$ (Thm. \ref{thm:baseline}). This bound is able to get rid of the dependence on $\s{N}$, implying non-trivial guarantees even when $\s{T}\ll \s{N}$. That is, we require only $\s{polylog} (\s{M}+\s{N})$ exploratory recommendations per user. However, the dependence of the algorithm on $\s{T}$ is sub-optimal. 
To address this, we study the special but critical case of rank-one reward matrix. The rank-$1$ setting is itself technically challenging  \citep{katariya2017stochastic} and encapsulates many key issues.
We provide a  novel  algorithm \textsc{OCTAL}
 in Algorithm \ref{algo:phased_elim} and a modified version in Algorithm \ref{algo:phased_elim2} that achieves nearly optimal regret bound of $O(\s{polylog} (\s{M}+\s{N}) \s{T}^{1/2})$ (see Theorems \ref{thm:main1} and \ref{thm:main2}). The key insight is that in rank-one case, we need to cluster users based on their true latent representation to ensure low regret.

Our method OCTAL consists of multiple phases of exponentially increasing number of rounds.  Each phase refines the current estimate of relevant sub-matrices of the reward matrix using standard matrix completion techniques.  
Using the latest estimate, we jointly refine the cluster of users and the estimate of the best items for users in each cluster. We can show that the regret in each phase decreases very quickly and since the count of phases required to get the correct clustering is small, we obtain our desired regret bounds. Finally, we show that our method achieves a regret guarantee that scales as $\widetilde{O}(\s{T}^{1/2})$ with the number of rounds $\s{T}$. We also show that the dependence on $\s{T}$ is optimal (see Theorem \ref{thm:lb}). Below we summarize our main contributions ($\widetilde{O}(\cdot)$ hides logarithmic factors): 
\begin{itemize}[leftmargin=*,noitemsep, nolistsep]
    \item We formulate the online low rank matrix completion problem and define the appropriate notion of regret to study.
     We propose  Algorithm \ref{algo:p1} based on the explore-then-commit approach that suffers a regret of $\widetilde{O}(\s{T}^{2/3})$. (Thm. \ref{thm:baseline}) 
    \item We propose a novel algorithm \textsc{OCTAL}
(Online Collaborative filTering using iterAtive user cLustering) and a modified version (Alg. \ref{algo:phased_elim} and Alg. \ref{algo:phased_elim2} respectively) for the special case of rank-$1$ reward matrices and guarantee a regret of  $\widetilde{O}(\s{T}^{1/2})$. Importantly, our algorithms  provide non-trivial regret guarantees in the practical regime of $\s{M}\gg\s{T}$ and $\s{N}\gg\s{T}$ i.e when the number of users and items is much larger than the number of recommendation rounds. Moreover, OCTAL does not suffer an issue of large cold-start time (possibly large exploration period) as in the ETC algorithm.
\item We conducted detailed empirical study of our proposed algorithms (see Appendix \ref{sec:sims}) on synthetic and multiple real datasets, and demonstrate that our algorithms can achieve significantly lower regret than methods that do not use collaboration between users. Furthermore, we show that for rank-$1$ case, while it is critical to tune the exploration period in ETC (as a function of rounds and sub-optimality gaps) and is difficult in practice \citep{lattimore2020bandit}[Ch. 6], OCTAL still suffers lower regret without such side-information (see Figures~\ref{fig:average_regret_gap1} and \ref{fig:average_regret_exploration}).


\end{itemize}

\noindent \textbf{Technical Challenges:}  
For rank-$1$ case, i.e., $\fl{P}=\fl{uv}^{\s{T}}$, one can cluster users in two bins: ones with $\fl{u}_i\geq 0$ and ones with $\fl{u}_i<0$. Cluster-$2$ users dislike the best items for Cluster-$1$ users and vice-versa. Thus, we require algorithms that can learn this cluster structure and exploit the fact that within the same cluster, relative ranking of all items remain the same. This is a challenging information to exploit, and requires learning the latent structure.  Note that an algorithm that first attempts to cluster all users  and then optimize the regret will suffer the same $\s{T}^{2/3}$ dependence in the worst case as in the ETC algorithm; this is because the difficulty of clustering the users are widely different. Instead, our proposed  \textsc{OCTAL} algorithm (Algorithms \ref{algo:phased_elim} and \ref{algo:phased_elim2}) in each iteration/phase performs two tasks: i) it tries to   eliminate some of the items similar to standard phase elimination method \citep{lattimore2020bandit}[Ch 6, Ex 6.8] for users that are already clustered, ii) it simultaneously tries to grow the set of clustered users. For partial clustering, we first apply low rank matrix completion guarantees over carefully constructed reward sub-matrices each of which correspond to a cluster of users and a set of active items that have high rewards for all users in the same cluster, and then use the partial reward matrix for eliminating some of the items in each cluster.

\subsection{Related Works}

To the best of our knowledge, we provide first rigorous  online matrix completion algorithms. But, there are several closely related results/techniques in the literature which we briefly survey below. 

A very similar setting was considered in \citet{sen2017contextual} where the authors considered a multi-armed bandit problem with $\s{L}$ contexts and $\s{K}$ arms with context dependent reward distributions. The authors assumed that the $\s{L}\times\s{K}$ reward matrix is low rank and can be factorized into non-negative components which allowed them to use recovery guarantees from non-negative matrix factorization. Moreover, the authors only showed ETC algorithms that resulted in $\s{T}^{2/3}$ regret guarantees. Our techniques can be used to improve upon the existing guarantees in \citet{sen2017contextual} in two ways 1) Removing the assumption of the low rank components being non-negative as we use matrix completion with entry-wise error guarantees. 2) The dependence on $\s{T}$ can be improved from $\s{T}^{2/3}$ to $\s{T}^{1/2}$ when the reward matrix $\fl{P}$ is rank-$1$.

Multi-dimensional online decision making problems namely stochastic rank-1 matrix bandits was introduced in \citet{katariya2017stochastic,katariya2017bernoulli,trinh2020solving}. In their settings, at each round $t\in [\s{T}]$, the learning agent can choose  one row and one column and observes a reward corresponding to an entry of a rank-$1$ matrix. Here, the regret is defined in terms of the best \textit{(row ,column)} pair which corresponds to the best arm. This setting was extended to the rank $r$ setting \citep{kveton2017stochastic}, rank $1$ multi-dimensional tensors \citep{hao2020low}, bilinear bandits \citep{jun2019bilinear,huang2021optimal} and generalized linear bandits \citep{lu2021low}. Although these papers provide tight regret guarantees, they cannot be translated to our problem. This is because, we solve a significantly different problem with an underlying rank-$1$ reward matrix $\fl{P}$ where we need to minimize the regret for all users (rows of $\fl{P}$) jointly. Hence, it is essential to find the entries (columns) of $\fl{P}$ with large rewards for each user(row) of $\fl{P}$; contrast this with the multi-dimensional online learning problem where it is sufficient to infer only the entry ((row,column) pair) in the matrix/tensor with the highest reward. Since the rewards for each user have different gaps, the analysis becomes involved for our OCTAL algorithm. Finally, \citep{dadkhahi2018alternating,zhou2020stochastic} also consider our problem setting but they only provide heuristic algorithms without any theoretical guarantees.


Another closely related line of work is the theoretical model for User-based Collaborative Filtering (CF) studied in \citet{Bresler:2014,Bresler:2016,Heckel:2017,Mina2019,huleihel2021learning}. In particular, these papers were the first to motivate and theoretically analyze the collaborative framework with the restriction that the same item cannot be recommended more than once to the same user. 
Here a significantly stricter cluster structure assumption is made over users where users in same cluster have similar preferences. 
Such models are restrictive as they provide theoretical guarantees only on a very relaxed notion of regret (termed \textit{pseudo-regret}). 


In the past decade, several papers have studied the problem of  {\em offline} low rank matrix completion on its own    \citep{mazumder2010spectral,negahban2012restricted,chen2019noisy,Montanari12,abbe2020entrywise,jain2013low,jain2017non} and also in the presence of side information such as social graphs or similarity graphs \citep{xu2013speedup,ahn2018binary,ahn2021fundamental,elmahdy2020matrix,jo2021discrete,zhang2022mc2g}.
Some of these results namely the ones that provide $\|\cdot\|_{\infty}$ norm guarantees on the estimated matrix can be adapted into  Explore-Then-Commit (ETC) style algorithms (see Sec. \ref{sec:warmup}). 
Finally, there is significant amount of related theoretical work for online non-convex learning \citep{suggala2020online, yang2018optimal, huang2020online} and  empirical work for online Collaborative Filtering \citep{huang2020online,lu2013second,zhang2015simple} but they do not study the regret in online matrix completion setting.



\section{Problem Definition}\label{subsec:prob_defn}

\paragraph{Notations:} We write $[m]$ to denote the set $\{1,2,\dots,m\}$. For a vector $\fl{v}\in \bb{R}^m$, $\fl{v}_i$ denotes the $i^{\s{th}}$ element; for any set $\ca{U}\subseteq [m]$, let $\fl{v}_\ca{U}$ denote the vector $\fl{v}$ restricted to the indices in $\ca{U}$. $\fl{A}_i$ denotes the $i^{\s{th}}$ row of $\fl{A}$ and
$\fl{A}_{ij}$  denotes the $(i,j)$-th element of matrix $\fl{A}$. $[n]$ denotes the set $\{1,2,\dots,n\}$. For any set $\ca{U}\subset [m],\ca{V}\subset [n]$,  $\fl{A}_{\ca{U},\ca{V}}$ denotes the matrix $\fl{A}$ restricted to the rows in $\ca{U}$ and columns in $\ca{V}$. Also, let $\|\fl{A}\|_{2\rightarrow \infty}$ be the maximum $\ell_2$ norm of the rows of  $\fl{A}$ and $\|\fl{A}\|_{\infty}$ be the absolute value of the largest entry in $\fl{A}$. We write $\bb{E}X$ to denote the expectation of a random variable $X$.


Consider a system with a set of $\s{M}$ users and $\s{N}$ items. Let $\fl{P}=\fl{U}\fl{V}^{\s{T}}\in \bb{R}^{\s{M}\times \s{N}}$ be the unknown reward matrix of rank $r<\min(\s{M},\s{N})$ where $\fl{U}\in \bb{R}^{\s{M}\times r}$ and $\fl{V}\in \bb{R}^{\s{N}\times r}$ denote the latent embeddings corresponding to users and items respectively. In other words, we can denote $\fl{P}_{ij} \triangleq  \langle \fl{u}_i, \fl{v}_j \rangle$ where $\fl{u}_i, \fl{v}_j \in \bb{R}^r$ denotes the $r$-dimensional embeddings of $i$-th user and the $j$-th item, respectively. Often, we will also use the SVD decomposition of $\fl{P}=\fl{\bar{U}}\Sigma\fl{\bar{V}}$ where $\bar{\fl{U}}\in \bb{R}^{\s{M}\times r},\bar{\fl{V}}\in \bb{R}^{\s{N}\times r}$ are orthonormal matrices i.e. $\fl{\bar{U}}^{\s{T}}\fl{\bar{U}}=\fl{I}$ and $\fl{\bar{V}}^{\s{T}}\fl{\bar{V}}=\fl{I}$  
and $\f{\Sigma}\triangleq \s{diag}(\lambda_1,\lambda_2,\dots,\lambda_r) \in \bb{R}^{r\times r}$ is a diagonal matrix. We will denote the condition number of the matrix $\fl{P}$ by $\kappa\triangleq (\max_i \lambda_i)(\min_i \lambda_i)^{-1}$.


 Consider a system  that recommends one item to every user, in each round $t\in [\s{T}]$. Let, $\fl{R}^{(t)}_{u\rho_u(t)}$ be the reward for recommending item $\rho_u(t)\in [\s{N}]$ for user $u$. Also, let: \vspace*{-3pt}
\begin{align}\label{eq:obs}
    \fl{R}^{(t)}_{u\rho_u(t)} = \fl{P}_{u\rho_u(t)}+\fl{E}^{(t)}_{u\rho_u(t)}\vspace*{-3pt}
\end{align}
where $\fl{E}^{(t)}_{u\rho_u(t)}$ denotes the unbiased additive noise. Each element of  $\{\fl{E}^{(t)}_{u\rho_u(t)}\}_{\substack{u\in [\s{M}], t \in [\s{T}]}}$ is assumed to be i.i.d. zero mean sub-gaussian random variables with variance proxy $\sigma^2$. That is, $\bb{E}[\fl{E}^{(t)}_{u\rho_u(t)}]=0$ and  $\bb{E}[\exp(s\fl{E}^{(t)}_{u\rho_u(t)}) ]\le \exp(\sigma^2s^2/2)$ for all $u\in [\s{M}], t\in [\s{T}]$. The goal is to minimize the expected regret where the expectation is over randomness in rewards and the algorithm: 
\begin{align}
\s{Reg}(\s{T})\triangleq \frac{\s{T}}{\s{M}}\sum_{u \in [\s{M}]}\max_{j \in [\s{N}]} \fl{P}_{uj}- \bb{E}[\sum_{t\in[\s{T}]}\frac{1}{\s{M}}\sum_{u\in[\s{M}]}\mathbf{R}_{u\rho_u(t)}^{(t)}].\label{eqn:regretExpect}
\end{align}

In this problem, the interesting regime is  $(\s{N},\s{M})\gg \s{T}$ as is often the case for most  practical recommendation systems. Here, treating each user separately will lead to vacuous regret bounds as each item needs to be observed at least once by each user to find the best item for each user. However, low-rank structure of the rewards can help share information about items across users.  

\begin{rmk}\label{rmk:begin}
If $\s{T}\gg \s{N}$, then we can treat each user as a separate multi-armed bandit problem. In that case, in our setting, the well-studied Upper Confidence Bound (UCB) algorithm achieves an expected regret of at most $O(\sigma\sqrt{\s{NT}\log \s{T}})$ (Theorem 2.1 in \cite{bubeck2012regret}). 
\end{rmk}

\section{Preliminaries}\label{sec:Preliminaries}
\begin{algorithm}[!t]
\caption{\textsc{Estimate}   \label{algo:estimate}}
\begin{algorithmic}[1]
\REQUIRE Set of users $\ca{U}\subseteq[\s{M}]$, set of items $\ca{V}\subseteq[\s{N}]$, total  rounds $m$, set of indices $\Omega\subseteq \ca{U}\times \ca{V}$, rounds in each iteration $b=\max_{u \in \ca{U}} |v \in \ca{V}\mid (u,v) \in \Omega|$,  regularization parameter $\lambda$. Index of round $t$ is relative to the first round when the algorithm is invoked; hence $t=1,2,\dots,m$.

\FOR{$\ell=1,2,\dots,m/b$} 
\STATE For all $(i,j)\in \Omega$, set $\s{Mask}_{ij}=0$.
\FOR{$\ell'=1,2,\dots,b$}
\FOR{each user $u\in \ca{U}$ in round $t=(\ell-1)b+\ell'$}
\STATE Recommend an item $\rho_u(t)$ in $\{j \in \ca{V}\mid (u,j)\in \Omega, \s{Mask}_{uj}=0\}$ and set $\s{Mask}_{u\rho_u(t)}=1$. 
If not possible then recommend any item $\rho_u(t)$ in $\ca{V}$ s.t. $(u,\rho_u(t))\not\in\Omega$.
Observe $\fl{R}^{(t)}_{u\rho_u(t)}$.
\ENDFOR
\ENDFOR
\ENDFOR

\STATE For each $(u,j)\in \Omega$, compute $\fl{Z}_{uj}$ to be average of $\lfloor m/b \rfloor$ observations corresponding to user $u$ being recommended item $j$ i.e. $\fl{Z}_{uj} = \text{avg}\{\fl{R}^{(t)}_{u\rho_u(t)} \text{ for }t\in [m]\mid \rho_u(t)=j\}$. 
Discard all other observations corresponding to indices not in $\Omega$.

\STATE Without loss of generality, assume $|\ca{U}| \le |\ca{V}|$. For each $i\in \ca{V}$, independently set $\delta_i$ to be a value in the set $[\lceil|\ca{V}|/|\ca{U}|\rceil]$ uniformly at random. Partition indices in $\ca{V}$ into $\ca{V}^{(1)},\ca{V}^{(2)},\dots,\ca{V}^{(k)}$ where $k=\lceil|\ca{V}|/|\ca{U}|\rceil$ and $\ca{V}^{(q)}=\{i\in \ca{V}\mid \delta_i =q\}$ for each $q\in [k]$. Set $\Omega^{(q)}\leftarrow \Omega \cap (\ca{U}\times \ca{V}^{(q)})$ for all $q\in [k]$. \#\textit{If $|\ca{U}| \ge |\ca{V}|$, we partition the indices in $\ca{U}$}. 

\FOR{$q\in [k]$}
\STATE Solve convex program 
\vspace*{-15pt}
\begin{align}\label{eq:convex}
\tiny
    \min_{\fl{Q}^{(q)}\in \bb{R}^{|\ca{U}|\times|\ca{V}^{(q)}|}} \frac{1}{2}\sum_{(i,j)\in \Omega^{(q)}}\Big(\fl{Q}^{(q)}_{i\pi(j)}-\fl{Z}_{ij}\Big)^2+\lambda\|\fl{Q}^{(q)}\|_{\star},
\vspace*{-15pt}    
\end{align}
where $\|\fl{Q}^{(q)}\|_{\star}$ denotes nuclear norm of matrix $\fl{Q}^{(q)}$ and $\pi(j)$ is index of $j$ in set $\ca{V}^{(q)}$. 
\ENDFOR
\STATE Return  $\widetilde{\fl{Q}}\in \bb{R}^{\s{M}\times \s{N}}$ s.t. $\widetilde{\fl{Q}}_{\ca{U},\ca{V}^{(q)}}=\fl{Q}^{(q)}$ for all $q\in [k]$ and for every  $(i,j)\not \in \ca{U}\times \ca{V}$, $\widetilde{\fl{Q}}_{ij}=0$.
\end{algorithmic}
\end{algorithm}
Let us introduce a different observation model from (\ref{eq:obs}).
Consider an unknown rank $r$ matrix $\fl{P} \in \bb{R}^{\s{M}\times \s{N}}$.  For each entry $i\in [\s{M}],j\in [\s{N}]$, we  observe:  
\begin{align}\label{eq:obs_bernoulli}
    &\fl{P}_{ij}+\fl{E}_{ij} \text{ with probability } p,\quad 0 \text{ with probability } 1-p,
\end{align}
where  $\fl{E}_{ij}$ are  independent zero mean sub-gaussian random variables with variance proxy $\sigma^2>0$. 
We now introduce the following result from \cite{chen2019noisy}: 


\begin{lemma}[Theorem 1 in \cite{chen2019noisy}]\label{lem:source}
 Let rank $r=O(1)$ matrix $\fl{P}\in \bb{R}^{d \times d}$ with SVD decomposition $\fl{P}=\fl{\bar{U}}\f{\Sigma}\fl{\bar{V}}^{\s{T}}$ satisfy  $\|\fl{\bar{U}}\|_{2,\infty}\le \sqrt{\mu r/d}, \|\fl{\bar{V}}\|_{2,\infty}\le \sqrt{\mu r/d}$ and condition number $\kappa = O(1)$.
 Let $1\ge p \ge C \mu^2 d^{-1} \log^3 d$ for some sufficiently large constant $C>0$, $\sigma = O\Big(\sqrt{\frac{pd}{\mu^3\log d}}\|\fl{P}\|_{\infty}\Big)$. Suppose we observe noisy entries of $\fl{P}$ according to observation model in (\ref{eq:obs_bernoulli}). Then, with probability exceeding $ 1-O(d^{-3})$, 
 we can compute a 
  matrix $\widehat{\fl{P}}$ by using Algorithm \ref{algo:estimate_offline_1} (Appendix \ref{app:preliminaries}) with parameters ($\ca{U}=[\s{M}],\ca{V}=[\s{N}],\sigma^2,r,p$) s.t., 
  \begin{align}\label{eq:guarantee2}
  \small
    \|\widehat{\fl{P}}-\fl{P}\|_{\infty} \le O\Big(\frac{\sigma}{\min_i\lambda_{i}}\cdot \sqrt{\frac{\mu d\log d}{p}} \|\fl{P}\|_{\infty}\Big).
\vspace*{-30pt}    
  \end{align}
\end{lemma}
Note there are several difficulties in using Lemma \ref{lem:source}  directly in our setting which are discussed below:
\begin{rmk}[Matrix Completion for Rectangular Matrices]\label{rmk:undesirable}
Lemma \ref{lem:source} is described for square matrices and a trivial approach to use Lemma \ref{lem:source} for rectangular matrices with $\s{M}$ rows and $\s{N}$ columns (say $\s{N}\ge \s{M}$) by appending $\s{N}-\s{M}$ zero rows leads to an undesirable $(\s{N}/\s{M})^{1/2}$ factor  (Lemma \ref{thm:randomsample2}) in the error bound (the $(\s{N}/\s{M})^{1/2}$ factor does not arise if we care about spectral/Frobenius norm instead of $\s{L}_{\infty}$ norm). One way to resolve the issue is to partition the columns into $\s{N}/\s{M}$ groups by assigning each column into one of the groups uniformly at random. Thus, we create $\s{N}/\s{M}$ matrices which are almost square and apply Lemma \ref{lem:source} to recover an estimate that is close in $\s{L}_{\infty}$ norm. Thus we can recover an estimate of the entire matrix which is close in $\s{L}_{\infty}$ norm up to the desired accuracy without suffering the undesirable $(\s{N}/\s{M})^{1/2}$ factor (Lemma \ref{thm:randomsample3} and Steps 10-12 in Algorithm \ref{algo:estimate}). 
\end{rmk}

\begin{rmk}[Observation Models]\label{rmk:observation}
The observation model in  \eqref{eq:obs_bernoulli} is significantly different from  \eqref{eq:obs}. In the former, a noisy version of each element of $\fl{P}$ is observed independently with probability $p$ while in the latter, in each round $t\in [\s{T}]$, for each user $u\in [\s{M}]$, we observe noisy version of a chosen element $\rho_u(t)$. Our approach to resolve this discrepancy theoretically is to first sample a set $\Omega$ of indices according to  \eqref{eq:obs_bernoulli} and subsequently use  \eqref{eq:obs} to observe the indices in $\Omega$ (see Steps 3-6 in Algorithm \ref{algo:estimate} and Corollary \ref{coro:obs3}). Of course, this implies obtaining observations corresponding to indices in a super-set of $\Omega$ (see Step 5 in Algorithm \ref{algo:estimate}) and only using the observations in $\Omega$ for obtaining an estimate of the underlying matrix. In practice, this is not necessary and we can use all the observed indices to obtain the estimate in Step 12 of Algorithm \ref{algo:estimate}.  
\end{rmk}

\begin{rmk}[Repetition and Median Tricks]\label{rmk:median}
The smallest error that is possible to achieve by using Lemma \ref{lem:source} is by substituting $p=1$ and thereby obtaining $\|\widehat{\fl{P}}-\fl{P}\|_{\infty} \le O\Big(\sigma(\min_i\lambda_{i})^{-1}\cdot \sqrt{\mu d\log d} \|\fl{P}\|_{\infty}\Big)$ and moreover, the probability of failure is polynomially small in the dimension $d$; however, this is insufficient when $d$ is not large enough. Two simple tricks allow us to resolve this issue: 1) First we can obtain repeated observations from the same entry of the reward matrix and take its average; $s$ repetitions can bring down the noise variance to $\sigma^2/s$ 2) Second, we can use the median trick where we obtain several independent estimates of the reward matrix and compute the element-wise median to boost the success probability (see proof of Lemma \ref{lem:min_acc}).
\end{rmk}
 
We address all these issues (see Appendix \ref{app:preliminaries} for detailed proofs) and arrive at the following lemma: 

\begin{lemma}\label{lem:min_acc}
Let rank $r=O(1)$ reward matrix $\fl{P}\in \bb{R}^{\s{M} \times \s{N}}$ with SVD decomposition $\fl{P}=\fl{\bar{U}}\f{\Sigma}\fl{\bar{V}}^{\s{T}}$ satisfy  $\|\fl{\bar{U}}\|_{2,\infty}\le \sqrt{\mu r/\s{M}}, \|\fl{\bar{V}}\|_{2,\infty}\le \sqrt{\mu r/\s{N}}$ and condition number $\kappa = O(1)$. Let $d_1=\max(\s{M},\s{N})$, $d_2=\min(\s{M},\s{N})$  such that $d_2=\Omega(\mu r \log (rd_2))$ and  
 $1 \ge p \ge C\mu^2d_2^{-1} \log^3 d_2$ for sufficiently large constant $C>0$. 
 Suppose we observe noisy entries of $\fl{P}$ according to observation model in (\ref{eq:obs}).
For any positive integer $s>0$ satisfying $\frac{\sigma}{\sqrt{s}}=O\Big(\sqrt{\frac{pd_2}{\mu^3\log d_2}}\|\fl{P}\|_{\infty}\Big)$, there exists an algorithm $\ca{A}$ with parameters $s,p,\sigma$ that uses $m=O\Big(s\log (\s{MN}\delta^{-1})(\s{N}p+\sqrt{\s{N}p\log \s{M}\delta^{-1}})\Big)$ rounds to compute a matrix $\widehat{\fl{P}}$ such that with probability exceeding $ 1-O(\delta\log (\s{MN}\delta^{-1}))$  
\begin{align}\label{eq:final}
   \| \fl{P} - \widehat{\fl{P}}\|_{\infty} \le O\Big(\frac{\sigma r }{\sqrt{sd_2}}\sqrt{\frac{\mu^3\log d_2}{p}}\Big).
\end{align}

\end{lemma}

\begin{rmk}
Alg. $\ca{A}$ repeats the following process $O(\log (\s{MN}\delta^{-1}))$ times: 1) sample subset of indices $\Omega\subseteq [\s{M}]\times[\s{N}]$ such that every $(i,j)\in [\s{M}]\times[\s{N}]$ is inserted into $\Omega$ independently with probability $p$. 2) By setting $b=\max_{i\in [\s{M}]}\left|j \in [\s{N}] \mid (i,j)\in \Omega \right|$, Algorithm $\ca{A}$ invokes Alg. \ref{algo:estimate} with total rounds $bs$, number of rounds in each iteration $b$, set $\Omega$, set of users $[\s{M}]$, items $[\s{N}]$ and regularization parameter $\lambda=C_{\lambda}\sigma\sqrt{\min(\s{M},\s{N})p}$ for a suitable constant $C_{\lambda}>0$ in order to compute an estimate of $\fl{P}$. The final estimate $\widehat{\fl{P}}$ is computed by taking an entry-wise median of each individual estimate obtained as output from several invocations of Alg. \ref{algo:estimate}. Alternatively, Alg. $\ca{A}$ is detailed in Alg. \ref{algo:estimate_2} in Appendix \ref{app:preliminaries}.
\end{rmk}

Note that the total number of noisy observations made from the matrix $\fl{P}$ is $m\cdot \s{M} \ge \s{MN}\cdot p \cdot s$. Therefore, informally speaking, the average number of observations per index is $p\cdot s$ which results in an error of $\widetilde{O}(\sigma/\sqrt{sp})$ ignoring other terms (contrast with error $\widetilde{O}(\sigma/\sqrt{p})$ in \eqref{eq:guarantee2}.)

\begin{rmk}[Setting parameters $s,p$ in lemma \ref{lem:min_acc}]\label{rmk:set}
Lemma \ref{lem:min_acc} has three input parameters namely $s\in \bb{Z}, 0 \le p \le 1$ and $0\le \delta \le 1$. For any set of input parameters $(\eta,\nu)$, our goal is to set $s,p,\delta$ as functions of known $\sigma,r,\mu,d_2$ such that we can recover $\| \fl{P} - \widehat{\fl{P}}\|_{\infty} \le \eta $ with probability $1-\nu$ 
for which the conditions on $\sigma$ and $p$ are satisfied. From (\ref{eq:final}), we must have $\sqrt{sp} = \frac{c\sigma r }{\sqrt{d_2}}\frac{\sqrt{\mu^3\log d_2}}{\eta}$ for some appropriate constant $c>0$. If $r=O(1)$ and $\eta \le \|\fl{P}\|_{\infty}$, then an appropriate choice of $c$ also satisfies the condition $\frac{\sigma}{\sqrt{s}}=O\Big(\sqrt{\frac{pd_2}{\mu^3\log d_2}}\|\fl{P}\|_{\infty}\Big)$.
More precisely, we are going to set $p=C\mu^2 d_2^{-1}\log^3 d_2$ and $s=\Big\lceil \Big(\frac{c\sigma r \sqrt{\mu}}{\eta\log d_2}\Big)^2 \Big\rceil$ in order to obtain the desired guarantee. 


\end{rmk}

\section{Explore-Then-Commit (ETC) Algorithm}\label{sec:warmup}
In this section, we present an Explore-Then-Commit (ETC)  based algorithm for online low-rank matrix completion. The algorithm has two disjoint phases of exploration  and  exploitation. We will first jointly explore the set of items for all users for a certain number of rounds and compute an estimate $\widehat{\fl{P}}$ of the reward matrix $\fl{P}$. Subsequently, we commit to the estimated best item found for each user and sample the reward of the best item for the remaining rounds in the exploitation phase for that user. 
Note that the exploration phase involves using a matrix completion estimator in order to estimate the entire reward matrix $\fl{P}$ from few observed entries. Our regret guarantees in this framework is derived by carefully balancing exploration phase length and the matrix estimation error (detailed proof provided in Appendix \ref{app:warm_up}).




\begin{algorithm}[!t]
\caption{\textsc{ETC Algorithm}   \label{algo:p1}}
\begin{algorithmic}[1]
\REQUIRE users $\s{M}$, items $\s{N}$, rounds $\s{T}$, noise $\sigma^2$, rank $r$ of $\fl{P}$, upper bound on magnitude of expected rewards $||\fl{P}||_{\infty}$, no. of estimates $f=O(\log (\s{MNT}))$.

\STATE Set $d_2=\min(\s{M},\s{N})$ and $v=(\s{N \|\fl{P}\|_{\infty}})^{-2/3}\Big(\frac{\s{T}\sigma r}{\sqrt{d_2}}\sqrt{\mu^3 \log d_2}\Big)^{2/3}$. Set $p=C\mu^2d_2^{-1}\log^3 d_2$, $s=\lceil vp^{-1} \rceil$ and $\lambda=C_{\lambda}\sigma\sqrt{d_2p}$ for some constants $C,C_{\lambda}>0$.

\FOR{$k=1,2,\dots,f$}
\STATE For each tuple of indices $(i,j)\in [\s{M}]\times [\s{N}]$, independently set $\delta_{ij}=1$ with probability $p$ and $\delta_{ij}=0$ with probability $1-p$.

\STATE  Denote $\Omega=\{(i,j)\in [\s{M}]\times [\s{N}] \mid \delta_{ij}=1\}$ and
 $b=\max_{i \in [\s{M}]}\mid |j \in [\s{N}]\mid (i,j) \in \Omega|$ to be the maximum number of index tuples in a particular row. Set total number of rounds to be $bs$.

\STATE Compute the $k^{\s{th}}$ estimate  $\widehat{\fl{P}}^{(k)}=\textsc{Estimate}([\s{M}],[\s{N}],bs,\Omega,b,\lambda)$. \#
\textit{(Algorithm \ref{algo:estimate}  is used to recommend items to every user for $bs$ rounds.} 
\ENDFOR

\STATE Compute final estimate estimate $\widehat{\fl{P}}$ by taking the entry-wise median of $\widehat{\fl{P}}^{(1)},\widehat{\fl{P}}^{(2)},\dots,\widehat{\fl{P}}^{(f)}$.


\FOR{ each of remaining rounds}
\STATE Recommend $\s{argmax}_{j \in [\s{N}]}\widehat{\fl{P}}_{ij}$ for each user $i\in [\s{M}]$. \# \textit{Number of remaining rounds is $\s{T}-bsf$.}
\ENDFOR
\end{algorithmic}
\end{algorithm}

\begin{thm}\label{thm:baseline}
Consider the rank-$r$ online matrix completion problem with $\s{M}$ users, $\s{N}$ items, $\s{T}$ recommendation rounds. Set $d_2=\min(\s{M},\s{N})$. Let $\fl{R}^{(t)}_{u\rho_u(t)}$ be the reward in each round, defined as in \eqref{eq:obs}. Suppose $d_2=\Omega(\mu r\log(rd_2))$. Let $\fl{P}\in \bb{R}^{\s{M}\times \s{N}}$ be the expected reward matrix that satisfies the conditions stated in Lemma \ref{lem:min_acc} , and let $\sigma^2$ be the noise variance in rewards. Then, Algorithm \ref{algo:p1}, applied to the online rank-$r$ matrix completion problem guarantees the following regret: 
\begin{align}\label{eq:ub_etan2}
   \s{Reg}(\s{T}) = O\Big(\Big(\s{T}^{\frac23}(\sigma^2r^2 \|\fl{P}\|_{\infty})^{\frac13}\Big(\frac{\mu^3 \s{N} \log d_2}{d_2}\Big)^{1/3}+\frac{\s{N}\mu^2\|\fl{P}\|_{\infty}}{d_2}\Big) \log^5(\s{MNT}) +\frac{\|\fl{P}\|_{\infty}}{\s{T}^{2}}\Big).
\end{align}
\end{thm}

\begin{rmk}[Non-trivial regret bounds]
Theorem \ref{thm:baseline} provides non-trivial regret guarantees in the key regime when $\s{N} \gg \s{T}$ and $\s{M}>\s{N}$ where the regret scales only logarithmically on $\s{M},\s{N}$. This is intuitively satisfying since in each round we are obtaining $\s{M}$ observations, so more users translate to more information which in-turn allows better understanding of the underlying reward matrix. 
However, the dependence of regret on $T$ (namely $\s{T}^{2/3}$) is sub-optimal. In the subsequent section, we provide a novel  algorithm to obtain regret guarantees with $\s{T}^{1/2}$ for rank-$1$ $\fl{P}$. 

\end{rmk}

\begin{rmk}[Gap dependent bounds]
Define the minimum gap to be $\Delta=\min_{u\in [\s{M}]} \left|\fl{P}_{u\pi_u(1)}-\fl{P}_{u\pi_u(2)}\right|$ where $\pi_u(1),\pi_u(2)$ corresponds to the items with the highest and second highest reward for user $u$ respectively. If the quantity $\Delta$ is known then it is possible to design ETC algorithms where length of the exploration phase is tuned accordingly in order to obtain regret bounds that scale logarithmically with the number of rounds $\s{T}$.
\end{rmk}


\section{OCTAL Algorithm}\label{sec:main1}

In this section we present our algorithm OCTAL (Algorithm \ref{algo:phased_elim}) for online matrix completion where the reward matrix $\fl{P}$ is rank $1$. The set of users is described by a latent vector $\fl{u}\in \bb{R}^{\s{M}}$ and the set of items is described by a latent vector $\fl{v}\in \bb{R}^{\s{N}}$. Thus  $\fl{P}=\fl{u}\fl{v}^{\s{T}}$ with SVD decomposition  $\fl{P}=\bar{\lambda}\fl{\bar{u}}\fl{\bar{v}}^{\s{T}}$.




\begin{algorithm*}[t]
\caption{\small \textsc{OCTAL
(Online Collaborative filTering using iterAtive user cLustering)}   \label{algo:phased_elim}}
\begin{algorithmic}[1]
\REQUIRE Number of users $\s{M}$, items $\s{N}$, rounds $\s{T}$, noise $\sigma^2$, bound on the entry-wise magnitude of expected rewards $||\fl{P}||_{\infty}$, incoherence $\mu$.
\STATE  Set $\ca{M}^{(1,1)}=\ca{M}^{(1,2)}=\phi$ and $\ca{B}^{(1)}=[\s{M}]$. Set $\ca{N}^{(1,1)}=\ca{N}^{(1,2)}=\phi$. Set $f=O(\log (\s{MNT}))$ and suitable constants $a,c,C,C',C_{\lambda}>0$.

\FOR{$\ell=1,2,\dots,$}

\STATE Set $\Delta_{\ell}=C'2^{-\ell}\min\Big(\|\fl{P}\|_{\infty},\frac{\sigma\sqrt{\mu}}{\log \s{N}}\Big)$. 

\FOR{$k=1,2,\dots,f$}

\FOR{each pair of non-null sets $(\ca{B}^{(\ell)},\s{N})$, $(\ca{M}^{(\ell,1)},\ca{N}^{(\ell,1)}),(\ca{M}^{(\ell,2)},\ca{N}^{(\ell,2)})\subseteq [\s{M}]\times [\s{N}]$}

\STATE Denote $(\ca{T}^{(1)},\ca{T}^{(2)})$ to be the considered pair of sets and $i\in \{0,1,2\}$ to be its index.

\STATE Set $d_{2,i}=\min(|\ca{T}^{(1)}|,|\ca{T}^{(2)}|)$. Set $p_{\ell,i}=C\mu^2 d_{2,i}^{-1}\log^3 d_{2,i}$ and $s_{\ell,i}=\Big\lceil \Big(\frac{c\sigma  \sqrt{\mu}}{\Delta_{\ell}\log d_{2,i}}\Big)^2 \Big\rceil$.

\STATE For each tuple of indices $(u,v)\in \ca{T}^{(1)}\times \ca{T}^{(2)}$, independently set $\delta_{uv}=1$ with probability $p_{\ell,i}$ and $\delta_{uv}=0$ with probability $1-p_{\ell,i}$.

\STATE  Denote $\Omega^{(i)}=\{(u,v)\in \ca{T}^{(1)}\times \ca{T}^{(2)} \mid \delta_{uv}=1\}$ and
 $b_{\ell,i}=\max_{u \in \ca{U}} |v \in \ca{V}\mid (u,v) \in \Omega|$. Set total number of rounds to be $m_{\ell,i}=b_{\ell,i}s_{\ell,i}$.

\ENDFOR

\STATE Set $m_{\ell}=\max_{i\in \{0,1,2\}} m_{\ell,i}$.

\STATE Compute $\widetilde{\fl{Q}}^{(\ell,k)}=\textsc{Estimate}(|\ca{B}^{(\ell)}|,[\s{N}],m_{\ell},\Omega^{(0)},b_{\ell,0},\lambda=C_{\lambda}\sigma\sqrt{d_{2,0}p_{\ell}})$. \#
\textit{Algorithm \ref{algo:estimate} is used to recommend items to every user in $\ca{B}^{(\ell)}$ for $m_{\ell}$ rounds.} 

\STATE For $i\in\{1,2\}$, compute $\widetilde{\fl{P}}^{(\ell,i,f)}=\textsc{Estimate}(|\ca{M}^{(\ell,i)}|,|\ca{N}^{(\ell,i)}|,m_{\ell},\Omega^{(i)},b_{\ell,i},\lambda=C_{\lambda}\sigma\sqrt{d_{2,i}p_{\ell}})$. \#
\textit{Algorithm \ref{algo:estimate} recommends items to every user in $\ca{M}^{(\ell,i)}$ for $m_{\ell}$ rounds}.


\ENDFOR 

\STATE Compute $\widetilde{\fl{Q}}^{(\ell)}=$Entrywise Median$(\{\widetilde{\fl{Q}}^{(\ell,k)}\}_{k=1}^{f})$, $\widetilde{\fl{P}}^{(\ell,i)}=$Entrywise Median$(\{\widetilde{\fl{P}}^{(\ell,i,k)}\}_{k=1}^{f})$ for $i\in \{1,2\}$.

\STATE Set $\ca{B}^{(\ell+1)} \equiv \Big\{u \in \ca{B}^{(\ell)}\mid \left|\max_{t\in [\s{N}]}\fl{\widetilde{Q}}^{(\ell)}_{ut}-\min_{t\in [\s{N}]}\fl{\widetilde{Q}}^{(\ell)}_{ut}\right|\le 2a\Delta_{\ell}\Big\}$ 

\STATE Compute $\ca{T}_{u}^{(\ell+1)}=\{j \in [\s{N}]\}\mid \fl{\widetilde{Q}}^{(\ell)}_{uj}+\Delta_{\ell}> \max_{t \in [\s{N}]}\fl{\widetilde{Q}}^{(\ell)}_{ut}\}$ for all $u \in \ca{B}^{(\ell)}\setminus \ca{B}^{(\ell+1)}$.

\STATE For $i\in \{1,2\}$, for all users $u \in \ca{M}^{(\ell,i)}$, compute $\ca{T}_{u}^{(\ell+1)}=$ $\{j \in \ca{N}^{(\ell,i)}\mid \fl{\widetilde{P}}^{(\ell,i)}_{uj}+\Delta_{\ell}> \max_{t \in \ca{N}^{(\ell,i)}}\fl{\widetilde{P}}^{(\ell,i)}_{ut}\}$.


\STATE Set $v$ to be any user in $[\s{M}]\setminus \ca{B}^{(\ell+1)}$. Set $\ca{M}^{(\ell+1,1)}=\{u \in [\s{M}]\setminus \ca{B}^{(\ell+1)} \mid \ca{T}_{u}^{(\ell+1)}\cap \ca{T}_{v}^{(\ell+1)}\neq \phi\}$. Set $\ca{M}^{(\ell+1,2)}=[\s{M}]\setminus (\ca{B}^{(\ell+1)}\cup \ca{M}^{(\ell+1,1)})$.

\STATE Compute $\ca{N}^{(\ell+1,1)}=\bigcap_{u \in \ca{M}^{(\ell+1,1)}}\ca{T}_u^{(\ell+1)}$,  $\ \ \ \ca{N}^{(\ell+1,2)}=\bigcap_{u \in \ca{M}^{(\ell+1,2)}}\ca{T}_u^{(\ell+1)}$.

\STATE For $i\in \{1,2\}$, if $|\ca{M}^{(\ell+1,i)}|\le \frac{\s{M}}{\sqrt{\s{T}}}$, then set $\ca{B}^{(\ell+1)} \leftarrow \ca{B}^{(\ell+1)} \cup \ca{M}^{(\ell+1,i)}$ and  $\ca{M}^{(\ell+1,i)}\leftarrow \phi$.


\ENDFOR

\end{algorithmic}
\end{algorithm*}

{\bf Algorithm  Overview:}\label{subsec:overview} 
Our first key observation is that as $\fl{P}$ is rank-one, we can partition the set of users into two  disjoint clusters $\ca{C}_1,\ca{C}_2$ where $\ca{C}_1 \equiv \{i \in [\s{M}] \mid \fl{u}_i \ge 0\}$ and $\ca{C}_2 \equiv [\s{M}]\setminus \ca{C}_1$. Clearly, for all users $u\in \ca{C}_1$, the item that results in maximum reward is $j_{\max}=\s{argmax}_{t\in [\s{N}]} \fl{v}_t$. On the other hand, for all users $u\in \ca{C}_2$, the item that results in maximum reward is $j_{\min}=\s{argmin}_{t\in [\s{N}]} \fl{v}_t$. Thus, if we can identify $\ca{C}_1,\ca{C}_2$ and estimate items with high reward (identical for users in the same cluster) using few recommendations per user, we can ensure low regret. 

But, initially $\ca{C}_1,\ca{C}_2$ are unknown, so all users are {\em unlabelled} i.e., their cluster is unknown. 
 In each phase (the outer loop indexed by $\ell$),  Algorithm~\ref{algo:phased_elim} tries to label at least a few unlabelled users correctly. This is achieved by progressively refining estimate $\widetilde{\fl{Q}}$ of the reward matrix $\fl{P}$ restricted to the unlabelled users and all items  (Step 12). Subsequently, unlabelled  users for which the difference in maximum and minimum reward (inferred from estimated reward matrix) is large are labelled  (Step 19). At the same time, in Step $13$ users labelled in previous phases are partitioned into two clusters (denoted by $\ca{M}^{(\ell,1)}$ and $\ca{M}^{(\ell,2)}$) and for each of them, the algorithm refines an estimate of two distinct sub-matrices of the reward matrix $\fl{P}$ by recommending items only from a refined set ($\ca{N}^{(\ell,1)}$ and $\ca{N}^{(\ell,2)}$ respectively) containing the best item ($j_{\max}$ or $j_{\min}$). 
 We also identify a small set of \textit{good} items for each labelled user (including users labelled in previous phases), which correspond to large estimated rewards. 
 We partition all these users into two clusters ($\ca{M}^{(\ell+1,1)}$ and $\ca{M}^{(\ell+1,2)}$) such that the set of \textit{good} items for users in different clusters are disjoint. 
 We can prove that such a partitioning is possible; users in same cluster have same sign of user embedding.    
 




 
 We also prove that the set of good items contain the best item for each labelled user ($j_{\max}$ or $j_{\min}$). So, after each phase, for each cluster of users, we compute the intersection of \textit{good} items over all users in the cluster.  This subset of items (\textit{joint good} items) must contain the best item for that cluster and therefore we can discard the other items (Step  20). We can show that all items in the set of \textit{joint good} items ($\ca{N}^{(\ell+1,1)}$ and $\ca{N}^{(\ell+1,2)}$) have rewards which is  close to the reward of the best item.  Therefore the algorithm suffers small regret if for each group of labelled users, the algorithm recommends items from the set of \textit{joint good} items (Step 13) in the next phase.
 We can further show that for the set of unlabelled users, the difference in rewards between the best item and worst item is small and hence the regret for such users is small, irrespective of the recommended item (Step 12). 
 Note that until the number of labelled users is sufficiently large, we do not consider them separately (Step 21). A crucial part of our analysis is to show that for any subset of users and items considered in Step 5, the number of rounds sufficient to recover a good estimate of the expected reward sub-matrix is small irrespective of the number of considered items (if the number of users is sufficiently large).


\begin{rmk}[Practical considerations]
 In general OCTAL (Alg. \ref{algo:phased_elim}) is computationally faster than the ETC Algorithm (Alg. \ref{algo:p1}) with a higher exploration length. This is because OCTAL eliminates large chunks of items in every phase and therefore has to solve easier optimization problems; on the other hand, ETC has to solve a low rank matrix completion problem in $\s{MN}$ variables that becomes slower with the exploration length (datapoints). Moreover, OCTAL algorithm runs in phases with the initial phases being very small; hence the users do not have to wait for a long time to even get personalized recommendations like in ETC. These features make OCTAL much more practical than ETC.
\end{rmk}
 
 To summarize, in Algorithm \ref{algo:phased_elim}, the entire set of rounds $[\s{T}]$ is partitioned into phases of exponentially increasing length. In each phase, for the set of unlabelled users, we do pure exploration and recommend random items from the set of all possible items (Step 12). The set of labelled users are partitioned into two clusters; for each, we follow a semi-exploration strategy where we recommend random items from a set of \textit{joint good} items (Steps 13). We now introduce the following definition:

 
 
\begin{defn}[$(\alpha,\mu)$-Local Incoherence]\label{def:inc}
For $0 \le \alpha \le 1 $, a vector $\fl{v}\in \bb{R}^m$ is $(\alpha,\mu)$-local incoherent if for all sets $\ca{U}\subseteq [m]$ satisfying $|\ca{U}|\ge \alpha m$, we have $\|\fl{v}_{\ca{U}}\|_{\infty} \le \sqrt{\frac{\mu}{|\ca{U}|}}\|\fl{v}_{\ca{U}}\|_2$.  
\end{defn}
 
Local incoherence for a vector $\fl{v}$ implies that any sub-vector of $\fl{v}$ having a significant size must be incoherent as well. 
Note that the local incoherence condition is trivially satisfied if the magnitude of each vector entry is bounded from below. We are now ready to state our main result:



\begin{thm}\label{thm:main1}
Consider the rank-$1$ online matrix completion problem with $\s{T}$ rounds, $\s{M}$ users s.t. $\s{M}\ge \sqrt{\s{T}}$ and $\s{N}$ items. Denote $d_2=\min(\s{M},\s{N})$.
Let $\fl{R}^{(t)}_{u\rho_u(t)}$ be the reward in each round, defined as in \eqref{eq:obs}. Let $\sigma^2$ be the noise variance in rewards and let $\fl{P}\in \bb{R}^{\s{M}\times \s{N}}$ be the expected reward matrix with SVD decomposition $\fl{P}=\lambda\fl{\bar{u}}\fl{\bar{v}}^{\s{T}}$ such that $\fl{\bar{u}}$ is $(\s{T}^{-1/2},\mu)$-locally incoherent, $ \|\fl{\bar{v}}\|_{\infty}\le \sqrt{\mu/\s{N}}$, $d_2=\Omega(\mu \log d_2)$  
  and $|\fl{\bar{v}}_{j_{\min}}|=\Theta(|\fl{\bar{v}}_{j_{\max}}|)$.  Then, by suitably choosing parameters $\{\Delta_{\ell}\}_{\ell}$, positive integers $\{s_{(\ell,0)},s_{(\ell,1)},s_{(\ell,2)}\}_{\ell}$ and $1 \ge \{p_{(\ell,0)},p_{(\ell,1)},p_{(\ell,2)}\}_{\ell} \ge 0$ as described in Algorithm \ref{algo:phased_elim}, we can ensure a regret guarantee of $\s{Reg}(\s{T})=O(\sqrt{\s{T}}\|\fl{P}\|_{\infty}+\s{J}\sqrt{\s{TV}})$ where $\s{J}=O\Big(\log \Big(\frac{1}{\sqrt{\s{VT}^{-1}}}\min\Big(\|\fl{P}\|_{\infty},\frac{\sigma\sqrt{\mu}}{\log \s{N}}\Big)\Big)\Big)$ and $\s{V}=\Big(\max(1,\frac{\s{N}\sqrt{\s{T}}}{\s{M}})\sigma^2\mu^3\log^2(\s{MNT})\Big)$.
\end{thm}

Similar to Algorithm \ref{algo:p1}, Algorithm \ref{algo:phased_elim} allows non-trivial regret guarantees even when $\s{N}\gg \s{T}$  provided the number of users is significantly large as well i.e. $\s{M}=\widetilde{\Omega}(\s{N}\sqrt{\s{T}})$.

\begin{rmk}
Under slightly stronger local incoherence conditions on the vector $\bar{\fl{u}}$, we can analyze a modified version of \textsc{OCTAL} (Alg. \ref{algo:phased_elim2}) without requiring users $\s{M}$ to be large. When $|\ca{C}_1|\approx |\ca{C}|_2$, the 
regret guarantee (Thm. \ref{thm:main2}) scales as $\widetilde{O}(\sqrt{\s{NT}/\s{M}})$. Due to space limitations, details of Algorithm \ref{algo:phased_elim2} and the proof of Theorem \ref{thm:main2} can be found in Appendix \ref{app:repeated2}
\end{rmk}

 Finally, we show that the above dependence on $\s{N,M,T}$ matches the lower bound that we obtain by reduction to the well-known multi-armed bandit problem.



\begin{thm}\label{thm:lb}
Let $\fl{P}\in [0,1]^{\s{M}\times \s{N}}$ be a rank $1$ reward matrix and the noise variance $\sigma^2=1$. In that case, any algorithm for online matrix completion problem will suffer regret of $\Omega(\sqrt{\s{NT}\s{M}^{-1}})$.
\end{thm}




\section{Conclusions}\label{sec:conc}
We studied the problem of online rank-one matrix completion in the setting of repeated item recommendations and blocked item recommendations, which should be applicable for several practical recommendation systems. We analyzed an explore-then-commit (ETC) style method which is able to get the regret averaged over users to be nearly independent of number of items. That is, per user, we require only logarithmic many item recommendations to get non-trivial regret bound. But, the dependence on the number of rounds $\s{T}$ is sub-optimal. We further improved this dependence by proposing OCTAL that carefully combines exploration, exploitation and clustering for different users/items. Our methods iteratively refines estimate of the underlying reward matrix, while also identifying users which can be recommended certain items confidently. Our algorithms and proof techniques are  significantly different than existing bandit learning literature. 
We believe that our work only scratches the surface of an important problem domain with several open problems.  For example, Algorithm~\ref{algo:phased_elim} requires rank-$1$ reward matrix. Generalizing the result to rank-$r$ reward matrices would be interesting. Furthermore, relaxing assumptions on the reward matrix like stochastic noise or additive noise should be relevant for several important settings. 
Finally, collaborative filtering can feed users related items, hence might exacerbate their biases. Our method might actually help mitigate the bias due to explicit exploration, but further investigation into such challenges is important.

\bibliographystyle{abbrvnat}

\newpage
\appendix

\section{Experiments}\label{sec:sims}

\begin{figure*}[!t]
  \begin{subfigure}[t]{0.47\textwidth}
    \centering 
    \includegraphics[scale = 0.34]{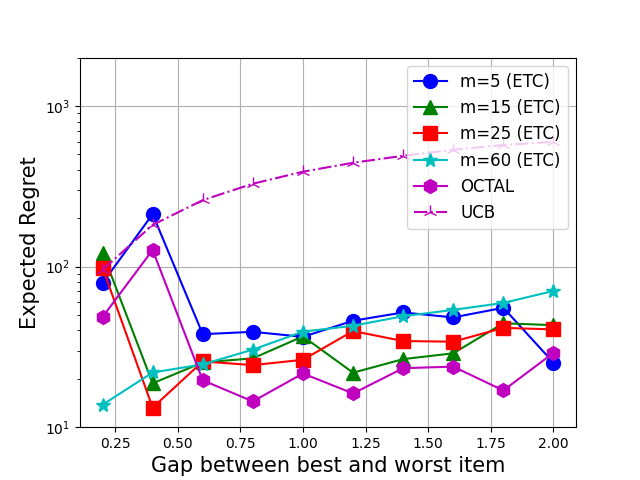}\vspace*{-5pt}
    \caption{Regret comparison computed by Alg. \ref{algo:p1} for exploration periods ($5,15,45,50$), Alg. \ref{algo:phased_elim} and UCB when $\s{T}=1000$. The average regret is plotted with gap in reward between best and worst item for all users.}
 ~\label{fig:average_regret_gap1}
  \end{subfigure}
  \hfill
\begin{subfigure}[t]{0.47 \textwidth}
\centering
   \includegraphics[scale = 0.34]{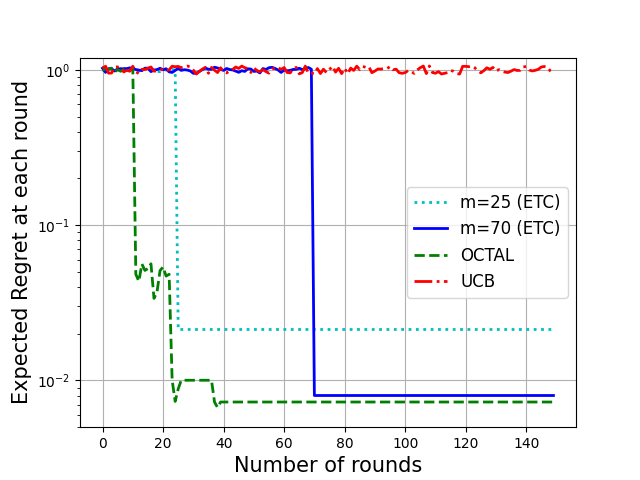}\vspace*{-5pt}
   \caption{\small Comparison of the regret incurred in each round ($\s{M}^{-1}\sum_{u \in [\s{M}]}\f{\mu}_u^{\star}- \bb{E}\frac{1}{\s{M}}\sum_{u\in[\s{M}]}\mathbf{R}_{u\rho_u(t)}^{(t)}$) by Algorithm \ref{algo:p1} (for exploration period $m=25,70$), Algorithm \ref{algo:phased_elim} and the UCB algorithm when $\s{T}=1000$.}
       ~\label{fig:time_comparison}
 \end{subfigure}%
 \caption{Comparison of regret incurred by OCTAL algorithm (Algorithm~\ref{algo:phased_elim}, and ETC algorithm  \ref{algo:p1} against UCB, when expected reward matrix is rank-$1$. See Section~\ref{app:simulation} for data generation process. Clearly, both OCTAL and ETC have significantly lower regret than baseline UCB method which does not exploit reward matrix's low-rank structure.}
\end{figure*}

\subsection{Synthetic Datasets}\label{app:simulation}

 We set the total number of users $\s{M}=100$, the total number of items $\s{N}=150$. For the total number of rounds, we look at two settings namely 1) $\s{T}=1000$ ($\s{T} \gg \s{N}$) and 2) $\s{T}=100$ ($\s{T} \ll \s{N}$). The former one allows us to demonstrate the tension between exploration and exploitation clearly as the number of rounds is large and moreover, it also allows us to compare with the Upper Confidence Bound (UCB) algorithm (see Remark \ref{rmk:begin}) that does not use the low rank structure of the algorithm. In the latter setting, we demonstrate empirically that our regret bounds are non-trivial even when the number of rounds is significantly smaller than the number of items. 
We design the unknown reward matrix $\fl{P}=\fl{uv}^{\s{T}}$ by designing $\fl{u} \in \bb{R}^{\s{M}},\fl{v}\in \bb{R}^{\s{N}}$ corresponding to the embeddings of the users and items respectively in the following manner: the entries of $\fl{u}$ are randomly sampled from the set $\{1,-1\}$. All the entries in the vector $\fl{v}$ are uniformly sampled from the interval $[\s{Gap}/2,-\mathsf{Gap}/2]$ where $\s{Gap}$ is a parameter unknown to the implemented algorithms. With such a construction, we ensure almost surely that the difference between the reward of the best item and second best item also increases with $\s{Gap}$.
The algorithms are allowed to make observations from a noisy version of $\fl{P}$ (say $\fl{P}_{\s{noisy}}$) where $\fl{P}_{\s{noisy}} \leftarrow \fl{P}+\fl{E}$ where every entry of $\fl{E}$ is sampled i.i.d according to $\ca{N}(0,0.1)$. \\


We vary $\s{Gap}$ and run Algorithm \ref{algo:p1} with different exploration periods $m$.
For each configuration of $\mathsf{Gap},m$ we run Algorithm \ref{algo:p1} $10$ times and store the regret (see the relevant definition in equation \ref{eq:obs}) averaged across the $10$ runs. 
In Figure \ref{fig:average_regret_exploration}, we compare the regret of Algorithm \ref{algo:p1} with different values of $\s{Gap}$ as we change the exploration periods for $\s{T}=1000$. 
Clearly, with a small exploration period, the algorithm commits to the wrong item more often and with a large exploration period, the cost of exploration is too high. The existence of a sweet spot can be noticed from the U-curves in Figure \ref{fig:average_regret_exploration}. 

\begin{figure*}[!htbp]
  \begin{subfigure}[t]{0.49\textwidth}
    \centering 
    \includegraphics[scale = 0.4]{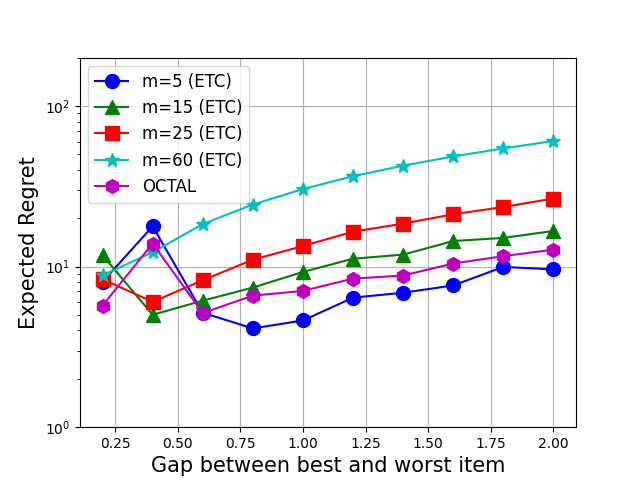}
    \caption{Comparison of the regret computed by Algorithm \ref{algo:p1} for exploration periods $5,15,25,60$ ($5,15$ are the optimal exploration periods for Algorithm \ref{algo:p1}) and Algorithm \ref{algo:phased_elim} when $\s{T}=100$. The average regret is plotted with gap in reward between best and worst item for all users.}
          ~\label{fig:average_regret_gap2}
  \end{subfigure}
  \hfill
  \begin{subfigure}[t]{0.49\textwidth}
  \centering
     \includegraphics[scale =  0.4]{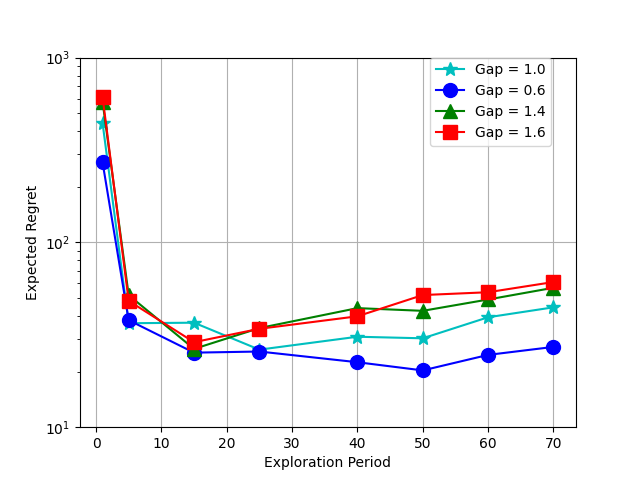}
     \caption{\small Comparison of the regret computed by Algorithm \ref{algo:p1} for $4$ different gaps in reward between best and worst item when $\s{T}=1000$. The average regret is plotted with different exploration periods for four different values of $\s{Gap}$ resulting in U-curves.}
          ~\label{fig:average_regret_exploration}
  \end{subfigure}\hfill 


 \caption{\small Simulation results on Regret averaged over a few runs (See Defintion \ref{eqn:regretExpect}) of Algorithms \ref{algo:p1} and \ref{algo:phased_elim}}
\end{figure*}

Similarly, we run Algorithm \ref{algo:phased_elim} for the same values of $\s{Gap}$, each for $10$ times and store the average regret. We also run the Upper Confidence Bound ($(1,1)$-UCB) algorithm \cite{bubeck2012regret} for each user separately and compute the regret (averaged over 10 simulations).
In Figure \ref{fig:average_regret_gap1}, we compare the regrets incurred by Algorithm \ref{algo:p1}, Algorithm \ref{algo:phased_elim} and the UCB algorithm when $\s{T}=1000$. Similarly, in Figure \ref{fig:average_regret_gap2}, we compare the regrets incurred by Algorithm \ref{algo:p1} and Algorithm \ref{algo:phased_elim} when $\s{T}=100$. Note that when $\s{T}=1000$, the tension between exploration and exploitation is clearer as the number of exploitation rounds is larger for Algorithm \ref{algo:p1}; Algorithm \ref{algo:phased_elim} performs better than ETC algorithms for all the different values of exploration periods used for many of the different values of $\s{Gap}$. However, even for $\s{T}=100$, the performance of Algorithm \ref{algo:phased_elim} is quite close to that of Algorithm \ref{algo:p1} with the best performing exploration period parameter.

For Algorithm \ref{algo:phased_elim}, in phase indexed by $\ell$, we use the number of rounds to be $m_{\ell}=10+2^{\ell}$.
We also make Step 11 in Algorithm \ref{algo:phased_elim} more robust by taking those items in $\ca{N}^{(\ell+1,1)},\ca{N}^{(\ell+1,2)}$ that are present in at least $2/3$ of the set $\ca{T}_u^{(\ell+1)}$ for users $u$ in the sets $\ca{M}^{(\ell+1,1)},\ca{M}^{(\ell+1,2)}$ respectively. If, due to some failure event, $\ca{N}^{(\ell+1,1)}=\phi$ (or $\ca{N}^{(\ell+1,2)}=\phi$) then we set $\ca{N}^{(\ell+1,1)}$ ($\ca{N}^{(\ell+1,2)}$) to be the item that is present is maximum number of sets  $\ca{T}_u^{(\ell+1)}$ for users $u$ in the sets $\ca{M}^{(\ell+1,1)}\; (\ca{M}^{(\ell+1,2)})$.


In our final experiment in the repeated setting with $\s{T}=1000$, we plot the regret $\s{M}^{-1}\sum_{u \in [\s{M}]}\max_{j \in [\s{N}]}\fl{P}_{uj}- \frac{1}{\s{M}}\sum_{u\in[\s{M}]}\mathbf{P}_{u\rho_u(t)}$ at time $t$ (up to $t<150$) as $t$ is increased for Algorithms \ref{algo:p1},  \ref{algo:phased_elim} and the UCB Algorithm (the reported values are averaged over $10$ simulations). 
For Algorithm \ref{algo:p1}, the regret is high at all times during the exploration period and then it experiences a sharp drop in the exploitation period. On the other hand, for Algorithm \ref{algo:phased_elim}, the regret decreases in each phase and so we observe a more gradual decrease in the regret; thus Algorithm \ref{algo:phased_elim} is an anytime algorithm. Also note that the regret of the UCB algorithm does not decrease during the first $150$ rounds.

\subsection{Real Datasets}

\paragraph{MovieLens:}

Next, we demonstrate  experimental results on the MovieLens dataset on the lines of \cite{kveton2017stochastic}.  The MovieLens 1M dataset comprises of 1 million ratings provided by 6K users to 4K movies. We cluster the users into disjoint groups ($241$) where each cluster represents a unique combination of gender, age group, and occupation in the MovieLens dataset.  Moreover, we  use a random subset of movies ($300$) for which the average rating is between 2.5 and 3.5 ( in order to avoid movies which are too good or too bad). For each pair of (user group, movie), we take the average of all the ratings provided by users in that group for that movie. Around $33\%$ of the $241\times 300$ ratings matrix could be filled in this manner. In order to complete the matrix, we optimized a convex program for low rank matrix completion (see eq. (6) and \citep{chen2019noisy}) by minimizing the MSE with a nuclear norm regularizer.

We randomly take $\mathsf{M}=128$ user groups, $\mathsf{N}=128$ movies and take the total rounds $\mathsf{T}=100$. Note that the number of rounds is less than the number of items (movies) and therefore, the UCB algorithm implemented separately for each user group can only explore and incur a regret of $179.61$. For the ETC Algorithm (Alg. \ref{algo:p1}), as a baseline (for a fair comparison with OCTAL), we consider the rank $1$-approximation of the estimated reward matrix $\widehat{\fl{P}}$ (ETC Rank-$1$). From Figure \ref{fig:movielens1}, note that the OCTAL Algorithm (Alg. \ref{algo:phased_elim} with the same minor mofications as in the synthetic datasets)  outperforms ETC with the rank-$1$ approximation by a significant margin.
In fact, in Figure \ref{fig:movilens3}, we can see that the OCTAL Algorithm has smaller cumulative regret in every round as compared to ETC with rank-$1$ approximation for different exploration periods. All results that are reported are an average of $10$ independent runs.

Although the OCTAL algorithm crucially uses the rank $1$ structure, it has reasonable performance even when the reward matrix has a larger rank but can still be approximated well with a rank $1$ matrix; in our experiment, it turns out that the highest singular value of the ratings matrix is $\approx 10$ times larger than the second largest singular value.  In figure \ref{fig:movilens2}, we also compare the round-wise incurred regret at different rounds $t \in [100]$. 

\begin{figure*}[!htbp]
  \begin{subfigure}[t]{0.33\textwidth}
    \centering 
    \includegraphics[scale = 0.3]{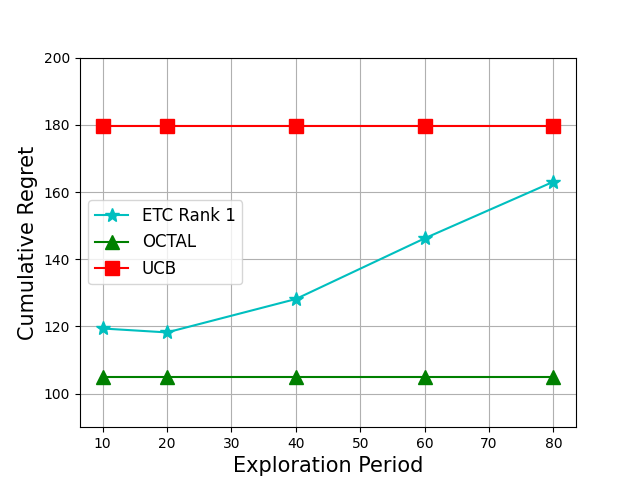}\vspace*{-5pt}
 \caption{Comparison of regret at $\s{T}=100$ for UCB, OCTAL and ETC (with rank-$1$ approxmiation) for different exploration periods}~\label{fig:movielens1}
  \end{subfigure}
  \hfill
\begin{subfigure}[t]{0.32 \textwidth}
\centering
   \includegraphics[scale = 0.3]{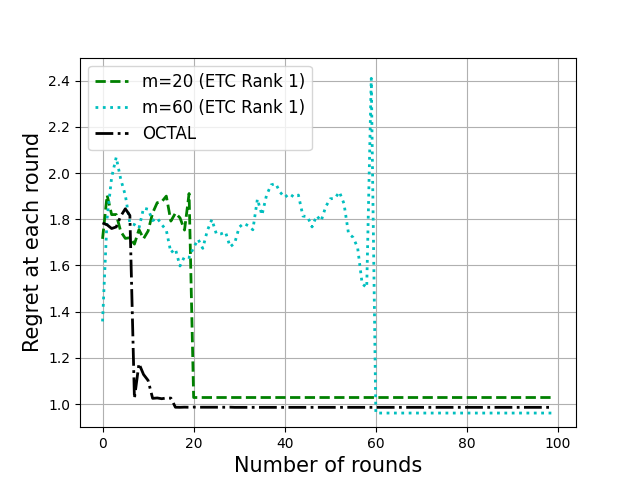}\vspace*{-5pt}
       \caption{Comparison of regret at every round for OCTAL and ETC with rank-$1$  approximation and exploration period $m=20,40$}~\label{fig:movilens2}
 \end{subfigure}%
 \hfill
\begin{subfigure}[t]{0.32 \textwidth}
\centering
   \includegraphics[scale = 0.3]{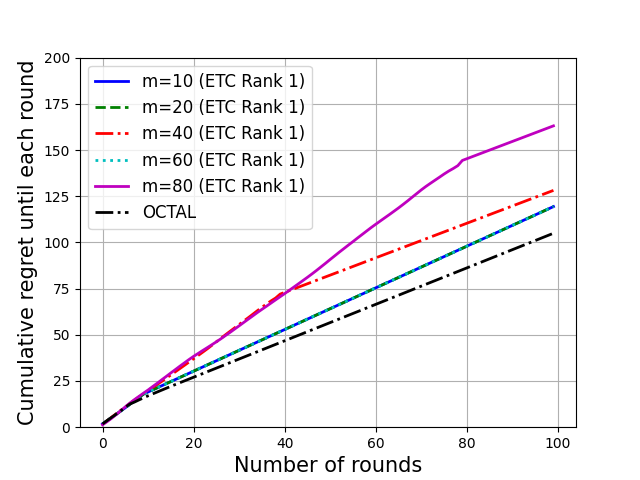}\vspace*{-5pt}
       \caption{Comparison of cumulative regret until every round for OCTAL and ETC with rank-$1$ approximation and different exploration periods}~\label{fig:movilens3}
 \end{subfigure}%
 \caption{Comparison of regret incurred by OCTAL algorithm (Algorithm~\ref{algo:phased_elim}, ETC algorithm  \ref{algo:p1} with rank-$1$ approximation of the estimated matrix against UCB for the MovieLens dataset. In Figure \ref{fig:movielens1}, we show that OCTAL and both versions of ETC have significantly lower regret than baseline UCB method. In Figure \ref{fig:movilens2}, we show that the regret of OCTAL decreases in every phase unlike ETC algorithms. In Figure \ref{fig:movilens3}, we show that the cumulative regret of OCTAL is lower than that of ETC with rank $1$ approximation at every round $t\le \s{T}$.}
\end{figure*}

Notice from Figure \ref{fig:movilens2} that the regret of the OCTAL algorithm decreases continuously making it anytime which is extremely useful in practical applications. On the other hand, the regret of the ETC algorithm only decreases after the exploration period is over which is undesirable. In fact, for a large number of initial rounds, the OCTAL algorithm outperforms the ETC algorithms.

\paragraph{Jester Dataset:} Next, we consider the Jester dataset \citep{goldberg2001eigentaste} which consists data from $24983$ users who have provided ratings for $100$ jokes that are between $-10.0$ to $+10.0$. We select $\s{M}=100$ users who have rated all the $\s{N}=100$ jokes and use the corresponding $100\times 100$ ratings matrix as the underlying reward matrix; we select the number of recommendation rounds $\s{T}=100$ However, this reward matrix is not well-approximated by a low rank matrix unlike the MovieLens dataset since there exists a long tail of singular values with large magnitude. Despite this, we demonstrate the efficacy of our algorithms. First, we run the UCB algorithm for each of the $100$ users separately for $100$ rounds and thus incur an average regret of $734.66$.

\begin{figure*}[!htbp]
  \begin{subfigure}[t]{0.33\textwidth}
    \centering 
    \includegraphics[scale = 0.3]{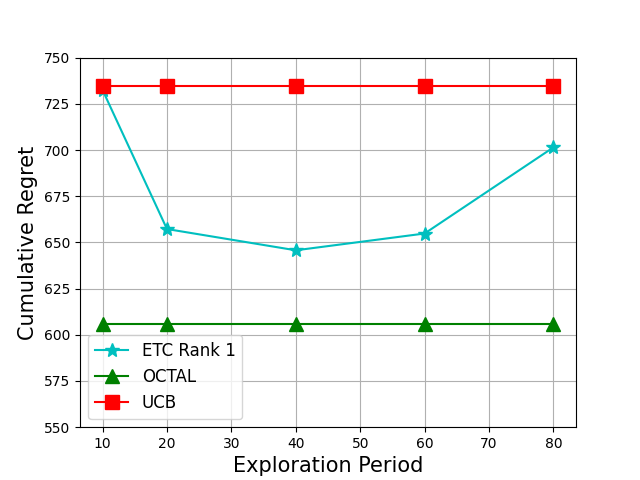}\vspace*{-5pt}
 \caption{Comparison of regret at $\s{T}=100$ for UCB, OCTAL and ETC (with rank-$1$ approxmiation) for different exploration periods}~\label{fig:jester1}
  \end{subfigure}
  \hfill
\begin{subfigure}[t]{0.32 \textwidth}
\centering
   \includegraphics[scale = 0.3]{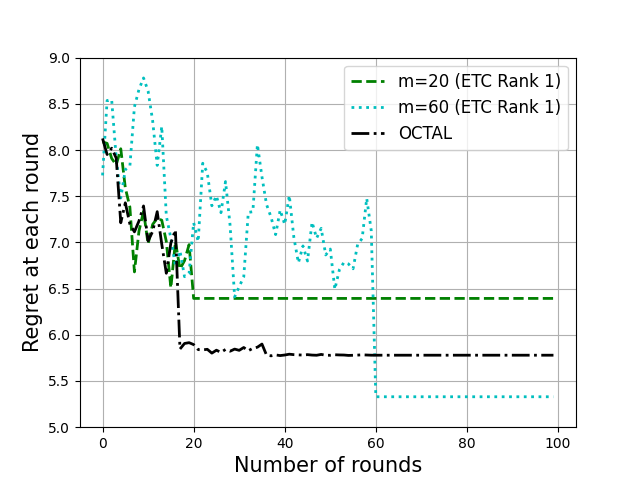}\vspace*{-5pt}
       \caption{Comparison of regret at every round for OCTAL and ETC with rank-$1$  approximation and exploration period $m=20,40$}~\label{fig:jester2}
 \end{subfigure}%
 \hfill
\begin{subfigure}[t]{0.32 \textwidth}
\centering
   \includegraphics[scale = 0.3]{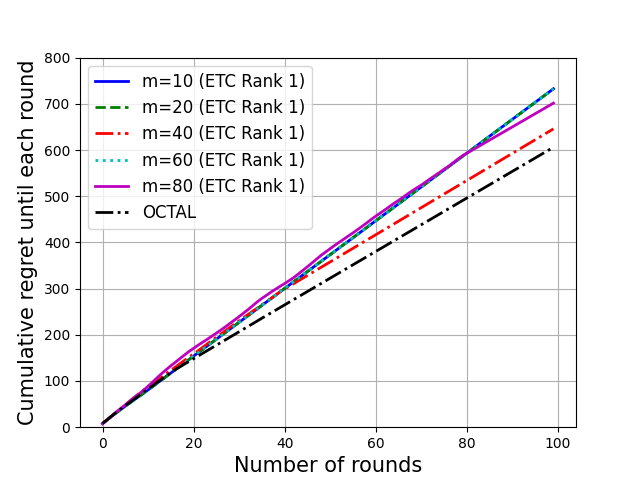}\vspace*{-5pt}
       \caption{Comparison of cumulative regret until every round for OCTAL and ETC with rank-$1$ approximation and different exploration periods}~\label{fig:jester3}
 \end{subfigure}%
 \caption{Comparison of regret incurred by OCTAL algorithm (Algorithm~\ref{algo:phased_elim}, ETC algorithm  \ref{algo:p1} with rank-$1$ approximation of the estimated matrix against UCB for the Jester dataset. In Figure \ref{fig:jester1}, we show that OCTAL and both versions of ETC have significantly lower regret than baseline UCB method. In Figure \ref{fig:jester2}, we show that the regret of OCTAL decreases in every phase unlike ETC algorithms. In Figure \ref{fig:jester3}, we show that the cumulative regret of OCTAL is lower than that of ETC with rank $1$ approximation at every round $t\le \s{T}$.}
\end{figure*}

As in the Movielens setting, we run Algorithm \ref{algo:p1} namely the ETC Algorithm (with rank-$1$ approximation) for different exploration periods and the OCTAL algorithm (Algorithm \ref{algo:phased_elim} with the same minor mofications as in the synthetic datasets) in this setting.  Note from Figure \ref{fig:jester1} that both the ETC and OCTAL algorithm have improved performances as compared to the UCB algorithm. Again, from Figure \ref{fig:jester3}, the cumulative regret of OCTAL is lower than that of ETC with rank-$1$ approximation for different exploration periods at every round. In Figure \ref{fig:jester2}, we show the regret at every round incurred by OCTAL and ETC with rank-$1$ approximation; the regret of OCTAL decreases in every phase which is useful in practical recommendation systems.
Recall that the ETC algorithm with rank-$1$ approximation has a poor performance since it suffers from the $\s{T}^{2/3}$ dependence in the worst case. Again, we must point out that it is critical to tune the exploration period in ETC (as a function of rounds and sub-optimality gaps - unknown in practice and therefore difficult)  \citep{lattimore2020bandit}[Ch. 6], OCTAL still suffers lower regret without such side-information.

\paragraph{Book-Crossing Dataset:} Next we consider the Book-Crossing Dataset \citep{ziegler2005improving} which consists of ratings from $278$K users (with demographic information) for $271$K books. As in the Movielens setting, we identify $\s{M}=70$ disjoint clusters of users as a unique combination of country and age bin (0-25,25-50,50-75,75-100) that have rated at least $100$ books. Next we take $\s{N}=150$ books that have the most ratings. For each user cluster and each chosen book, we take the average rating of all users in the cluster for that book. We select the number of rounds $\s{T}=100$ and as before, we run the UCB algorithm separately for each user, ETC algorithm (Alg. \ref{algo:p1}) with the rank-$1$ approximation of the estimated reward matrix, OCTAL Algorthm (Alg. \ref{algo:phased_elim} with the same minor mofications as in the synthetic datasets). The results are demonstrated in Figures \ref{fig:bx1},\ref{fig:bx2} and \ref{fig:bx3}. As before the UCB algorithm has the worst performance; OCTAL has a superior performance to ETC with rank-$1$ approximation for different exploration periods (and at all rounds). The conclusions are very similar to that obtained in the Jester dataset. 

\begin{figure*}[!htbp]
  \begin{subfigure}[t]{0.33\textwidth}
    \centering 
    \includegraphics[scale = 0.3]{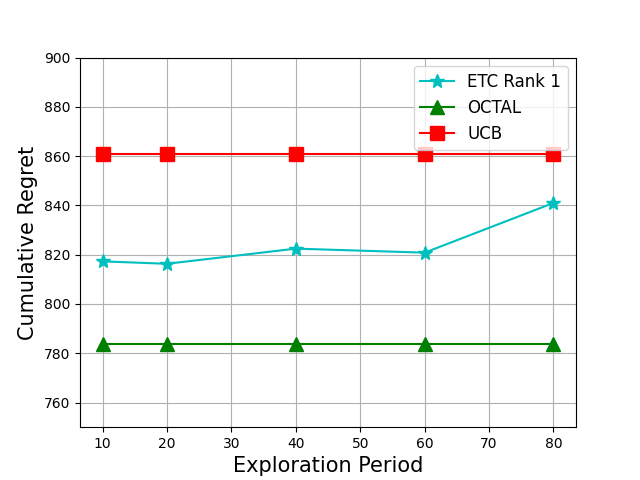}\vspace*{-5pt}
 \caption{Comparison of regret at $\s{T}=100$ for UCB, OCTAL and ETC (with rank-$1$ approxmiation) for different exploration periods}~\label{fig:bx1}
  \end{subfigure}
  \hfill
\begin{subfigure}[t]{0.32 \textwidth}
\centering
   \includegraphics[scale = 0.3]{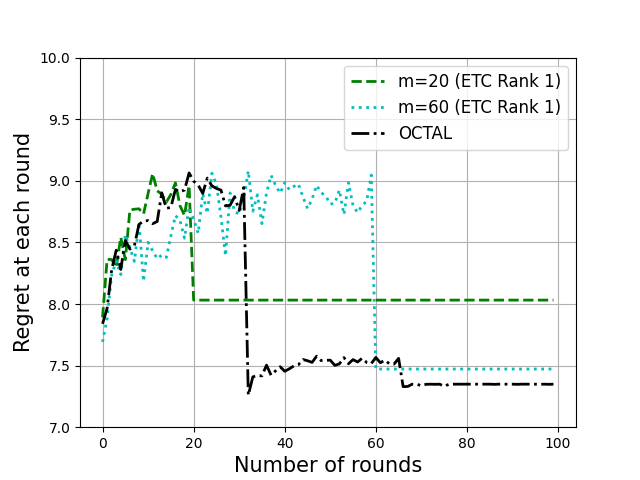}\vspace*{-5pt}
       \caption{Comparison of regret at every round for OCTAL and ETC with rank-$1$  approximation and exploration period $m=20,40$}~\label{fig:bx2}
 \end{subfigure}%
 \hfill
\begin{subfigure}[t]{0.32 \textwidth}
\centering
   \includegraphics[scale = 0.3]{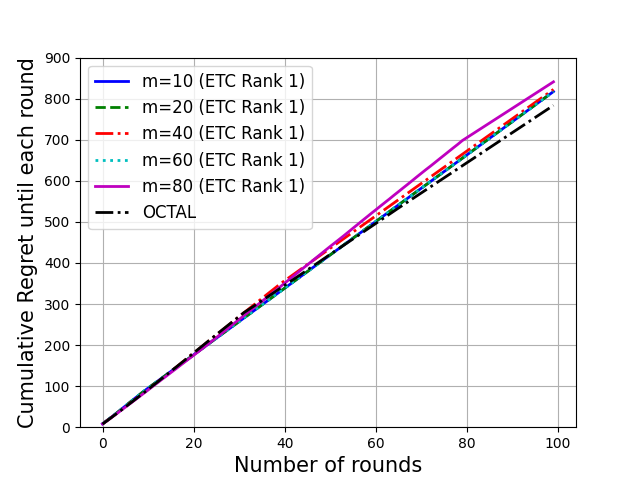}\vspace*{-5pt}
       \caption{Comparison of cumulative regret until every round for OCTAL and ETC with rank-$1$ approximation and different exploration periods}~\label{fig:bx3}
 \end{subfigure}%
 \caption{Comparison of regret incurred by OCTAL algorithm (Algorithm~\ref{algo:phased_elim}, ETC algorithm  \ref{algo:p1} with rank-$1$ approximation of the estimated matrix against UCB for the Book Crosssing dataset. In Figure \ref{fig:bx1}, we show that OCTAL and both versions of ETC have significantly lower regret than baseline UCB method. In Figure \ref{fig:bx2}, we show that the regret of OCTAL decreases in every phase unlike ETC algorithms. In Figure \ref{fig:bx3}, we show that the cumulative regret of OCTAL is lower than that of ETC with rank $1$ approximation at every round $t\le \s{T}$.}
\end{figure*}

\vspace*{-5pt}


\section{Missing Proofs in Section \ref{sec:Preliminaries}}\label{app:preliminaries}

\begin{algorithm}[!t]
\caption{\textsc{Estimate Low Rank Matrix (Sub-matrix) Offline \citep{chen2019noisy}}   \label{algo:estimate_offline_1}}
\begin{algorithmic}[1]
\REQUIRE rows $\ca{U}\subseteq [\s{M}]$, columns $\ca{V}\subseteq[\s{N}]$, noise $\sigma^2$, rank $r$ of $\fl{P}$, bernoulli sampling probability $0\le p \le 1$. 

\STATE Set $d_2=\min(\left|\ca{U}\right|,\left|\ca{V}\right|)$ and $\lambda=C_{\lambda}\sigma\sqrt{d_2p}$ for some constant $C_{\lambda}>0$. 

\STATE For each tuple of indices $(i,j)\in \ca{U}\times \ca{V}$, independently set $\delta_{ij}=1$ with probability $p$ and $\delta_{ij}=0$ with probability $1-p$.

\STATE  Denote $\Omega=\{(i,j)\in \ca{U}\times \ca{V} \mid \delta_{ij}=1\}$. Observe $\fl{Z}_{ij}$ (noisy version of $\fl{P}_{ij}$) for all $(i,j)\in \Omega$

\STATE Solve convex program \vspace*{-15pt}
\begin{align}\label{eq:convex_offline_1}
    \min_{\fl{\widehat{P}}\in \bb{R}^{|\ca{U}|\times|\ca{V}|}} \frac{1}{2}\sum_{(i,j)\in \Omega}\Big(\fl{\widehat{P}}_{ij}-\fl{Z}_{ij}\Big)^2+\lambda\|\widehat{\fl{P}}\|_{\star}
\end{align}
where $\|\widehat{\fl{P}}\|_{\star}$ denotes nuclear norm of matrix $\widehat{\fl{P}}$
\STATE Return  $\widehat{\fl{P}}$.
\end{algorithmic}
\end{algorithm}

We start with the following corollary:

\begin{lemma}[Theorem 1 in \cite{chen2019noisy}]\label{lem:source2}
 Let $\fl{P}=\fl{\bar{U}}\f{\Sigma}\fl{\bar{V}}^{\s{T}}\in \bb{R}^{d\times d}$ such that $\fl{\bar{U}}\in \bb{R}^{d\times r},\fl{\bar{V}}\in \bb{R}^{d \times r}$ and $\f{\Sigma} \triangleq \s{diag}(\lambda_1,\lambda_2,\dots,\lambda_r) \in \bb{R}^{r \times r}$ with $\fl{\bar{U}}^{\s{T}}\fl{\bar{U}}=\fl{\bar{V}}^{\s{T}}\fl{\bar{V}}=\fl{I}$ and $\|\fl{\bar{U}}\|_{2,\infty}\le \sqrt{\mu r/d}, \|\fl{\bar{V}}\|_{2,\infty}\le \sqrt{\mu r/d}$.  
 Let $1\ge p \ge C \mu^2 d^{-1} \log^3 d$ for some sufficiently large constant $C>0$, $\sigma = O\Big(\sqrt{\frac{pd}{\mu^3\log d}}\|\fl{P}\|_{\infty}\Big)$, rank $r=O(1)$ and condition number $\kappa \triangleq \frac{\max_i \lambda_i}{\min_i \lambda_i} = O(1) $. Then, with probability exceeding $ 1-O(d^{-3})$, we can compute a 
  matrix $\widehat{\fl{P}}\in \bb{R}^{\s{M}\times \s{N}}$ by using Algorithm \ref{algo:estimate_offline_1} with parameters ($\ca{U}=[\s{M}],\ca{V}=[\s{N}],\sigma^2,r,p$) s.t.,  
  \begin{align}\label{eq:guarantee_source}
    \|\widehat{\fl{P}}-\fl{P}\|_{\infty} \le O\Big(\frac{\sigma}{\min_i\lambda_{i}}\cdot \sqrt{\frac{\mu d\log d}{p}} \|\fl{P}\|_{\infty}\Big).
  \end{align}
\end{lemma}

Next, we extend Lemma \ref{lem:source2} to rectangular matrices in the following result:

\begin{lemma}\label{thm:randomsample2}
 Let $\fl{P}=\fl{\bar{U}}\f{\Sigma}\fl{\bar{V}}^{\s{T}}\in \bb{R}^{\s{M} \times \s{N}}$ such that $\fl{\bar{U}}\in \bb{R}^{\s{M}\times r},\fl{\bar{V}}\in \bb{R}^{\s{N} \times r}$ and $\f{\Sigma} \triangleq \s{diag}(\lambda_1,\lambda_2,\dots,\lambda_r) \in \bb{R}^{r \times r}$ with $\fl{\bar{U}}^{\s{T}}\fl{\bar{U}}=\fl{\bar{V}}^{\s{T}}\fl{\bar{V}}=\fl{I}$ and $\|\fl{\bar{U}}\|_{2,\infty}\le \sqrt{\mu r/\s{M}}, \|\fl{\bar{V}}\|_{2,\infty}\le \sqrt{\mu r/\s{N}}$. Let $d_1=\max(\s{M},\s{N})$ and $d_2=\min(\s{M},\s{N})$.   
 Let $1 \ge p \ge C \mu^2 d_1d_2^{-2} \log^3 d_1$ for some sufficiently large constant $C>0$, $\sigma = O\Big(\sqrt{\frac{pd_2^{3}}{d_1^{2}\mu^3 \log^3 d_1}}\|\fl{P}\|_{\infty}\Big)$, rank $r=O(1)$ and condition number $\kappa \triangleq \frac{\max_i \sigma_i}{\min_i \sigma_i} = O(1)$. Then, with probability exceeding $ 1-O(d_1^{-3})$, we can compute a 
  matrix $\widehat{\fl{P}}\in \bb{R}^{\s{M}\times \s{N}}$ by using Algorithm \ref{algo:estimate_offline_1} with parameters ($\ca{U}=[\s{M}],\ca{V}=[\s{N}],\sigma^2,r,p$) s.t.,  
  \begin{align}\label{eq:guarantee}
  \|\widehat{\fl{P}}-\fl{P}\|_{\infty} =  O\Big(\frac{\sigma r}{\sqrt{d_2}}\Big(\frac{d_1}{d_2}\Big)^{1/2}\sqrt{\frac{\mu^3 \log d_1}{p}}\Big)
  \end{align}

\end{lemma}

\begin{proof}[Proof of Lemma \ref{thm:randomsample2}]

Without loss of generality, let us assume that the matrix $\fl{P}$ is tall i.e. $\s{M} \ge \s{N}$. Now, 
let us construct the matrix  
\begin{align*}
    \fl{Q} =  
    \begin{bmatrix}
    \fl{P} & \fl{0}_{\s{M}\times \s{M-N}} 
    \end{bmatrix}
    =  \bar{\fl{U}} \f{\Sigma} [\bar{\fl{V}}^{\s{T}} \; \f{0}_{\s{N-M}}^{\s{T}}]
\end{align*}
where $\fl{Q}\in \bb{R}^{\s{M}\times \s{M}}$. 
Clearly, the decomposition $\fl{Q}=\bar{\fl{U}} \f{\Sigma} [\bar{\fl{V}}^{\s{T}} \; \f{0}_{\s{N-M}}^{\s{T}}]$ also coincides with the SVD of $\fl{Q}$ since both matrices $\bar{\fl{U}}$ and $[\bar{\fl{V}}^{\s{T}} \; \f{0}_{\s{N-M}}^{\s{T}}]^{\s{T}}$ are orthonormal matrices while $\Sigma$ remains unchanged. In case when $\s{N}> \s{M}$, we can construct $\fl{Q}$ similarly by vertically stacking $\fl{P}$ with a zero matrix of dimensions $(\s{N-M})\times \s{M}$. Hence, generally speaking, let us denote $d_1=\max(\s{M},\s{N})$ and $d_2=\min(\s{M},\s{N})$.

The matrix $\fl{Q}$ is $\bar{\mu}$-incoherent where $\bar{\mu}r(d_1)^{-1}=\mu r d_2^{-1}$ implying that $\bar{\mu}=\mu d_1/d_2$. Moreover, we also have $\|\fl{Q}\|_{\infty}=\|\fl{P}\|_{\infty}$ implying that $\max_{ij}\left|\fl{P}_{ij}\right|=\max_{ij}\left|\fl{Q}_{ij}\right|$.
Therefore, by invoking Lemma \ref{lem:source},
the sample size must obey 
\begin{align*}
    &p \ge \frac{C\mu^2 d_1}{d_2^2}\log^3 (d_1) \quad \text{and} \quad \sigma = O\Big(\sqrt{\frac{pd_2^{3}}{d_1^{2} \mu^3\log^3 d_1}}\|\fl{P}\|_{\infty}\Big) 
\end{align*}
Then with probability at least $O(d_1^{-3})$, we can recover a  matrix $\widehat{\fl{Q}}$ such that 
\begin{align*}
    \|\widehat{\fl{Q}}-\fl{Q}\|_{\infty} \le O\Big(\frac{\sigma}{\min_i\lambda_i}\sqrt{\frac{\mu \frac{d_1}{d_2} d_1\log d_1}{p}}\|\fl{P}\|_{\infty}\Big).
\end{align*}
Using the fact that $\|\fl{P}\|_{\infty} \le \max_i \lambda_i  \|\bar{\fl{U}}\|_{2,\infty}\|\bar{\fl{V}}\|_{2,\infty} = \max_i \lambda_i  \mu r/\sqrt{d_1 d_2} \implies \|\fl{P}\|_{\infty}/\min_i \lambda_i = \kappa \cdot \mu r/\sqrt{d_1 d_2} = O\Big(\mu r/\sqrt{d_1 d_2}\Big) $ (using the fact that $\kappa = O(1)$ , we obtain a matrix $\widehat{\fl{P}}$ such that 
\begin{align*}
    \|\widehat{\fl{P}}-\fl{P}\|_{\infty} \le O\Big(\frac{\sigma \mu  r}{\sqrt{d_1d_2}}\Big(\frac{d_1}{d_2}\Big)^{1/2}\sqrt{\frac{\mu d_1 \log d_1}{p}}\Big) = O\Big(\frac{\sigma r}{\sqrt{d_2}}\Big(\frac{d_1}{d_2}\Big)^{1/2}\sqrt{\frac{\mu^3 \log d_1}{p}}\Big).
\end{align*}

\end{proof}

\begin{algorithm}[!t]
\caption{\textsc{Estimate Low Rank Matrix (Sub-matrix) Offline - Rectangular Matrices}   \label{algo:estimate_offline_2}}
\begin{algorithmic}[1]
\REQUIRE rows $\ca{U}\subseteq [\s{M}]$, columns $\ca{V}\subseteq[\s{N}]$, noise $\sigma^2$, rank $r$ of $\fl{P}$, bernoulli sampling probability $0\le p \le 1$. 

\STATE Set $d_2=\min(\left|\ca{U}\right|,\left|\ca{V}\right|)$ and $\lambda=C_{\lambda}\sigma\sqrt{d_2p}$ for some constant $C_{\lambda}>0$. 

\STATE For each tuple of indices $(i,j)\in \ca{U}\times \ca{V}$, independently set $\delta_{ij}=1$ with probability $p$ and $\delta_{ij}=0$ with probability $1-p$.

\STATE  Denote $\Omega=\{(i,j)\in \ca{U}\times \ca{V} \mid \delta_{ij}=1\}$. Observe $\fl{Z}_{ij}$ (noisy version of $\fl{P}_{ij}$) for all $(i,j)\in \Omega$

\STATE Without loss of generality, assume $|\ca{U}| \le |\ca{V}|$. For each $i\in \ca{V}$, independently set $\zeta_i$ to be a value in the set $[\lceil|\ca{V}|/|\ca{U}|\rceil]$ uniformly at random. Partition indices in $\ca{V}$ into $\ca{V}^{(1)},\ca{V}^{(2)},\dots,\ca{V}^{(k)}$ where $k=\lceil|\ca{V}|/|\ca{U}|\rceil$ and $\ca{V}^{(q)}=\{i\in \ca{V}\mid \zeta_i =q\}$ for each $q\in [k]$. Set $\Omega^{(q)}\leftarrow \Omega \cap (\ca{U}\times \ca{V}^{(q)})$ for all $q\in [k]$. \#\textit{If $|\ca{U}| \ge |\ca{V}|$, we partition the indices in $\ca{U}$}. 

\FOR{$q\in [k]$}
\STATE Solve convex program \vspace*{-15pt}
\begin{align}\label{eq:convex_offline_2}
    \min_{\widehat{\fl{P}}^{(q)}\in \bb{R}^{|\ca{U}|\times|\ca{V}^{(q)}|}} \frac{1}{2}\sum_{(i,j)\in \Omega^{(q)}}\Big(\widehat{\fl{P}}^{(q)}_{i\pi(j)}-\fl{Z}_{ij}\Big)^2+\lambda\|\widehat{\fl{P}}^{(q)}\|_{\star},
\end{align}
where $\|\widehat{\fl{P}}^{(q)}\|_{\star}$ denotes nuclear norm of matrix $\widehat{\fl{P}}^{(q)}$ and $\pi(j)$ is index of $j$ in set $\ca{V}^{(q)}$. 
\ENDFOR
\STATE Return  $\widehat{\fl{P}}\in \bb{R}^{\s{M}\times \s{N}}$ s.t. $\widehat{\fl{P}}_{\ca{U},\ca{V}^{(q)}}=\widehat{\fl{P}}^{(q)}$ for all $q\in [k]$ and for every  $(i,j)\not \in \ca{U}\times \ca{V}$, $\widehat{\fl{P}}_{ij}=0$.
\end{algorithmic}
\end{algorithm}

Here we recall the discussion in Remark \ref{rmk:undesirable}.
From Lemma \ref{thm:randomsample2}, we saw that in the guarantee provided in equation \ref{eq:guarantee}, the entry-wise error guarantee has an undesirable $(d_1/d_2)^{1/2}$ factor on the right hand side.  This is because appending a zero matrix to make the rectangular matrix square  leads to an increase in the incoherence factor by $(d_1/d_2)^{1/2}$ since there are many zero entries. This is what leads to the undesirable factor; but we can get around this by splitting the rectangular matrix randomly into approximately square sub-matrices and completing each of them individually. The improved algorithm for rectangular matrices is described in Algorithm \ref{algo:estimate_offline_2} - note that in Step 4, we split the rectangular matrix into approximately square matrices and in Steps 5-7, we complete each of the approximately square sub-matrix individually. 
Finally, in Step 8, we join the estimated sub-matrices to get an estimate of the entire matrix and return it. In the main paper, the corresponding steps are done in Lines 10-14 in Algorithm \ref{algo:estimate}.
We provide a formal proof in the following lemma:

\begin{lemma}\label{thm:randomsample3}
 Let the matrix $\fl{P}\in \bb{R}^{\s{M}\times \s{N}}$ satisfy the conditions as stated in Lemma \ref{thm:randomsample2}. Let $d_1=\max(\s{M},\s{N})$, $d_2 = \min(\s{M},\s{N})$ such that $d_2=\Omega(\mu r \log (rd_2))$, $1\ge p \ge C \mu^2 d_2^{-1} \log^3 d_2$ for some sufficiently large constant $C>0$ and $\sigma = O\Big(\sqrt{\frac{pd_2}{\mu^3\log d_2}}\|\fl{P}\|_{\infty}\Big)$. 
 We can use Algorithm \ref{algo:estimate_offline_2} with parameters ($\ca{U}=[\s{M}],\ca{V}=[\s{N}],\sigma^2,r,p$) to compute a matrix $\widehat{\fl{P}}\in \bb{R}^{\s{M}\times \s{N}}$ such that
 for any pair of indices $(i,j)\in [\s{M}]\times[\s{N}]$,
 \begin{align}\label{eq:rectangular}
  |\widehat{\fl{P}}_{ij}-\fl{P}_{ij}|\le O\Big(\frac{\sigma r }{\sqrt{d_2}}\sqrt{\frac{\mu^3\log d_2}{p}}\Big)
 \end{align}
with probability exceeding $1-O(d_2^{-3})$. Again, with probability exceeding $ 1-O(d_1d_2^{-4})$, the output
  matrix $\widehat{\fl{P}}$ also satisfies  
$
   \| \fl{P} - \widehat{\fl{P}}\|_{\infty} \le O\Big(\frac{\sigma r }{\sqrt{d_2}}\sqrt{\frac{\mu^3\log d_2}{p}}\Big).
$
\end{lemma}

\begin{proof}
Let us assume that the matrix $\fl{P}$ is tall i.e. $\s{M} \ge \s{N}$. Now, let us partition the set of rows into $\frac{\s{M}}{\s{N}}$ groups by assigning each group uniformly at random to each row. Notice that the expected number of rows in each group is $\s{N}$ and by using Chernoff bound, the number of rows in each group lies in the interval $[\frac{\s{N}}{2},\frac{3\s{N}}{2}]$ with probability at least $1-2\exp(-\s{N}/12)$. When $\s{M} \le \s{N}$, we partition the set of columns in a similar manner into $\s{N}/\s{M}$ groups so the number of columns in each group lies in the interval $[\frac{\s{M}}{2},\frac{3\s{M}}{2}]$ with probability at least $1-2\exp(-\s{M}/12)$.
 Hence, generally speaking, let us denote $d_1=\max(\s{M},\s{N})$ and $d_2=\min(\s{M},\s{N})$; we constructed $d_1/d_2$ sub-matrices of $\fl{P}$ denoted by $\fl{P}^{(1)},\fl{P}^{(2)},\dots,\fl{P}^{(d_1/d_2)}$.
 With probability at least $1-2\exp(-d_2/12)$, the matrices $\fl{P}^{(1)},\fl{P}^{(2)},\dots,\fl{P}^{(d_1/d_2)}$ are approximately square i.e. the ratio of their dimensions is a constant and lie in the interval $[1/2,3/2]$.
 
 Let us analyze the guarantees on estimating  $\fl{P}^{(1)}$. The analysis for other matrices follow along similar lines. Note that $\fl{P}^{(1)}=\fl{U}\f{\Sigma}\fl{V}_{\s{sub}}^{\s{T}}$ where $\fl{V}_{\s{sub}}$ denotes the $\s{M}'\times r$ matrix where the rows in $\fl{V}_{\s{sub}}$ corresponds to the $\s{M}'$ rows in $\fl{V}$ assigned to $\fl{P}^{(1)}$. 
 
 First we will bound from below the minimum eigenvalue of the matrix $\fl{V}_{\s{sub}}^{\s{T}}\fl{V}_{\s{sub}}$. Note that every row of $\fl{V}$ is independently sampled  with probability $p \triangleq d_2/d_1$ for the matrix $\fl{P}^{(1)}$. Hence, we have 
 \begin{align*}
     \frac{1}{p}\fl{V}_{\s{sub}}^{\s{T}}\fl{V}_{\s{sub}} = \frac{1}{p}\sum_{i\in [d_1]} \delta_i\fl{V}_i\fl{V}_i^{\s{T}} = \sum_{i\in [d_1]} \fl{W}^{(i)}
 \end{align*}
 where $\delta_i$ denotes the indicator random variable which is true when $\fl{V}_i$ (the $i^{\s{th}}$ row of $\fl{V}$) is chosen for $\fl{P}^{(1)}$ and $\fl{W}^{(i)}=\frac{1}{p} \delta_i\fl{V}_i\fl{V}_i^{\s{T}}$. Notice that the random matrices $\fl{W}^{(i)}$ are independent with $\bb{E}\fl{W}^{(i)}=\fl{V}_i\fl{V}_i^{\s{T}}$. Hence we define $\fl{Z}^{(i)}=\fl{W}^{(i)}-\bb{E}\fl{W}^{(i)}$ satisfying $\bb{E}\fl{Z}^{(i)}=0$. Moreover, for all $i\in [d_1]$, we have $\|\fl{Z}^{(i)}\|_2\le \Big(1+\frac{1}{p}\Big)\|\fl{V}_i\fl{V}_i^{\s{T}}\|_2 \le \Big(1+\frac{1}{p}\Big) \max_i \|\fl{V}_i\|_2^2 \le \frac{2\mu r}{p d_1}$. Next, we can show the following:
 \begin{align*}
     &\|\sum_{i\in [d_1]}\fl{Z}^{(i)}(\fl{Z}^{(i)})^{\s{T}}\|_2 \le \|\Big(\frac{1}{p}-1\Big)\sum_{i\in [d_1]}(\fl{V}_i\fl{V}_i^{\s{T}})(\fl{V}_i\fl{V}_i^{\s{T}})^{\s{T}}\|_2 \le \|\Big(\frac{1}{p}-1\Big)\|\fl{V}_i\|^2\sum_{i\in [d_1]}(\fl{V}_i\fl{V}_i^{\s{T}})\|_2 \\
     &\le \frac{\mu r}{pd_1} \lambda_{\max}(\fl{V}^{\s{T}}\fl{V}).
 \end{align*}
 Similarly, we will also have 
 \begin{align*}
     \|\sum_{i\in [d_1]}(\fl{Z}^{(i)})^{\s{T}}\fl{Z}^{(i)}\|_2 \le \frac{\mu r}{pd_1} \lambda_{\max}(\fl{V}^{\s{T}}\fl{V}) 
     \le \frac{\mu r}{p d_1}
 \end{align*}
 where we used that $\fl{V}^{\s{T}}\fl{V}$ is orthogonal.
 Therefore, by using Bernstein's inequality for matrices (Theorem 1.6 in \cite{tropp2012user}), we have with probability at least $1-\delta$,
 \begin{align*}
     \|\frac{1}{p}\fl{V}_{\s{sub}}^{\s{T}}\fl{V}_{\s{sub}}-\fl{V}^{\s{T}}\fl{V}\| \le \frac{2\mu r}{3pd_1}\log \frac{2r}{\delta}+\sqrt{\frac{\mu r}{pd_1} \log \frac{2r}{\delta}}.
 \end{align*}
 Hence, by using Weyl's inequality, we will have with probability $1-\delta$
 \begin{align*}
     \lambda_{\min}(\fl{V}_{\s{sub}}^{\s{T}}\fl{V}_{\s{sub}}) \ge p-\frac{2\mu r}{3d_1}\log \frac{2r}{\delta}-\sqrt{\frac{p\mu r}{d_1} \log \frac{2r}{\delta}}.
 \end{align*}
 Hence, by substituting $\delta=d_2^{-3}$ we have that with probability at least $1-d_2^{-3}$, if $d_2=\Omega(\mu r \log (rd_2))$, then  $\lambda_{\min}(\fl{V}_{\s{sub}}^{\s{T}}\fl{V}_{\s{sub}}) \ge p/2$ (since $p=d_1/d_2$) implying that $\fl{V}_{\s{sub}}^{\s{T}}\fl{V}_{\s{sub}}$ is invertible. 
 Also, under the same condition, note that we can show similarly that $\lambda_{\max}(\fl{V}_{\s{sub}}^{\s{T}}\fl{V}_{\s{sub}}) \le 3p/2$ implying that the condition number of each sub-matrix also stays $O(1)$ with probability at least $1-d_2^{-3}$. Now, note that $\fl{V}_{\s{sub}}$ is not orthogonal and therefore, we have
 \begin{align*}
     \fl{P}^{(1)}= \fl{U}\f{\Sigma}(\fl{V}_{\s{sub}}^{\s{T}}\fl{V}_{\s{sub}})^{1/2}(\fl{V}_{\s{sub}}^{\s{T}}\fl{V}_{\s{sub}})^{-1/2}\fl{V}_{\s{sub}}^{\s{T}} = \fl{U}\widehat{\fl{U}}\widehat{\f{\Sigma}}\widehat{\fl{V}}(\fl{V}_{\s{sub}}^{\s{T}}\fl{V}_{\s{sub}})^{-1/2}\fl{V}_{\s{sub}}^{\s{T}}
 \end{align*}
where $\widehat{\fl{U}}\widehat{\Sigma}\widehat{\fl{V}}$ is the SVD of the matrix $\f{\Sigma}(\fl{V}_{\s{sub}}^{\s{T}}\fl{V}_{\s{sub}})^{1/2}$.
Since $\widehat{\fl{U}}$ is orthogonal, $\widehat{\fl{U}}\fl{U}$ is orthogonal as well. Similarly, $(\widehat{\fl{V}}(\fl{V}_{\s{sub}}^{\s{T}}\fl{V}_{\s{sub}})^{-1/2}\fl{V}_{\s{sub}}^{\s{T}})^{\s{T}}$ is orthogonal as well whereas $\widehat{\f{\Sigma}}$ is diagonal. Hence $\fl{U}\widehat{\fl{U}}\widehat{\f{\Sigma}}\widehat{\fl{V}}(\fl{V}_{\s{sub}}^{\s{T}}\fl{V}_{\s{sub}})^{-1/2}\fl{V}_{\s{sub}}^{\s{T}}$ indeed corresponds to the SVD of $\fl{P}^{(1)}$ and we only need to argue about the incoherence of $\widehat{\fl{U}}\fl{U}$ and $\widehat{\fl{V}}(\fl{V}_{\s{sub}}^{\s{T}}\fl{V}_{\s{sub}})^{-1/2}\fl{V}_{\s{sub}}^{\s{T}}$. Notice that $\max_i \|(\fl{U}\widehat{\fl{U}})^{\s{T}}\fl{e}_i\|\le \|\fl{U}^{\s{T}}\fl{e}_i\|\le \sqrt{\frac{\mu r}{d_2}}$. On the other hand, 
\begin{align*}
    &\max_i \|\widehat{\fl{V}}(\fl{V}_{\s{sub}}^{\s{T}}\fl{V}_{\s{sub}})^{-1/2}\fl{V}_{\s{sub}}^{\s{T}}\fl{e}_i\| \le \max_i \|(\fl{V}_{\s{sub}}^{\s{T}}\fl{V}_{\s{sub}})^{-1/2}\fl{V}_{\s{sub}}^{\s{T}}\fl{e}_i\| \\
    &\le \frac{\|\fl{V}_{\s{sub}}\|_{2,\infty}}{\sqrt{\lambda_{\min}(\fl{V}_{\s{sub}}^{\s{T}}\fl{V}_{\s{sub}})}} \le \frac{\|\fl{V}\|_{2,\infty}}{\sqrt{p/2}} \le \sqrt{\frac{2\mu r}{d_2}}.
\end{align*}

Hence, with probability $1-2\exp(-d_2/12)-O(d_2^{-3})$, we can recover an estimate $\widehat{\fl{P}}^{(1)}$ and apply Lemma \ref{thm:randomsample2} to conclude that (recall that the rank of $\fl{P}^{(1)}$ is $O(1)$, we have proved that the condition number of $\fl{P}^{(1)}$ is $O(1)$ w.p. at least $1-O(d_2^{-3})$, the incoherence factor of $\fl{P}^{(1)}$ has increased by a factor of at most $2$ w.p. at least $1-O(d_2^{-3})$ and the ratio of the dimensions of $\fl{P}^{(1)}$ is in the interval $[1/2,3/2]$ w.p. at least $1-\exp(-d_2/12)$ implying that the conditions of Lemma \ref{thm:randomsample2} are satisfied for $\fl{P}^{(1)}$ w.h.p)
\begin{align*}
   \| \fl{P}^{(1)} - \widehat{\fl{P}}^{(1)}\|_{\infty} \le O\Big(\frac{\sigma r}{\sqrt{d_2}}\sqrt{\frac{\mu^3\log d_2}{p}}\Big).
\end{align*}
as long as the conditions stated in the Lemma are satisfied.
This implies that for any pair of indices $(i,j)\in [\s{M}]\times[\s{N}]$, we must have
 \begin{align}\label{eq:rectangular_2}
  |\widehat{\fl{P}}_{ij}-\fl{P}_{ij}|\le O\Big(\frac{\sigma r }{\sqrt{d_2}}\sqrt{\frac{\mu^3\log d_2}{p}}\Big)
 \end{align}
with probability exceeding $1-O(d_2^{-3})$

Finally, by taking a union bound over all the $d_1/d_2$ partitions, we can compute estimates of all the sub-matrices $\fl{P}^{(1)},\fl{P}^{(2)},\dots,\fl{P}^{(d_1/d_2)}$ that have similar guarantees as above with probability at least $1-2d_1d_2^{-1}\exp(-d_2/12)-O(d_1d_2^{-4})$. Hence, by combining all the estimates, we can obtain a final estimate $\widehat{\fl{P}}$ of the matrix $\fl{P}$ that satisfies
\begin{align*}
   \| \fl{P} - \widehat{\fl{P}}\|_{\infty} \le O\Big(\frac{\sigma r}{\sqrt{d_2}}\sqrt{\frac{\mu^3\log d_2}{p}}\Big).
\end{align*}
\end{proof}

 Notice that the theoretical guarantees presented in Lemmas \ref{lem:source2}, \ref{thm:randomsample2} and \ref{thm:randomsample3} hold for the offline low rank matrix completion problem when each (noisy) entry of the low rank matrix $\fl{P}\in \bb{R}^{\s{M}\times \s{N}}$ is observed independently with some probability $p$ (Bernoulli sampling model with probability $p$ - see eq. \ref{eq:obs_bernoulli}). To summarize, if the observed set of entries is $\Omega \subseteq [\s{M}]\times [\s{N}]$ (each index $(i,j)\in [\s{M}]\times [\s{N}]$ is present in $\Omega$ with probability $p$), then we can solve several convex optimization problems (see Steps 5-7 in Algorithm \ref{algo:estimate_offline_2}) restricted to observations in $\Omega$ to compute an estimate $\widehat{\fl{P}}$ of the low rank matrix $\fl{P}$.
 
However, in our problem setting (Section \ref{subsec:prob_defn}), recall that in each round $t\in [\s{T}]$, for each user $u\in [\s{M}]$, some item $\rho_u(t)$ is recommended. Our goal is to translate the theoretical guarantees under Bernoulli sampling model with probability $p$ (for some pre-determined $p$) to our setting. A simple approach (as mentioned in Remark \ref{rmk:observation}) is the following (described in Algorithm \ref{algo:estimate_1}): we first sample a set of indices $\Omega\subseteq [\s{M}]\times [\s{N}]$  such that each index $(i,j)\in [\s{M}]\times [\s{N}]$ is present in $\Omega$ with probability $p$ (Step 2 in Algorithm \ref{algo:estimate_1}). Subsequently, we aim to obtain a single observation (noisy) corresponding to each pair of indices in $\Omega$; in each round, for each user $i\in [\s{M}]$, we recommend an item $j\in [\s{N}]$ such that $(i,j)\in \Omega$ and has not been observed yet . If such an item is unavailable for the user $i$, we recommend any arbitrary item $j$ such that $(i,j)\not\in\Omega$ (Step 5 in Algorithm \ref{algo:estimate_1}) but ignore the observation while reconstruction of the reward matrix; note that this is unnecessary in practice. Once we have obtained observations corresponding to all pairs of indices in $\Omega$ (we discard all observations corresponding to entries not in $\Omega$), we can solve the convex optimization problems restricted to the observed entries in $\Omega$ (Steps 8-10 in Algorithm \ref{algo:estimate_1}) to compute an estimate $\widehat{\fl{P}}$ of the reward matrix $\fl{P}$ and apply the theoretical guarantees in Lemma \ref{thm:randomsample3} directly. The total number of rounds needed will be $\max_{i\in [\s{M}]}\left|(i,j)\in \Omega \mid j \in [\s{N}] \right|$ i.e. the maximum number of indices in a particular row present in $\Omega$ (which we will show to be bounded from above with high probability in the following corollary). This is because we wait until we have one noisy observation for every pair of indices in $\Omega$.

In Algorithm \ref{algo:estimate} in the main paper, the aforementioned procedure is incorporated in Steps 10-14.

\begin{algorithm}[!t]
\caption{\textsc{Estimate Reward Matrix (Sub-matrix) Online}   \label{algo:estimate_1}}
\begin{algorithmic}[1]
\REQUIRE users $\ca{U}\subseteq [\s{M}]$, items $\ca{V}\subseteq[\s{N}]$, noise $\sigma^2$, rank $r$ of $\fl{P}$, bernoulli sampling probability $0\le p \le 1$. Index of round $t$ is relative to the first round when the algorithm is invoked; hence $t=1,2,\dots$.

\STATE Set $d_2=\min(\left|\ca{U}\right|,\left|\ca{V}\right|)$ and $\lambda=C_{\lambda}\sigma\sqrt{d_2p}$ for some constant $C_{\lambda}>0$. 

\STATE For each tuple of indices $(i,j)\in [\s{M}]\times [\s{N}]$, independently set $\delta_{ij}=1$ with probability $p$ and $\delta_{ij}=0$ with probability $1-p$.

\STATE  Denote $\Omega=\{(i,j)\in \ca{U}\times \ca{V} \mid \delta_{ij}=1\}$ and
 $m=\max_{i \in [\s{M}]}\mid |j \in \ca{V}\mid (i,j) \in \Omega|$ to be the maximum number of index tuples in a particular row. For all $(i,j)\in \Omega$, set $\s{Mask}_{ij}=0$.

\FOR{rounds $t=1,2,\dots,m$}
 \STATE For each user $u\in \ca{U}$, recommend an item $\rho_u(t)$ in $\{j \in \ca{V}\mid (u,j)\in \Omega, \s{Mask}_{uj}=0\}$ and set $\s{Mask}_{u\rho_u(t)}=1$. 
If not possible then recommend any item $\rho_u(t)$ in $\ca{V}$ s.t. $(u,\rho_u(t))\not\in\Omega$.
Observe $\fl{R}^{(t)}_{u\rho_u(t)}$.
\ENDFOR

\STATE Without loss of generality, assume $|\ca{U}| \le |\ca{V}|$. For each $i\in \ca{V}$, independently set $\zeta_i$ to be a value in the set $[\lceil|\ca{V}|/|\ca{U}|\rceil]$ uniformly at random. Partition indices in $\ca{V}$ into $\ca{V}^{(1)},\ca{V}^{(2)},\dots,\ca{V}^{(k)}$ where $k=\lceil|\ca{V}|/|\ca{U}|\rceil$ and $\ca{V}^{(q)}=\{i\in \ca{V}\mid \zeta_i =q\}$ for each $q\in [k]$. Set $\Omega^{(q)}\leftarrow \Omega \cap (\ca{U}\times \ca{V}^{(q)})$ for all $q\in [k]$. \#\textit{If $|\ca{U}| \ge |\ca{V}|$, we partition the indices in $\ca{U}$}. 

\FOR{$q\in [k]$}
\STATE Solve convex program \vspace*{-15pt}
\begin{align}\label{eq:convex_3}
    \min_{\widehat{\fl{P}}^{(q)}\in \bb{R}^{|\ca{U}|\times|\ca{V}^{(q)}|}} \frac{1}{2}\sum_{(i,j)\in \Omega^{(q)}}\Big(\widehat{\fl{P}}^{(q)}_{i\pi(j)}-\fl{Z}_{ij}\Big)^2+\lambda\|\widehat{\fl{P}}^{(q)}\|_{\star},
\end{align}
where $\|\widehat{\fl{P}}^{(q)}\|_{\star}$ denotes nuclear norm of matrix $\widehat{\fl{P}}^{(q)}$ and $\pi(j)$ is index of $j$ in set $\ca{V}^{(q)}$. 
\ENDFOR
\STATE Return  $\widehat{\fl{P}}\in \bb{R}^{\s{M}\times \s{N}}$ s.t. $\widehat{\fl{P}}_{\ca{U},\ca{V}^{(q)}}=\widehat{\fl{P}}^{(q)}$ for all $q\in [k]$ and for every  $(i,j)\not \in \ca{U}\times \ca{V}$, $\widehat{\fl{P}}_{ij}=0$.

\end{algorithmic}
\end{algorithm}

\begin{coro}\label{coro:obs3}
Consider a set $\Omega\subseteq [\s{M}]\times[\s{N}]$ of indices such that every index $(i,j)\in [\s{M}]\times[\s{N}]$
is present in $\Omega$ independently with probability $p$.
Consider algorithm \ref{algo:estimate_1} with parameters ($\ca{U}=[\s{M}],\ca{V}=[\s{N}],\sigma^2,r,p$) that recommends items to users according to Step $5$ from the set $\Omega$. Suppose the rank $r$ reward matrix $\fl{P}$ and parameters $p,\sigma$ satisfies the conditions stated in Lemma \ref{thm:randomsample3}.
In that case, using $m=O\Big(\s{N}p+\sqrt{\s{N}p\log \s{M}\delta^{-1}}\Big)$ rounds, Algorithm \ref{algo:estimate_1} is able to compute a matrix $\widehat{\fl{P}}$ such that for any $(i,j)\in [\s{M}]\times[\s{N}]$, we have with probability exceeding $ 1-\delta-O(d_2^{-3})$
 \begin{align}
  |\widehat{\fl{P}}_{ij}-\fl{P}_{ij}|\le O\Big(\frac{\sigma r }{\sqrt{d_2}}\sqrt{\frac{\mu^3\log d_2}{p}}\Big).
 \end{align}
\end{coro}

\begin{proof}[Proof of Corollary \ref{coro:obs3}]

 Recall that $d_1=\max(\s{M},\s{N})$ and $d_2=\min(\s{M},\s{N})$.
Suppose, with some parameter $p=\Omega(\mu^2d_2^{-1}\log ^3 d_2)$, we sample a set $\Omega\in [\s{M}]\times [\s{N}]$ of indices (Step 2 in Algorithm \ref{algo:estimate_1}). 
Let us define the event $\ca{F}_1$ which is true when the maximum number of indices observed in some row $m\triangleq \max_{i\in [\s{M}]}\left|(i,j)\in \Omega \mid j \in [\s{N}] \right|$ satisfies $m=\Omega\Big(\s{N}p+\sqrt{\s{N}p\log \s{M}\delta^{-1}}\Big)$. Recall that our goal is to obtain a single noisy reward observation for every (user,item) pair whose corresponding indices is present in $\Omega$. In Step 5 of Algorithm \ref{algo:estimate_1}, in each round, one item is recommended to every user -clearly the total number of rounds needed to achieve our goal is $m$.

We will bound the probability of the event $\ca{F}_1$ from above by using Chernoff bound. Let us denote the number of items observed for user $i\in [\s{M}]$ to be $Y_i$ i.e. $Y_i= \left|(i,j)\in \Omega \mid j \in [\s{N}]\right|$. Algorithm $\ca{A}$ can then obtain the noisy entries of $\fl{P}$ corresponding to the set $\Omega$ by doing the following: in each round, for each user $i\in [\s{M}]$, if there is an unobserved tuple of indices $(i,j)\in \Omega$, then $\ca{A}$ will recommend $j$ to user $i$ and obtain an noisy observation $\fl{P}_{ij}+\fl{E}_{ij}$; on the other hand, if there no unobserved entry, then $\ca{A}$ will simply recommend a random item $j$ such that $(i,j) \not \in \Omega$ (Step 5 in Algorithm \ref{algo:estimate_1}). 
Subequently, Algorithm \ref{algo:estimate_1} discards the observations corresponding to indices not in $\Omega$ and only utilizes the observations in $\Omega$ to compute an estimate of the reward matrix $\fl{P}$.
Notice that each of the random variables $Y_1,Y_2,\dots,Y_{\s{M}} \sim \s{Binomial}(\s{N},p)$ and are independent. By using Chernoff bound, we have that for each $i\in [\s{M}]$,
\begin{align*}
    &\Pr\Big(\cup_{i\in [\s{M}]}\left|Y_i-\s{N}p\right|\ge \s{N}p\epsilon\Big) \le 2\s{M}\exp\Big(-\frac{\epsilon^2\s{N}p}{3}\Big) \\
    &\implies Y_i \le \s{N}p+O\Big(\sqrt{\s{N}p\log \s{M}\delta^{-1}}\Big) \text{ for all }i\in[\s{M}] 
\end{align*}
with probability $1-\delta$ implying that $\Pr(\ca{F}_1)\le \delta$. Let $\ca{F}_2$ be the event when the recovered matrix $\widehat{\fl{P}}$ does not satisfy the guarantee on $\|\widehat{\fl{P}}-\fl{P}\|_{\infty}$ as stated in Lemma \ref{thm:randomsample3} equation \ref{eq:rectangular} given a set of observed indices $\Omega$ sampled according to equation \ref{eq:obs_bernoulli}. From Lemma \ref{thm:randomsample3}, we know that $\Pr(\ca{F}_2)=O(d_2^{-3})$ where $d_1=\max(\s{M},\s{N})$. Hence we can conclude $\Pr(\ca{F}_1 \cup \ca{F}_2) \le  \Pr(\ca{F}_1)+\Pr(\ca{F}_2) = \delta+O(d_2^{-3})$. Note that conditioned on the event that $\ca{F}_1$ does not hold true, we need only $m=O\Big(\s{N}p+\sqrt{\s{N}p\log \s{M}\delta^{-1}}\Big)$ rounds to obtain a set of noisy observations corresponding to $\Omega' \supseteq \Omega$. 

Hence, to conclude, with only $m=O\Big(\s{N}p+\sqrt{\s{N}p\log \s{M}\delta^{-1}}\Big)$ rounds, Alg. \ref{algo:estimate_1} can obtain a single noisy reward observation corresponding to each set of indices in $\Omega$ (where every tuple of indices is present with probability $p$); note that observations for indices outside $\Omega$ are discarded completely. Subsequently, the noisy reward observations in $\Omega$ are used to compute an estimate $\widehat{\fl{P}}$ of the reward matrix $\fl{P}$ with the (Steps 8-11 in Algorithm \ref{algo:estimate_1}) theoretical guarantees presented in Lemma \ref{thm:randomsample3} holding true.
This completes the proof of the corollary. 
\end{proof}

\begin{rmk}\label{rmk:failure_prob}[Remark 3 in \cite{chen2019noisy}]
Note the failure probability in Corollary \ref{coro:obs3}, Lemmas \ref{thm:randomsample3}, \ref{thm:randomsample2} and \ref{lem:source2}  is $O(d_2^{-3})$. However the constant $3$ can be replaced by any arbitrary constant $c$ for example $c=100$ without any change in the guarantees on $\|\widehat{\fl{P}}-\fl{P}\|_{\infty}$. Hence, the guarantees presented in Lemma \ref{thm:randomsample3} and Corollary \ref{coro:obs3} hold with probability at least $1-O(d_2^{-c})$ for any arbitrary constant $c$.
\end{rmk}

As mentioned in Remark \ref{rmk:median}, the drawbacks of Lemma \ref{thm:randomsample3} and Corollary \ref{coro:obs3} are that the smallest error that is possible to achieve by using Lemma \ref{lem:source} is by substituting $p=1$ and thereby obtaining $\|\widehat{\fl{P}}-\fl{P}\|_{\infty} \le \zeta=O\Big(\sigma(\min_i\lambda_{i})^{-1}\cdot \sqrt{\mu d\log d} \|\fl{P}\|_{\infty}\Big)$ and moreover, the probability of failure is polynomially small in $d_2$ (see Remark \ref{rmk:failure_prob}). The former issue is substantial when we need an estimate $\widehat{\fl{P}}$ of the low rank matrix $\fl{P}$ with a smaller entry-wise estimation error than $\zeta$. The second issue is substantial when the shorter dimension $d_2$ is small. We provide an improved algorithm (Algorithm \ref{algo:estimate_2}) with technical modifications to fix the two aforementioned issues. Below, we describe our main ideas at a high level first:  

\noindent \textbf{Repeated recommendations for small estimation error:} Recall that our strategy for computing an estimate $\widehat{\fl{P}}$ with theoretical guarantees was the following two-step procedure 1) obtain a subset of indices $\Omega\subseteq [\s{M}]\times [\s{N}]$ such that each index $(i,j)$ is present with probability $p>0$ 2) obtain noisy reward observations corresponding to all indices ((user,item) pairs) in $\Omega$ - in each round, for each user $i\in [\s{M}]$ we recommended an item $j\in [\s{N}]$ such that $(i,j)\in \Omega$ if such a $j$ exists otherwise we recommend any item $j\in [\s{N}]$ such that $(i,j)\not\in \Omega$.

In our problem, an algorithm has the flexibility of recommending an item more than once to the same user. Suppose we have sampled a subset of indices $\Omega\subseteq [\s{M}]\times [\s{N}]$ as described above and is fixed.
We modify the second step in the following way: for some fixed $s>0$, our modified aim is to obtain $s$ noisy observations corresponding to all entries ((user,item) pairs) in $\Omega$. To do so, we repeat the process of recommending items to users corresponding to indices in $\Omega$ $s$ times (See Step 5 in Algorithm \ref{algo:estimate_2}). Clearly, if the total number of rounds required to obtain noisy reward observations corresponding to  all pairs of indices in $\Omega$ once is $m$ (bounded from above w.h.p - see Corollary \ref{coro:obs3}), then the total number of rounds required for obtaining $s$ observations corresponding to each entry in $\Omega$ is $ms$.
As before, all reward observations outside the indices in $\Omega$ are discarded. For each entry $(i,j)\in \Omega$, we take $\fl{Z}_{ij}$ to be the average of the $s$ reward observations corresponding to the recommendation of item $j$ to user $u$- clearly, the variance proxy of the averaged observation is $\sigma^2/s$. Thus, we can again solve similar convex optimization problems as before (see Steps 13-15 in Algorithm \ref{algo:estimate_2}) 
to obtain the same theoretical guarantee as in Corollary \ref{coro:obs3} but with noise variance $\sigma^2/s$. This intuition is formalized in Lemma \ref{lem:min_acc} below. In the main paper, the process of repeatedly recommending the items in a fixed subset of indices $\Omega$ is described in the For Loop in Steps 1-8 in Algorithm \ref{algo:estimate} - note that $\Omega$ is fixed in For Loop in Step 1 and in beginning of every iteration, we restart the recommendation procedure. 

\noindent \textbf{Independent estimates and entry-wise median for small error probability:}

To increase the probability of success, we can compute $f$ independent estimates of $\fl{P}$ (Line 2 and Line 16 in Algorithm \ref{algo:estimate_2}) namely $\widehat{\fl{P}}^{(1)},\widehat{\fl{P}}^{(2)},\dots,\widehat{\fl{P}}^{(f)}$  and compute their entry-wise median i.e. $\widehat{\fl{P}}_{ij}=\s{median}(\widehat{\fl{P}}^{(1)}_{ij},\dots,\widehat{\fl{P}}^{(f)}_{ij})$ for all $(i,j)\in [\s{M}]\times [\s{N}]$.
From our previous argument, the number of rounds for obtaining a single estimate $\widehat{\fl{P}}$ of the matrix $\fl{P}$ is $O(ms)$ (where the noise variance is reduced to $\sigma^2/s$) with high probability. Hence the total number of rounds for obtaining $f$ estimates will be $\widetilde{O}(msf)$ with high probability. If we set $f=O(\log(\s{MN}\delta^{-1}))$ for some fixed $\delta>0$, then by taking a union bound over all $\s{MN}$ entries, we can show that our guarantees hold with probability at least $1-\delta$ (thus we remove the dependence of the failure probability on $d_2)$. Again, this intuition is formalized in Lemma \ref{lem:min_acc} below. In Algorithm \ref{algo:p1}, note that the $f$ independent estimates are computed in Lines 2-6 and the entry-wise median is computed in Line 7. In Algorithm \ref{algo:phased_elim}, the $f$ independent estimates for each sub-matrix are computed in Lines 4-10 and the entry-wise median is computed in Line 15. 

\begin{algorithm}[!t]
\caption{\textsc{Estimate Reward Matrix (Sub-matrix) Online - Improved Algorithm}   \label{algo:estimate_2}}
\begin{algorithmic}[1]
\REQUIRE users $\ca{U}\subseteq [\s{M}]$, items $\ca{V}\subseteq[\s{N}]$, noise $\sigma^2$, rank $r$ of $\fl{P}$, bernoulli sampling probability $0\le p \le 1$, repetitions $s$ and number of independent estimates $f$. Index of round $t$ is relative to the first round when the algorithm is invoked; hence $t=1,2,\dots$.

\STATE Set $d_2=\min(\left|\ca{U}\right|,\left|\ca{V}\right|)$ and $\lambda=C_{\lambda}\sigma\sqrt{d_2p}$ for some constant $C_{\lambda}>0$. 

\FOR{$z=1,2,\dots,f$} 

\STATE For each tuple of indices $(i,j)\in [\s{M}]\times [\s{N}]$, independently set $\delta_{ij}=1$ with probability $p$ and $\delta_{ij}=0$ with probability $1-p$.

\STATE  Denote $\Omega=\{(i,j)\in \ca{U}\times \ca{V} \mid \delta_{ij}=1\}$ and
 $m_{z}=\max_{i \in \ca{U}}\mid |j \in \ca{V}\mid (i,j) \in \Omega|$ to be the maximum number of index tuples in a particular row. 

\FOR{$\ell=1,2,\dots,s$}

\STATE For all $(i,j)\in \Omega$, set $\s{Mask}_{ij}=0$.

\FOR{$\ell'=1,2,\dots,m_z$}
 \STATE For each user $u\in \ca{U}$ in round $t=\sum_{z'=1}^{z-1}sm_{z'}+(\ell-1)s+\ell'$, recommend an item $\rho_u(t)$ in $\{j \in \ca{V}\mid (u,j)\in \Omega, \s{Mask}_{uj}=0\}$ and set $\s{Mask}_{u\rho_u(t)}=1$. 
If not possible then recommend any item $\rho_u(t)$ in $\ca{V}$ s.t. $(u,\rho_u(t))\not\in\Omega$.
Observe $\fl{R}^{(t)}_{u\rho_u(t)}$.
\ENDFOR

\ENDFOR

\STATE For each tuple $(u,j)\in \Omega$, compute $\fl{Z}_{uj}$ to be average of  $\{\fl{R}^{(t)}_{u\rho_u(t)} \text{ for }t\in [\sum_{z'=1}^{z-1}sm_{z'},\sum_{z'=1}^{z}sm_{z'}]\mid \rho_u(t)=j\}$. \# \textit{$\fl{Z}_{uj}$ is the average of $s$ independent observations at each index $(u,j)\in \Omega$ for a fixed $z$}

\STATE Without loss of generality, assume $|\ca{U}| \le |\ca{V}|$. For each $i\in \ca{V}$, independently set $\zeta_i$ to be a value in the set $[\lceil|\ca{V}|/|\ca{U}|\rceil]$ uniformly at random. Partition indices in $\ca{V}$ into $\ca{V}^{(1)},\ca{V}^{(2)},\dots,\ca{V}^{(k)}$ where $k=\lceil|\ca{V}|/|\ca{U}|\rceil$ and $\ca{V}^{(q)}=\{i\in \ca{V}\mid \zeta_i =q\}$ for each $q\in [k]$. Set $\Omega^{(q)}\leftarrow \Omega \cap (\ca{U}\times \ca{V}^{(q)})$ for all $q\in [k]$. \#\textit{If $|\ca{U}| \ge |\ca{V}|$, we partition the indices in $\ca{U}$}. 

\FOR{$q\in [k]$}
\STATE Solve convex program \vspace*{-15pt}
\begin{align}\label{eq:convex_3}
    \min_{\widehat{\fl{P}}^{(q)}\in \bb{R}^{|\ca{U}|\times|\ca{V}^{(q)}|}} \frac{1}{2}\sum_{(i,j)\in \Omega^{(q)}}\Big(\widehat{\fl{P}}^{(q)}_{i\pi(j)}-\fl{Z}_{ij}\Big)^2+\lambda\|\widehat{\fl{P}}^{(q)}\|_{\star},
\end{align}
where $\|\widehat{\fl{P}}^{(q)}\|_{\star}$ denotes nuclear norm of matrix $\widehat{\fl{P}}^{(q)}$ and $\pi(j)$ is index of $j$ in set $\ca{V}^{(q)}$. 
\ENDFOR
\STATE Compute  $\widehat{\fl{P}}^{(z)}\in \bb{R}^{\s{M}\times \s{N}}$ s.t. $\widehat{\fl{P}}^{(z)}_{\ca{U},\ca{V}^{(q)}}=\fl{Q}^{(q)}$ for all $q\in [k]$ and for every  $(i,j)\not \in \ca{U}\times \ca{V}$, $\widehat{\fl{P}}^{(z)}_{ij}=0$. \#\textit{We are computing independent estimates $\widehat{\fl{P}}^{(z)}$}
\ENDFOR
\STATE Obtain estimate $\widehat{\fl{P}}$ by taking the entry-wise median of $\widehat{\fl{P}}^{(1)},\widehat{\fl{P}}^{(2)},\dots,\widehat{\fl{P}}^{(f)}$. Return $\widehat{\fl{P}}$.
\end{algorithmic}
\end{algorithm}

\begin{lemma}[Restatement of Lemma \ref{lem:min_acc}]
Let rank $r=O(1)$ reward matrix $\fl{P}\in \bb{R}^{\s{M} \times \s{N}}$ with SVD decomposition $\fl{P}=\fl{\bar{U}}\f{\Sigma}\fl{\bar{V}}^{\s{T}}$ satisfy  $\|\fl{\bar{U}}\|_{2,\infty}\le \sqrt{\mu r/\s{M}}, \|\fl{\bar{V}}\|_{2,\infty}\le \sqrt{\mu r/\s{N}}$ and condition number $\kappa = O(1)$. Let $d_1=\max(\s{M},\s{N})$ and $d_2=\min(\s{M},\s{N})$ such that  
 $1 \ge p \ge C\mu^2d_2^{-1} \log^3 d_2$ for sufficiently large constant $C>0$. 
 Suppose we observe noisy entries of $\fl{P}$ according to observation model in (\ref{eq:obs}).
For any positive integer $s>0$ satisfying $\frac{\sigma}{\sqrt{s}}=O\Big(\sqrt{\frac{pd_2}{\mu^3\log d_2}}\|\fl{P}\|_{\infty}\Big)$, there exists an algorithm $\ca{A}$ (See Algorithm \ref{algo:estimate_2} with input parameters $\ca{U}=[\s{M}],\ca{V}=\s{N},\sigma^2,r,p,s,f$) with parameters $s,p,\sigma$ that uses $m=O\Big(s\log (\s{MN}\delta^{-1})(\s{N}p+\sqrt{\s{N}p\log \s{M}\delta^{-1}})\Big)$ rounds to recommend items for users and compute a matrix $\widehat{\fl{P}}$ such that with probability exceeding $ 1-O(\delta\log (\s{MN}\delta^{-1}))$, we have  
\begin{align}\label{eq:final}
   \| \fl{P} - \widehat{\fl{P}}\|_{\infty} \le O\Big(\frac{\sigma r }{\sqrt{sd_2}}\sqrt{\frac{\mu^3\log d_2}{p}}\Big).
\end{align}
\end{lemma}

\begin{proof}
\begin{claim}
Fix any index $(i,j)\in [\s{M}]\times [\s{N}]$. Then, if the conditions in the Lemma statement are satisfied, for any $z\in [f]$ in Line 2 of Alg. \ref{algo:estimate_2}, the computed estimate $\widehat{\fl{P}}^{(z)}$ in Line 16 in Alg. \ref{algo:estimate_2} satisfies the following: 1)  with probability exceeding $9/10$, we have 
 \begin{align*}
  |\widehat{\fl{P}}^{(z)}_{ij}-\fl{P}_{ij}|\le O\Big(\frac{\sigma r }{\sqrt{sd_2}}\sqrt{\frac{\mu^3\log d_2}{p}}\Big)
 \end{align*}
 2) the total number of rounds required to compute $\widehat{\fl{P}}^{(z)}$ is given by $sm_z = O\Big(s(\s{N}p+\sqrt{\s{N}p\log(\s{M}\delta^{-1})})\Big)$.
\end{claim}

\begin{proof}

\noindent \textbf{Repeated Recommendations for small estimation error:}
Consider Lines 5-10 in Algorithm \ref{algo:estimate_2} prior to which we have a fixed subset of indices $\Omega\subseteq [\s{M}]\times [\s{N}]$ where every index tuple $(i,j)$ is present in $\Omega$ with probability $p$. In Line 8, for each user $i\in [\s{M}]$, Algorithm \ref{algo:estimate_2} recommends an item $j\in [\s{N}]$ such that $(i,j)\in \Omega$ if such a $j$ exists otherwise we recommend any item $j\in [\s{N}]$ such that $(i,j)\not\in \Omega$. Because of the For Loop in Line 5, Algorithm \ref{algo:estimate_2} does the aforementioned step $s$ times for the same fixed subset of indices $\Omega$. After discarding all observations corresponding to indices outside $\Omega$, at the end of the For Loop in Line 10, we have $s$ noisy iid observations (sub-gaussian random variables) $\fl{R}_{ij}^{(t_1)},\fl{R}_{ij}^{(t_2)},\dots,\fl{R}_{ij}^{(t_s)}$ each with expectation $\fl{P}_{ij}$ and variance proxy $\sigma^2$ i.e. we have for all $h\in [s]$
\begin{align*}
    \bb{E}\exp\Big(\lambda(\fl{R}_{ij}^{(t_h)}-\fl{P}_{ij})\Big) \le \exp(\sigma^2\lambda^2/2) \; \forall \lambda\in \bb{R}.
\end{align*}
In Line 11 in Algorithm \ref{algo:estimate_2}, we take $\fl{Z}_{ij}=s^{-1}\sum_{h=1}^{s} \fl{R}_{ij}^{(t_h)}$. Clearly, we have $\bb{E}\fl{Z}_{ij} = \fl{P}_{ij}$ and further,
$\fl{Z}_{ij}$ has a variance proxy of $\sigma^2/s$ as shown below:
\begin{align*}
    &\bb{E}\exp\Big(\lambda(\fl{Z}_{ij}-\fl{P}_{ij})\Big) = \bb{E} \prod_{h=1}^{s}\exp\Big(\lambda s^{-1}(\fl{R}_{ij}^{(t_h)}-\fl{P}_{ij})\Big) \\
    &=\prod_{h=1}^{s} \bb{E}\exp\Big(\lambda s^{-1}(\fl{R}_{ij}^{(t_h)}-\fl{P}_{ij})\Big) \le \exp(\sigma^2\lambda^2/2s) \; \forall \lambda \in \bb{R}.
\end{align*}
Therefore, for any $z\in [f]$, if we solve the convex programs in Lines 13-16 of Algorithm \ref{algo:estimate_2}, then we directly apply the guarantees in Corollary \ref{coro:obs3} to conclude that the computed matrix in Line 16  $\widehat{\fl{P}}^{(z)}$ satisfies the following: for any $(i,j)\in [\s{M}]\times[\s{N}]$, we have with probability exceeding $ 1-\delta-O(d_2^{-c})$ (see Remark \ref{rmk:failure_prob} - we can set $c$ such that the failure probability $\delta+O(d_2^{-c})<1/10$)
 \begin{align}\label{eq:normalized}
  |\widehat{\fl{P}}^{(z)}_{ij}-\fl{P}_{ij}|\le O\Big(\frac{\sigma r }{\sqrt{sd_2}}\sqrt{\frac{\mu^3\log d_2}{p}}\Big)
 \end{align}
 where with $m_{z}=\max_{i \in [\s{M}]}\mid |j \in [\s{N}]\mid (i,j) \in \Omega|$, the total number of rounds required to compute $\widehat{\fl{P}}^{(z)}$ is given by $sm_z = O\Big(s(\s{N}p+\sqrt{\s{N}p\log(\s{M}\delta^{-1})})\Big)$. Also, note that since the new and reduced noise variance proxy is $\sigma^2/s$, hence the condition  $\sigma=O\Big(\sqrt{\frac{pd_2}{\mu^3\log d_2}}\|\mathbf{P}\|_{\infty}\Big)$ stated in Lemma \ref{thm:randomsample3} (referred to in Corollary 1 that is being invoked) translates to $\sigma/\sqrt{s}=O\Big(\sqrt{\frac{pd_2}{\mu^3\log d_2}}\|\mathbf{P}\|_{\infty}\Big)$. This is the new condition on $\sigma$ that is stated in the current lemma.
\end{proof}

\begin{claim}
Fix any index $(i,j)\in [\s{M}\times [\s{N}]]$.
Suppose the conditions in the Lemma statement hold true. Consider the final estimate $\widehat{\fl{P}}$ computed by taking the entry-wise median of $\widehat{\fl{P}}^{(1)},\widehat{\fl{P}}^{(2)},\dots,\widehat{\fl{P}}^{(f)}$. With $f=O(\log (\s{MN}\delta^{-1}))$, the $(ij)^{\s{th}}$ entry of the final estimate $\widehat{\fl{P}}_{ij}$ satisfies
\begin{align*}
  |\widehat{\fl{P}}_{ij}-\fl{P}_{ij}| = O\Big(\frac{\sigma r }{\sqrt{sd_2}}\sqrt{\frac{\mu^3\log d_2}{p}}\Big)
 \end{align*}
 with probability $1-\delta/\s{MN}$.
\end{claim}

\begin{proof}
\noindent \textbf{Independent estimates and entry-wise median for small probability of error:} Consider Lines 2,3 and 18 in Algorithm \ref{algo:estimate_2}. Algorithm \ref{algo:estimate_2} basically computes $f$ independent estimates $\widehat{\fl{P}}^{(1)},\widehat{\fl{P}}^{(2)},\dots,\widehat{\fl{P}}^{(f)}$ of the reward matrix $\fl{P}$ in order to boost the probability of success. 
Furthermore, we can compute a final estimate $\widehat{\fl{P}}$ (Line 18) by computing the entry-wise median of the matrix estimates $\widehat{\fl{P}}^{(1)},\widehat{\fl{P}}^{(2)},\dots,\widehat{\fl{P}}^{(f)}$ i.e. for all $(i,j)\in [\s{M}]\times [\s{N}]$, we compute $\widehat{\fl{P}}_{ij}=\s{median}(\widehat{\fl{P}}^{(1)}_{ij},\dots,\widehat{\fl{P}}^{(f)}_{ij})$.

Our goal is to increase the probability of success to $1-\delta$  and we will show that the entry-wise median of the $f$ estimates $\widehat{\fl{P}}$ does have the increased probability of success. As mentioned in the claim statement, fix the pair of indices $(i,j)$.
Consider the random variable \begin{align*}
 Y^{(z)} = \mathbf{1}\Big[|\widehat{\mathbf{P}}^{(z)}_{ij}-\mathbf{P}_{ij}|= O\Big(\frac{\sigma r }{\sqrt{sd_2}}\sqrt{\frac{\mu^3\log d_2}{p}}\Big)\Big]   
\end{align*}
which is $1$ with probability at least $9/10$ (from Claim 1).
Since the estimates $\widehat{\mathbf{P}}^{(1)}, \widehat{\mathbf{P}}^{(2)}, \dots,\widehat{\mathbf{P}}^{(f)}$ are independently computed , the random variables $Y^{(1)}, Y^{(2)}, \dots, Y^{(f)}$ are independent as well. Hence consider the median $\widehat{\mathbf{P}}_{ij} = \text{median}(\widehat{\mathbf{P}}^{(1)}_{ij},\widehat{\mathbf{P}}^{(2)}_{ij},\dots,\widehat{\mathbf{P}}^{(f)}_{ij})$. Note that the median $\widehat{\mathbf{P}}_{ij}$ will satisfy
$|\widehat{\mathbf{P}}_{ij}-\mathbf{P}_{ij}|= O\Big(\frac{\sigma r }{\sqrt{sd_2}}\sqrt{\frac{\mu^3\log d_2}{p}}\Big)$
if at least half of the random variables $Y^{(1)}, Y^{(2)}, \dots, Y^{(f)}$ are non-zero.  Hence, we can apply Chernoff bound directly to state that 
$\Pr(\sum_{i=1}^{f} Y^{(i)} < f/2 ) \le 2\exp(-4f/75)$
Therefore, by setting $f =O(\log (\mathsf{MN}\delta^{-1}))$, we must have that  $\Pr(\sum_{i=1}^{f} Y^{(i)} < f/2 ) \le \delta/\mathsf{MN}$. Hence we must have that
 $\widehat{\fl{P}}_{ij}$ satisfies
\begin{align}\label{eq:normalized_2}
  |\widehat{\fl{P}}_{ij}-\fl{P}_{ij}|= O\Big(\frac{\sigma r }{\sqrt{sd_2}}\sqrt{\frac{\mu^3\log d_2}{p}}\Big)
 \end{align}
 with probability $1-\delta/\s{MN}$. 
\end{proof}

Now, to complete the proof of the lemma, by taking a union bound over all indices, we must have that 
\begin{align}
    \|\widehat{\fl{P}}-\fl{P}\|_{\infty} \le O\Big(\frac{\sigma r}{\sqrt{sd_2}}\sqrt{\frac{\mu^3\log d_2}{p}}\Big)
\end{align}
with probability at least $1-\delta$. 
Also, note that the total number of rounds needed to compute these estimates is at most $\sum_{z}m_z sf$. From Corollary \ref{coro:obs3}, we know that for any $z\in [f]$, $m_z$ is bounded from above by $O(\s{N}p+\sqrt{\s{N}p\log(\s{M}\delta^{-1})})$ with probability $1-\delta$. Hence, by taking a union bound over all $z\in [f]$, the total number of rounds needed to compute all the $f$ estimates must be bounded from above by 
\begin{align*}
    \sum_{z}m_z sf = O\Big(sf(\s{N}p+\sqrt{\s{N}p\log(\s{M}\delta^{-1})})\Big)
\end{align*}
with probability at least $1-\delta f$. Therefore, by a union bound over both failure events, the total failure probability is $1-O(\delta+\delta\log (\s{MN}\delta^{-1}))= 1- O(\delta\log (\s{MN}\delta^{-1}))$. This completes the proof of the lemma.


\end{proof}

\section{Missing Proofs in Section \ref{sec:warmup}}\label{app:warm_up}

\begin{thmu}[Restatement of Theorem \ref{thm:baseline}]
Consider the rank-$r$ online matrix completion problem with $\s{M}$ users, $\s{N}$ items, $\s{T}$ recommendation rounds. Set $d_2=\min(\s{M},\s{N})$. Let $\fl{R}^{(t)}_{u\rho_u(t)}$ be the reward in each round, defined as in \eqref{eq:obs}. Suppose $d_2=\Omega(\mu r\log(rd_2))$. Let $\fl{P}\in \bb{R}^{\s{M}\times \s{N}}$ be the expected reward matrix that satisfies the conditions stated in Lemma \ref{lem:min_acc} , and let $\sigma^2$ be the noise variance in rewards. Then, Algorithm \ref{algo:p1}, applied to the online rank-$r$ matrix completion problem guarantees the following regret: 
\begin{align}\label{eq:ub_etan3}
   \s{Reg}(\s{T}) = O\Big(\Big(\s{T}^{\frac23}(\sigma^2r^2 \|\fl{P}\|_{\infty})^{\frac13}\Big(\frac{\mu^3 \s{N} \log d_2}{d_2}\Big)^{1/3}+\frac{\s{N}\mu^2}{d_2}\log^3 d_2\Big) \log^2(\s{MNT}) +\frac{\|\fl{P}\|_{\infty}}{\s{T}^{2}}\Big).
\end{align}
\end{thmu}

\begin{proof}[Proof of Theorem \ref{thm:baseline}]
 
Suppose we explore for a period of $\s{S}$ rounds such that the exploration period (Steps 3-7 in Algorithm \ref{algo:p1}) succeeds with a probability of $1-\nu$ i.e. conditioned on the event that
the exploration period succeeds, we obtain an estimate $\widehat{\fl{P}}$ of the reward matrix $\fl{P}$ satisfying $\|\fl{P}-\widehat{\fl{P}}\|_{\infty}\le \eta$. In that case, conditioned on the event that
the exploration period succeeds, at each step of the online algorithm in the exploitation phase (Step 9 in Algorithm \ref{algo:p1}), we suffer a regret of at most $2\eta$. To see this, fix any user $u\in [\s{M}]$. Let $i=\s{argmax}_{t\in [\s{N}]} \fl{P}_{ut}$ and $i'=\s{argmax}_t \widehat{\fl{P}}_{ut}$ be the items with the largest rewards in the matrices $\fl{P},\widehat{\fl{P}}$ respectively for the user $u$. In that case, using the fact that $\|\fl{P}-\widehat{\fl{P}}\|_{\infty}\le \eta$, we have
\begin{align*}
    \fl{P}_{ui'}-\fl{P}_{ui} = \fl{P}_{ui'}-\widehat{\fl{P}}_{ui'}+\widehat{\fl{P}}_{ui'}-\widehat{\fl{P}}_{ui}+\widehat{\fl{P}}_{ui}-\fl{P}_{ui} \ge -2\eta.
\end{align*}
Conditioned on the event that the exploration fails (and therefore the exploration stage as well), the regret at each step can be bounded from above by $2\|\fl{P}\|_{\infty}$. In that case, we have 
\begin{align*}
    \s{Reg}(\s{T}) &=  \frac{T}{\s{M}}\sum_{u \in [\s{M}]}\f{\mu}_u^{\star}- \sum_{t\in[\s{T}]}\frac{1}{\s{M}}\sum_{u\in[\s{M}]}\mathbf{P}_{u\pi_u(t)}\le \max_{u \in [\s{M}]} \Big(\s{T}\f{\mu}_u^{\star}-\sum_{t \in [\s{T}]}\fl{P}_{u\pi_u(t)}\Big) \\
    &\le 2\s{S}\|\fl{P}\|_{\infty}+2\s{(T-S)}\eta\Pr(\text{Exploration succeeds})+2\s{(T-S)}\|\fl{P}\|_{\infty}\Pr(\text{Exploration fails}) \\
    &\le 2\s{S}\|\fl{P}\|_{\infty}+2\s{T}\eta+2\s{T}\nu\|\fl{P}\|_{\infty}.
\end{align*}

Now, with $d_2=\min(\s{M},\s{N})$ satisfying $d_2=\Omega(\mu r \log (rd_2))$, we can directly use Lemma \ref{lem:min_acc} with $\s{S}=O(s\log (\s{MN}\delta^{-1})(\s{N}p+\sqrt{\s{N}p\log \s{M}\delta^{-1}})$ such that $\eta=O(\frac{\sigma r }{\sqrt{sd_2}}\sqrt{\frac{\mu^3\log d_2}{p}})$ and $\nu=\delta\log (\s{MN}\delta^{-1})$. Hence, we have 
\begin{align*}
    \s{Reg}(\s{T})  
    &\le O\Big(\underbrace{s\log (\s{MN}\delta^{-1})(\s{N}p+\sqrt{\s{N}p\log \s{M}\delta^{-1}}\Big)\|\fl{P}\|_{\infty}}_{\text{Term 1}}+\underbrace{\s{T}O\Big(\frac{\sigma r }{\sqrt{sd_2}}\sqrt{\frac{\mu^3\log d_2}{p}}\Big)}_{\text{Term 2}}\\
    &+\s{T}(\delta\log (\s{MN}\delta^{-1}))\|\fl{P}\|_{\infty}.
\end{align*}
For the sake of simplicity, we ignore some of the logarithmic terms and lower order terms in the regret while choosing the value of $sp$ that makes the regret small. To be precise, $sp$ is chosen by minimizing the quantity  $\mathsf{N}sp\|\mathbf{P}\|_{\infty}+\mathsf{T}\Big(\frac{\sigma r\sqrt{\mu^3 \log d_2}}{\sqrt{spd_2}}\Big)$ . It can be seen that the aforementioned quantity is minimized when the two terms are equal giving us  $ \mathsf{N}sp\|\mathbf{P}\|_{\infty}=\mathsf{T}\Big(\frac{\sigma r\sqrt{\mu^3 \log d_2}}{\sqrt{spd_2}}\Big)$ implying that $sp=( \mathsf{N}\|\mathbf{P}\|_{\infty})^{-2/3}\Big(\mathsf{T}\frac{\sigma r\sqrt{\mu^3 \log d_2}}{\sqrt{d_2}}\Big)^{2/3}$. 
Moreover, we will also choose $\delta=(\s{MNT})^{-4}$. Also, notice that $\s{N}p \ge 1$ since $p\ge C\mu^2d_2^{-1}\log ^3 d_2$ (for some appropriate constant $C\ge 1$) and $\mu \ge 1$; therefore
$\s{N}p+\sqrt{\s{N}p\log \s{M}\delta^{-1}}=O(\s{N}p\sqrt{\log\s{M}\delta^{-1}})$.
Subsequently, we have that  
\begin{align}\label{eq:etc_b1}
    \s{Reg}(\s{T}) = O\Big(\underbrace{\s{T}^{2/3}(\sigma^2r^2 \|\fl{P}\|_{\infty})^{1/3}\Big(\frac{\mu^3 \s{N} \log d_2}{d_2}\Big)^{1/3}\log^2(\s{MNT})}_{\text{Bound 1}}+\|\fl{P}\|_{\infty}\s{T}^{-2}\Big).
\end{align}

There exists an edge case when $sp \le C\mu^2d_2^{-1}\log^3 d_2$. Then we can substitute $s=1$ and $p=C\mu^2d_2^{-1}\log^3 d_2$. In that case, Term 2 will still be bounded by Bound 1 in equation \ref{eq:etc_b1} since $sp\le C\mu^2d_2^{-1}\log^3 d_2$. On the other hand Term 1 will now be bounded by $O(\frac{\s{N}\mu^2}{d_2}\log^3 d_2 \log^2 (\s{MNT})\|\fl{P}\|_{\infty})$. Hence, our regret will be bounded by 

\begin{align*}
    \s{Reg}(\s{T}) = O\Big(\s{T}^{2/3}(\sigma^2r^2 \|\fl{P}\|_{\infty})^{1/3}\Big(\frac{\mu^3 \s{N} \log d_2}{d_2}\Big)^{1/3}\log^2(\s{MNT})+\frac{\s{N}\mu^2}{d_2} \log^5(\s{MNT})\|\fl{P}\|_{\infty} +\|\fl{P}\|_{\infty}\s{T}^{-2}\Big).
\end{align*}

\end{proof}

\section{Missing Proofs in Section \ref{sec:main1}}\label{app:repeated}

In this section, we will consider the reward matrix $\fl{P}=\fl{u}\fl{v}^{\s{T}}=\lambda \fl{\bar{u}}\fl{\bar{v}}^{\s{T}}\in \bb{R}^{\s{M}\times \s{N}}$ to be a rank-$1$ matrix with $\|\fl{\bar{u}}\|_2=\|\fl{\bar{v}}\|_2=1$.
We make the following mild assumptions in line with Theorem \ref{thm:main1}: (recall that $j_{\max}=\s{arg}\max_i \fl{v}_i$ and $j_{\min}=\s{arg}\min_i \fl{v}_i$
\begin{enumerate}
    \item If $\s{M}=\s{T}^{\zeta}$ for $\zeta > \frac{1}{2}$, then the vector $\fl{\bar{u}}$ is  $(\s{M}^{-1/2\zeta},\mu)$-incoherent i.e. for any subset $\ca{U}\subseteq [\s{M}]$ of indices such that $|\ca{U}|\ge \s{M}^{1-(2\zeta)^{-1}}$, we must have $\|\fl{\bar{v}}_{\ca{U}}\|_{\infty}\le \sqrt{\frac{\mu}{|\ca{U}|}}\|\fl{\bar{v}}_{\ca{U}}\|_{2}$.
    \item $\beta=\max\Big(\left|\frac{\fl{\bar{v}}_{j_{\max}}}{\fl{\bar{v}}_{j_{\min}}}\right|,\left|\frac{\fl{\bar{v}}_{j_{\min}}}{\fl{\bar{v}}_{j_{\max}}}\right|\Big)$ for some positive constant $\beta>0$. Note that if we represent $\fl{P}=\fl{u}\fl{v}^{\s{T}}$ where $\fl{u}$ is the user embedding and $\fl{v}$ is the item embedding, then we have $\beta=\max\Big(\left|\frac{\fl{v}_{j_{\max}}}{\fl{v}_{j_{\min}}}\right|,\left|\frac{\fl{v}_{j_{\min}}}{\fl{v}_{j_{\max}}}\right|\Big)$. In the definition of $\beta$, the second term can be larger than the first because of the absolute value.
\end{enumerate}




In each phase $\ell$, we will maintain three groups of users (user groups $\ca{B}^{(\ell)},\ca{M}^{(\ell,1)},\ca{M}^{(\ell,2)}$ are disjoint and form a partition of $[\s{M}]$) and three corresponding groups of items (not necessarily disjoint) namely 1) $(\ca{B}^{(\ell)},[\s{N}])$ (initialized by $([\s{M}],[\s{N}])$, 2) ($\ca{M}^{(\ell,1)},\ca{N}^{(\ell,1)}$) (initialized with $(\phi,\phi)$) and 3) 
($\ca{M}^{(\ell,2)}$,$\ca{N}^{(\ell,2)}$) (initialized with $(\phi,\phi)$) such that $\ca{B}^{(\ell)}\cup \ca{M}^{(\ell,1)} \cup \ca{M}^{(\ell,2)} =[\s{M}]$ and $\ca{N}^{(\ell,1)},\ca{N}^{(\ell,2)} \subseteq [\s{N}]$. 
Suppose we have a sequence of tuples $(m_{\ell},\Delta_{\ell})_{\ell}$ that is going to be characterized precisely later.
In every phase indexed by $\ell$, we will compute three matrices $\widetilde{\fl{Q}}^{(\ell)}, \widetilde{\fl{P}}^{(\ell,1)}, \widetilde{\fl{P}}^{(\ell,2)} \in \bb{R}^{\s{M}\times \s{N}}$ (using $m_{\ell}$ rounds for each) that correspond to  estimates of three relevant sub-matrices of $\fl{P}$ namely $\fl{P}_{\ca{B}^{(\ell)},[\s{N}]},\fl{P}_{\ca{M}^{(\ell,1)},\ca{N}^{(\ell,1)}}$ and $\fl{P}_{\ca{M}^{(\ell,2)},\ca{N}^{(\ell,2)}}$ respectively. We will define the event $\ca{E}_1^{(\ell)}$ when the following holds true 
\begin{align*}
    &\left|\left|\widetilde{\fl{Q}}^{(\ell)}_{\ca{B}^{(\ell)},[\s{N}]}-\fl{P}_{\ca{B}^{(\ell)},[\s{N}]}\right|\right|_{\infty} \le \Delta_{\ell}/2 \\
    &\left|\left|\widetilde{\fl{P}}^{(\ell,1)}_{\ca{M}^{(\ell,1)},\ca{N}^{(\ell,1)}}-\fl{P}_{\ca{M}^{(\ell,1)},\ca{N}^{(\ell,1)}}\right|\right|_{\infty} \le \Delta_{\ell}/2 \\
    &\left|\left|\widetilde{\fl{P}}^{(\ell,2)}_{\ca{M}^{(\ell,2)},\ca{N}^{(\ell,2)}}-\fl{P}_{\ca{M}^{(\ell,2)},\ca{N}^{(\ell,2)}}\right|\right|_{\infty} \le \Delta_{\ell}/2
\end{align*}
by using a number of rounds $m_{\ell}$ that is bounded from above by $O\Big(\max(1,\frac{\s{N}}{\s{M}^{1-(2\zeta)^{-1}}})\frac{\sigma^2\mu^3\log^2(\s{MNT})}{\Delta_{\ell}}\Big)$. The randomness stems from the inherent randomness in the algorithm.

At each phase indexed by $\ell$, we will compute a temporary set of active items $\ca{T}_u^{(\ell+1)}$ for all users $u\in [\s{N}]$. We update $\ca{B}^{(\ell+1)} \equiv \Big\{u \in \ca{B}^{(\ell)}\mid \left|\max_{t}\fl{\widetilde{P}}^{(\ell)}_{ut}-\min_{t}\fl{\widetilde{P}}^{(\ell)}_{ut}\right|\le 2a\Delta_{\ell}\Big\}$ where an appropriate value of $a$ will be determined later. This corresponds to the sets of users whose rewards are almost similar across items and hence, it does not no matter which items are picked in the next phase.
We set $\ca{T}_u^{(\ell+1)} =[\s{N}]$ for all users $u\in\ca{B}^{(\ell+1)}$ and for the rest of the users (namely $u\in [\s{N}]\setminus\ca{B}^{(\ell+1)}$), we set $\ca{T}_u^{(\ell+1)}$ according to Steps 17,18 in Algorithm \ref{algo:phased_elim} for combined users in the set $\ca{B}^{(\ell)}\setminus \ca{B}^{(\ell+1)}$, $\ca{M}^{(\ell,1)}$ and $\ca{M}^{(\ell,2)}$ respectively. Subsequently, we update the following groups of users and corresponding items:
\begin{enumerate}
    \item We will set $v$ to be any user in $[\s{M}]\setminus \ca{B}^{(\ell+1)}$. Subsequently, we will define the set $\ca{M}^{(\ell+1,1)}=\{u \in [\s{M}]\setminus\ca{B}^{(\ell+1)} \mid \ca{T}_{u}^{(\ell+1)} \cap \ca{T}_{v}^{(\ell+1)}\neq \phi\}$. We update $\ca{N}^{(\ell+1,1)}=\bigcap_{u \in \ca{M}^{(\ell+1,1)}}\ca{T}_u^{(\ell+1)}$.  
    \item Finally, we will identify the set $\ca{M}^{(\ell+1,2)} \equiv  [\s{M}]\setminus (\ca{B}^{(\ell+1)}\cup \ca{M}^{(\ell+1,1)})$ to be the remaining users. As before, we update $\ca{N}^{(\ell+1,2)}=\bigcap_{u \in \ca{M}^{(\ell+1,2)}}\ca{T}_u^{(\ell+1)}$.
\end{enumerate}
For any phase indexed by $\ell$, we define the event $\ca{E}_2^{(\ell)}$ such that for every user  $u\in  \ca{M}^{(\ell,1)}$, the following holds:
\begin{align*}
   & \left|\fl{P}_{ut}-\max_{t'\in [\s{N}]}\fl{P}_{ut}\right| \le 2\Delta_{\ell-1} \text{  for all } t\in\ca{N}^{(\ell,1)} , u\in \ca{M}^{(\ell,1)}\\
    &\left|\fl{P}_{ut}-\max_{t'\in [\s{N}]}\fl{P}_{ut'}\right| \le 2\Delta_{\ell-1} \text{  for all } t\in\ca{N}^{(\ell,2)}, u\in \ca{M}^{(\ell,2)} \\
    &\left|\fl{P}_{ut}-\max_{t'\in [\s{N}]}\fl{P}_{ut'}\right| \le (2a+2)\Delta_{\ell-1} \text{  for all } t\in [\s{N}], u\in \ca{B}^{(\ell)}.
\end{align*}


Define the event $\ca{E}_3^{(\ell)}$ which is true when
$\ca{M}^{(\ell,1)}\subseteq \ca{C}_1$, $j_{\max}\in \ca{N}^{(\ell,1)}$ and  $\ca{M}^{(\ell,2)}\subseteq \ca{C}_2$, $j_{\min}\in \ca{N}^{(\ell,2)}$. Finally, let  $\ca{E}_4^{(\ell)}$ be the event which is true when
$\ca{M}^{(\ell,2)}\subseteq \ca{C}_1$, $j_{\min}\in \ca{N}^{(\ell,1)}$ and  $\ca{M}^{(\ell,1)}\subseteq \ca{C}_2$, $j_{\max}\in \ca{N}^{(\ell,2)}$. 

We now present a series of Lemma required for the proof of Theorem~\ref{thm:main1}. 
\begin{lemma}\label{lem1:first}
Consider two users $u,v$ such that $u \in \ca{C}_1$ and $v\in \ca{C}_2$. Suppose the event $\ca{E}_1^{(\ell)}$ is true and furthermore, either $\ca{E}_3^{(\ell)}$ or $\ca{E}_4^{(\ell)}$ is true. In that case, $\ca{T}_u^{(\ell+1)}\cap \ca{T}_v^{(\ell+1)} \neq \phi$ only if either $\max_{t \in [\s{N}]}\fl{P}_{ut}-\min_{t \in [\s{N}]}\fl{P}_{ut} \le 4\Delta_{\ell}$ or $\max_{t \in [\s{N}]}\fl{P}_{vt}-\min_{t \in [\s{N}]}\fl{P}_{vt} \le 4\Delta_{\ell}$ holds true.
\end{lemma}

\begin{proof}[Proof of Lemma \ref{lem1:first}]
Since either 
 $\ca{E}_3^{(\ell)}$ or $\ca{E}_4^{(\ell)}$  holds true, it must happen that $j_{\max}\in \ca{N}^{(\ell,1)}$ for all $u\in \ca{C}_1$ and $j_{\min}\in \ca{N}^{(\ell,2)}$ for all $u\in \ca{C}_2$. 
 Further, suppose there exists an item $j\in [\s{N}]$ such that $j\in \ca{T}_u^{(\ell+1)}\cap \ca{T}_v^{(\ell+1)}$. Hence this implies that 
 \begin{align*}
     &-\fl{\widetilde{P}}_{uj}^{(\ell,1)}+\max_{t\in \ca{N}^{(\ell,1)}}\fl{\widetilde{P}}_{ut}^{(\ell,1)} \le \Delta_{\ell} \\
     &\implies -\fl{\widetilde{P}}_{uj}^{(\ell,1)}+\fl{P}_{uj}-\fl{P}_{uj}+\fl{P}_{uj_{\max}}-\fl{P}_{uj_{\max}} \\
     &+\fl{\widetilde{P}}_{uj_{\max}}^{(\ell,1)}-\fl{\widetilde{P}}_{uj_{\max}}^{(\ell,1)}+\max_{t\in \ca{N}^{(\ell,1)}}\fl{\widetilde{P}}_{ut}^{(\ell,1)} \le \Delta_{\ell} \\
     &\implies \fl{P}_{uj_{\max}}-\fl{P}_{uj} \le 2\Delta_{\ell}
 \end{align*}
where we used the following facts: (a) $-\fl{P}_{uj_{\max}}+\fl{\widetilde{P}}_{uj_{\max}}^{(\ell,1)} \ge -\Delta_{\ell}/2$ (b) $-\fl{\widetilde{P}}_{uj_{\max}}^{(\ell,1)}+\max_{t\in \ca{N}^{(\ell,1)}}\fl{\widetilde{P}}_{ut}^{(\ell,1)} \ge 0$ (c) $-\fl{\widetilde{P}}_{uj}^{(\ell,1)}+\fl{P}_{uj} \ge -\Delta_{\ell}/2$. 
Similarly, we must have $\fl{P}_{vj_{\min}}-\fl{P}_{vj} \le 2\Delta_{\ell}$. Without loss of generality, assume that $\fl{u}_{u}\le -\fl{u}_{v}$. In that case, we will have 
\begin{align*}
    &\fl{P}_{uj_{\max}}-\fl{P}_{uj_{\min}} \le \fl{P}_{uj_{\max}}-\fl{P}_{uj}+\fl{P}_{uj}-\fl{P}_{uj_{\min}} \\
    & \le \fl{P}_{uj_{\max}}-\fl{P}_{uj}-\fl{P}_{vj}+\fl{P}_{vj_{\min}}   \le  4\Delta_{\ell}
\end{align*}
where we used the fact that $\fl{P}$ is a rank-1 matrix and hence $\fl{P}_{uj}-\fl{P}_{uj_{\min}} = \fl{u}_u(\fl{v}_j-\fl{v}_{j_{\min}}) \le -\fl{u}_v(\fl{v}_j-\fl{v}_{j_{\min}})=-\fl{P}_{vj}+\fl{P}_{vj_{\min}}$.
This leads to a contradiction and therefore completes the proof of the lemma. 
\end{proof}

\begin{lemma}\label{lem1:second}
Recall that in Step 16 of Algorithm \ref{algo:phased_elim}, we have $\ca{B}^{(\ell+1)} \equiv \Big\{u \in \ca{B}^{(\ell)}\mid \left|\max_{t\in [\s{N}]}\fl{\widetilde{Q}}^{(\ell)}_{ut}-\min_{t\in [\s{N}]}\fl{\widetilde{Q}}^{(\ell)}_{ut}\right|\le (2a+1)\Delta_{\ell}\Big\}$ for $a>3$.
Conditioned on the event $\ca{E}_1^{(\ell)}$, 
the set of users $\ca{B}^{(\ell+1)}$ must contain all users $u \in [\s{N}]$ satisfying $\max_{t \in [\s{N}]}\fl{P}_{ut}-\min_{t \in [\s{N}]}\fl{P}_{ut} \le 2a\Delta_{\ell}$. Furthermore, all users  $u\in\ca{B}^{(\ell+1)}$ must satisfy $\max_{t \in [\s{N}]}\fl{P}_{ut}-\min_{t \in [\s{N}]}\fl{P}_{ut} \le (2a+2)\Delta_{\ell}$.
\end{lemma}

\begin{proof}[Proof of Lemma \ref{lem1:second}]
Recall that $\ca{B}^{(\ell+1)} \equiv \Big\{u \in \ca{B}^{(\ell)}\mid \max_{t\in[\s{N}]}\fl{\widetilde{Q}}^{(\ell)}_{ut}-\min_{t\in[\s{N}]}\fl{\widetilde{Q}}^{(\ell)}_{ut}\le (2a+1)\Delta_{\ell}\Big\}$ for some appropriate constant $a>0$ that will be chosen later.
Suppose for a user $u\in \ca{B}^{(\ell)}$, we have $\max_{t \in [\s{N}]}\fl{P}_{ut}-\min_{t \in [\s{N}]}\fl{P}_{ut} \le 2a\Delta_{\ell}$.
 Let $t_1= \s{arg}\max_{t \in [\s{N}]}\fl{\widetilde{Q}_{ut}}^{(\ell)}$, $t_2= \s{arg}\min_{t \in [\s{N}]}\fl{\widetilde{Q}_{ut}}^{(\ell)}$ and $t_3= \s{arg}\min_{t \in [\s{N}]}\fl{P}_{ut}$. 
Since $\ca{E}_1^{(\ell)}$
is true, we must have
\begin{align*}
    &\left|\fl{\widetilde{Q}}^{(\ell)}_{ut_1}-\fl{P}_{ut_1}\right| \le \Delta_{\ell}/2, \quad \left|\fl{\widetilde{Q}}^{(\ell)}_{ut_1}-\fl{P}_{ut_1}\right| \le \Delta_{\ell}/2,  \\
    &\quad \text{and} \quad \left|\fl{\widetilde{Q}}^{(\ell)}_{ut_2}-\fl{P}_{ut_2}\right| \le \Delta_{\ell}/2.
\end{align*}
Therefore, we must have 
\begin{align*}
    &\fl{\widetilde{Q}}^{(\ell)}_{ut_1}-\fl{\widetilde{Q}}^{(\ell)}_{ut_2} \\
    &\le \fl{\widetilde{Q}}^{(\ell)}_{ut_1}-\fl{P}_{ut_1}+\fl{P}_{ut_1}-\max_{t\in [\s{N}]}\fl{P}_{ut}+\max_{t\in [\s{N}]}\fl{P}_{ut} \\
    &-\min_{t\in [\s{N}]}\fl{P}_{ut}+\min_{t\in [\s{N}]}\fl{P}_{ut}-\fl{P}_{ut_2}+\fl{P}_{ut_2}-\fl{\widetilde{Q}}^{(\ell)}_{ut_2}
    \le (2a+1)\Delta_{\ell}
\end{align*}
where we used the fact that $\fl{P}_{ut_1}-\max_{t\in [\s{N}]}\fl{P}_{ut} \le 0$ and $\min_{t\in [\s{N}]}\fl{P}_{ut}-\fl{P}_{ut_2} \le 0$.
Now, consider a user $u$ such that $\fl{\widetilde{Q}}^{(\ell)}_{ut_1}-\fl{\widetilde{Q}}^{(\ell)}_{ut_2} \le (2a+1)\Delta_{\ell}$. In that case, we will have 
\begin{align*}
    &\max_{t\in [\s{N}]}\fl{P}_{ut}-\min_{t\in [\s{N}]}\fl{P}_{ut} \\
    &\le \max_{t\in [\s{N}]}\fl{P}_{ut}-\fl{P}_{ut_1}+\fl{P}_{ut_1}-\fl{\widetilde{Q}}^{(\ell)}_{ut_1}+\fl{\widetilde{Q}}^{(\ell)}_{ut_1}    -\fl{\widetilde{Q}}^{(\ell)}_{ut_2} \\
    &+\fl{\widetilde{Q}}^{(\ell)}_{ut_2}-\fl{\widetilde{Q}}^{(\ell)}_{ut_3}+\fl{\widetilde{Q}}^{(\ell)}_{ut_3}-\min_{t\in [\s{N}]}\fl{P}_{ut} 
    \le (2a+2)\Delta_{\ell}
\end{align*}
where we used the fact that $\max_{t\in [\s{N}]}\fl{P}_{ut}-\fl{P}_{ut_1}\le 0$ and $\fl{\widetilde{Q}}^{(\ell)}_{ut_2}-\fl{\widetilde{Q}}^{(\ell)}_{ut_3} \le 0$.
\end{proof}

\begin{lemma}\label{lem1:third}
Suppose the event  $\ca{E}^{(\ell)}_1$ is true and furthermore, either $\ca{E}^{(\ell)}_3$ or $\ca{E}^{(\ell)}_4$ is true. In that case, for every user $u\in [\s{M}]\setminus \ca{B}^{(\ell+1)}$, we will have that $\s{arg}\max_{t\in [\s{N}]}\fl{P}_{ut} \in \ca{T}_u^{(\ell+1)}$.
\end{lemma}

\begin{proof}[Proof of Lemma \ref{lem1:third}]
Without loss of generality, let us assume that $\ca{E}_3^{(\ell)}$ is true. Hence, this implies that  $\ca{M}^{(\ell,1)}\subseteq \ca{C}_1$ and for 
 every user $u\in \ca{M}^{(\ell,1)}$, $j_{\max}\in \ca{N}^{(\ell,1)}$ and furthermore, $\ca{M}^{(\ell,2)}\subseteq \ca{C}_2$ and for every user $u\in \ca{M}^{(\ell,2)}$, $j_{\min}\in \ca{N}^{(\ell,2)}$. 
 Recall that for every user $u \in \ca{M}^{(\ell,1)}$, we compute $\ca{T}_{u}^{(\ell+1)}=\{j \in \ca{N}^{(\ell,1)}\}\mid \fl{\widetilde{P}}^{(\ell,1)}_{uj}+\Delta_{\ell}> \max_{t \in \ca{N}^{(\ell,1)}}\fl{\widetilde{P}}^{(\ell,1)}_{ut}\} \text{ for all }u \in \ca{M}^{(\ell,1)} 
 $. Clearly $\s{arg}\max_{t\in \ca{T}_{u}^{(\ell+1)}} \fl{\widetilde{P}}_{ut}^{(\ell,1)}=\s{arg}\max_{t\in \ca{N}_{u}^{(\ell,1)}} \fl{\widetilde{P}}_{ut}^{(\ell,1)}$. Let $t_1=\s{arg}\max_{t \in [\s{N}]}\fl{P}_{ut}^{(\ell,1)}$ and $t_2=\s{arg}\max_{t \in \ca{N}^{(\ell,1)}}\fl{\widetilde{P}}_{ut}^{(\ell,1)}$. By using triangle inequality, we will have \begin{align*}
     &\widetilde{\fl{P}}_{ut_1}^{(\ell,1)}-\widetilde{\fl{P}}_{ut_2}^{(\ell,1)} \le \widetilde{\fl{P}}_{ut_1}^{(\ell,1)}-\fl{P}_{ut_1}+\fl{P}_{ut_1} \\
     &-\fl{P}_{ut_2}+\fl{P}_{ut_2}-\widetilde{\fl{P}}_{ut_2}^{(\ell,1)} \le \Delta_{\ell}
 \end{align*}
 where we used the fact that $|\widetilde{\fl{P}}^{(\ell,1)}_{ut_1}-\fl{P}_{ut_1}|\le \Delta_{\ell}/2$, $|\widetilde{\fl{P}}^{(\ell,1)}_{ut_2}-\fl{P}_{ut_2}|\le \Delta_{\ell}/2$ and $\fl{P}_{ut_1}-\fl{P}_{ut_2}\le 0$. Hence, $t_1\in \ca{T}_u^{(\ell+1)}$ for any user $u\in \ca{M}^{(\ell,1)}$. By a similar analysis this holds for any user $u\in \ca{M}^{(\ell,2)}\cup (\ca{B}^{(\ell)}\setminus \ca{B}^{(\ell+1)})$ as well.
\end{proof}

\begin{lemma}\label{lem:fourth}
Suppose the event $\ca{E}_1^{(\ell)}$ is true and furthermore, either $\ca{E}_3^{(\ell)}$ or $\ca{E}_4^{(\ell)}$ is true. In that case the event $\ca{E}_2^{(\ell+1)}$ is true and one of $\ca{E}_3^{(\ell+1)},\ca{E}_4^{(\ell+1)}$ is also true.
\end{lemma}

\begin{proof}[Proof of Lemma \ref{lem:fourth}]
From Lemma \ref{lem1:second} (and using the fact that $a>3$ in Step 16 of Algorithm \ref{algo:phased_elim}), we know that $\ca{B}^{(\ell+1)}$ must contain all users $u \in [\s{N}]$ satisfying $\max_{t \in [\s{N}]}\fl{P}_{ut}-\min_{t \in [\s{N}]}\fl{P}_{ut} \le 4\Delta_{\ell}$. Therefore, from Lemma \ref{lem1:first}, there cannot exist users $u,v \in [\s{N}]\setminus \ca{B}^{(\ell+1)}$ such that $u\in \ca{C}_1,v\in \ca{C}_2$ and $\ca{T}_u^{(\ell+1)}\cap\ca{T}_v^{(\ell+1)}\neq \phi$.
Without loss of generality, suppose the user $v$ described in Line 10 in Algorithm \ref{algo:phased_elim}  is in $\ca{C}_1$. Since $\ca{M}^{(\ell+1,1)}$ contains all those users not in $\ca{B}^{(\ell+1)}$ such that $\ca{T}_u^{(\ell+1)}\cap\ca{T}_u^{(\ell+1)}\neq \phi$, $\ca{M}^{(\ell+1,1)}$ can only contain users in $\ca{C}_1$. Conversely, for any user $u\in \ca{C}_1\setminus \ca{B}^{(\ell+1)}$, $j_{\max}\in \ca{T}_u^{(\ell+1)}$ (using Lemma \ref{lem1:third}) and therefore $u\in \ca{M}^{(\ell+1,1)}$. In that case, it must happen that $\ca{M}^{(\ell+1,1)}=\ca{C}_1\setminus \ca{B}^{(\ell+1)}$ and similarly $\ca{M}^{(\ell+1,2)}=\ca{C}_2\setminus \ca{B}^{(\ell+1)}$. Therefore, again by using Lemma \ref{lem1:third}, $j_{\max}\in \ca{N}^{(\ell+1,1)}$ for all users $u\in \ca{M}^{(\ell+1,1)}$ and $j_{\min}\in \ca{N}^{(\ell+1,2)}$ for all users $u\in \ca{M}^{(\ell+1,2)}$ implying that $\ca{E}_3^{(\ell+1)}$ holds true.
On the other hand, if $v \in \ca{C}_2$, then by a similar analysis, $\ca{E}_4^{(\ell+1)}$ holds true. 

Next, we will show that $\ca{E}^{(\ell+1)}_{2}$ is also going to hold true. From Lemma \ref{lem1:second}, we know that all users $u\in \ca{B}^{(\ell+1)}$ must satisfy $\max_{t \in [\s{N}]}\fl{P}_{ut}-\min_{t \in [\s{N}]}\fl{P}_{ut} \le (2a+2)\Delta_{\ell}$.
Now, without loss of generality, assume that $\ca{E}_3^{(\ell)}$ is true. Consider any item $j \in \ca{N}^{(\ell+1,1)}$ implying that $j \in \ca{T}_{u}^{(\ell+1)}$ for all users $u\in \ca{M}^{(\ell+1,1)}$. Let $t_1= \s{arg}\max_{t \in \ca{T}_{u}^{(\ell+1)}} \widetilde{\fl{P}}_{ut}$. We must have 
\begin{align*}
    \fl{P}_{ut_1}-\fl{P}_{uj}=\fl{P}_{ut_1}-\fl{\widetilde{P}}_{ut_1}+\fl{\widetilde{P}}_{ut_1}-\fl{\widetilde{P}}_{uj}+\fl{\widetilde{P}}_{uj}-\fl{P}_{uj}. 
\end{align*}
On the other hand, since $j_{\max}\in \ca{T}_{u}^{(\ell+1)}$, we also have 
\begin{align*}
    &\fl{P}_{uj_{\max}}-\fl{P}_{ut_1} = \fl{P}_{uj_{\max}}-\fl{\widetilde{P}}^{(\ell,1)}_{uj_{\max}}+\fl{\widetilde{P}}^{(\ell,1)}_{uj_{\max}} \\
    &-\fl{\widetilde{P}}^{(\ell,1)}_{ut_1}+\fl{\widetilde{P}}^{(\ell,1)}_{ut_1}-\fl{P}_{ut_1}.
\end{align*}
Adding up the previous two equations, we get that
\begin{align*}
    &\fl{P}_{uj_{\max}}-\fl{P}_{uj} = \fl{P}_{uj_{\max}}-\fl{\widetilde{P}}^{(\ell,1)}_{uj_{\max}}+\fl{\widetilde{P}}^{(\ell,1)}_{uj_{\max}}-\fl{\widetilde{P}}^{(\ell,1)}_{ut_1} \\
    &+\fl{\widetilde{P}}^{(\ell,1)}_{ut_1}-\fl{\widetilde{P}}^{(\ell,1)}_{uj}+\fl{\widetilde{P}}^{(\ell,1)}_{uj}-\fl{P}_{uj} \le 2\Delta_{\ell}
\end{align*}
where we used the fact that $\fl{\widetilde{P}}^{(\ell,1)}_{ut_1}-\fl{\widetilde{P}}_{uj} \le \Delta_{\ell}$, $\fl{\widetilde{P}}^{(\ell,1)}_{uj_{\max}}-\fl{\widetilde{P}}^{(\ell,1)}_{ut_1} \le 0$, $\fl{P}_{uj_{\max}}-\fl{\widetilde{P}}^{(\ell,1)}_{uj_{\max}} \le \Delta_{\ell}/2 $ and $\fl{\widetilde{P}}^{(\ell,1)}_{uj}-\fl{P}_{uj} \le \Delta_{\ell}/2$. By showing a similar analysis for all users $u\in \ca{M}^{(\ell+1,2)} $ and all items $j\in \ca{N}^{(\ell+1,2)}$. we prove that $\ca{E}_2^{(\ell+1)}$ holds true.
\end{proof}

We will say that a user has been ignored in the phase indexed by $\ell$ under the following circumstances:
\begin{enumerate}
    \item The size of the set of users $\ca{M}^{(\ell,1)}$ (or $\ca{M}^{(\ell,2)}$) in Step 19 is smaller than $\s{M}^{1-(2\zeta)^{-1}}$. In that case, those users are put back in the set $\ca{B}^{(\ell)}$ (Step 21). In that case the users in the set $\ca{M}^{(\ell,1)}$ (respectively $\ca{M}^{(\ell,2)}$) will be ignored.
    \item The size of the set of users $\ca{B}^{(\ell)}$ in Step 6 is 
    smaller than $\s{M}^{1-(2\zeta)^{-1}}$. In this case, even if it is added with $\ca{M}^{(\ell,1)}$ or $\ca{M}^{(\ell,2)}$, it size becomes at most $2\s{M}^{1-(2\zeta)^{-1}}$ (the size of the added set in Step 21 must also be less than $\s{M}^{1-(2\zeta)^{-1}}$ due to the first point above). In that case, the users in $\ca{B}^{(\ell)}$ will be ignored.
\end{enumerate}

\begin{coro}\label{coro:crucial}
\begin{enumerate}
    \item Suppose the event $\ca{E}_2^{(\ell)}$ is true for all iterations indexed by $\ell$. In that case, in any phase indexed by $\ell$, for every user $u\in [\s{M}]$ that is not ignored in Algorithm \ref{algo:phased_elim} in the $\ell^{\s{th}}$ iteration, Algorithm \ref{algo:phased_elim} only recommends items that have reward at most $(2a+2)\Delta_{\ell-1}$ smaller than the reward of the best item $\s{arg max}_{j \in [\s{N}]}\fl{P}_{uj}$. 
    \item Furthermore, at each iteration indexed by $\ell$, at most $2\s{M}^{1-(2\zeta)^{-1}}$ users are ignored in Algorithm \ref{algo:phased_elim}.
\end{enumerate}
\end{coro}

\begin{proof}
 The proof of the first part follows directly from Lemma \ref{lem:fourth}
and the definition of $\ca{E}_2^{(\ell)}$. Note that for the first phase ($\ell=1$), we have $\ca{B}^{(1)}=[\s{M}]$ and therefore $\ca{E}_1^{(1)}$ being true implies that $\ca{E}_2^{(2)}$ is true and one of $\ca{E}_3^{(2)}$ or $\ca{E}_4^{(2)}$ is true as well.

 Let us now move on to the proof of the second part.  
Since the three sets $\ca{M}^{(\ell,1)},\ca{M}^{(\ell,2)},\ca{B}^{(\ell)}$ (before Step 12) partition the set of users $[\s{M}]$, at most two of them can be smaller than $\s{M}^{1-(2\zeta)^{-1}}$ and can be ignored. Hence, the total number of ignored users can be at most $2\s{M}^{1-(2\zeta)^{-1}}$.
\end{proof}

 \begin{lemma}[Incoherence]\label{lem:incoherence}
 Suppose $a>3\max(1+\beta,1+\frac{1}{\beta})$ in Step 6 of Algoithm \ref{algo:phased_elim}.
 For any iteration in Algorithm $\ref{algo:phased_elim}$ indexed by $\ell>0$, we will have 
 \begin{align*}
     \frac{\|\fl{v}_{\ca{N}^{(\ell+1,1)}}\|_{\infty}}{\|\fl{v}_{\ca{N}^{(\ell+1,1)}}\|_{2}} \le \frac{3\max(\beta,\beta^{-1})}{2\sqrt{|\ca{N}^{(\ell+1,1)}|}} \; \text{ and } \; \frac{\|\fl{v}_{\ca{N}^{(\ell+1,2)}}\|_{\infty}}{\|\fl{v}_{\ca{N}^{(\ell+1,2)}}\|_{2}} \le \frac{3\max(\beta,\beta^{-1})}{2\sqrt{|\ca{N}^{(\ell+1,2)}|}}
 \end{align*}
 \end{lemma}
 
 \begin{proof}
 Consider any user $u\in [\s{M}]$ such that 1) $u\in \ca{B}^{(\ell)}$ in the $\ell^{\s{th}}$ iteration but in the $(\ell+1)^{\s{th}}$ iteration, $u\in [\s{N}]\setminus \ca{B}^{(\ell+1)}$ and 2) $\fl{u}_u>0$. In this case, we must have $\fl{u}_u(\fl{v}_{j_{\max}}-\fl{v}_{j_{\min}}) \ge 2a\Delta_{\ell}$ from Lemma \ref{lem1:second} where $a>3\max(1+\beta,1+\frac{1}{\beta})$. Now consider any item $j\in [\s{N}]$ such that $\fl{P}_{uj_{\max}}-\fl{P}_{uj} > 2\Delta_{\ell}$.

 In that case, by triangle inequality we have $$ \fl{\widetilde{P}}_{uj_{\max}}^{(\ell,1)}-\fl{\widetilde{P}}_{u}^{(\ell,1)} > \fl{\widetilde{P}}_{uj_{\max}}^{(\ell,1)}-\fl{P}_{uj_{\max}}+\fl{P}_{uj_{\max}}-\fl{P}_{uj}+ \fl{\widetilde{P}}_{uj_{\max}}^{(\ell,1)} > 2\Delta_{\ell}-\frac{\Delta_{\ell}}{2}-\frac{\Delta_{\ell}}{2} > \Delta_{\ell}$$ and therefore 
 $j\not \in \ca{T}_u^{(\ell+1)}$. Hence consider any item $j\in [\s{N}]$ that belongs to $\ca{T}_u^{(\ell+1)}$; it must happen that $\fl{u}_u(\fl{v}_{j_{\max}}-\fl{v}_j)\le 2\Delta_{\ell} \le \frac{\fl{u}_u(\fl{v}_{j_{\max}}-\fl{v}_{j_{\min}})}{a}$ implying that $\fl{v}_j \ge \fl{v}_{j_{\max}}\Big(1-\frac{(\beta+1)}{a}\Big)$. Due to our choice of $a$, we get that $\fl{v}_j \ge \frac{2\fl{v}_{j_{\max}}}{3}$. Therefore $\|\fl{v}_{\ca{T}_u^{(\ell+1)}}\|_{\infty} \ge \frac{2\fl{v}_{j_{\max}}}{3}$. Since we construct the sets $\ca{N}^{(\ell+1,1)}$ and  $\ca{N}^{(\ell+1,2)}$ by taking intersections of the sets $\ca{T}_u^{(\ell+1)}$ of items for users in $\ca{M}^{(\ell,1)}$ and $\ca{M}^{(\ell,2)}$ respectively, one of these sets (say $\ca{N}^{(\ell+1,1)}$) must correspond to users with positive embedding while the other will correspond to users with a negative embedding (see Lemma \ref{lem:fourth}). Hence, we will have
 $\|\fl{v}_{\ca{N}^{(\ell+1,1)}}\|_{2} \ge \sqrt{|\ca{N}^{(\ell+1,1)}|}\frac{2|\fl{v}_{j_{\max}}|}{3}$ and $\|\fl{v}_{\ca{N}^{(\ell+1,1)}}\|_{\infty} \le \max(\beta,\beta^{-1})|\fl{v}_{j_{\max}}|$ thus completing the proof.
 By an analogous argument, the conclusion of the lemma is also true for users $u\in [\s{M}]$ with a negative embedding i.e.  $\fl{u}_u<0$. 
\end{proof}

\paragraph{Estimation of matrix $\fl{P}_{\ca{B}^{(\ell)},\s{N}}$ in Step 3 of Algorithm \ref{algo:phased_elim} in $\ell^{\s{th}}$ iteration:} We will only provide guarantees on estimation of the matrix $\fl{P}_{\ca{B}^{(\ell)},\s{N}}$ under the condition that the number of users in $\ca{B}^{(\ell)}$ is at least $\s{M}^{1-(2\zeta)^{-1}}$ (recall that $\s{T}^{\zeta}=\s{M}$).
Suppose the SVD of the matrix $\fl{P}$ is $\lambda \fl{\bar{u}}\fl{\bar{v}}$
satisfying $\|\fl{\bar{u}}\|=1, \|\fl{\bar{v}}\|=1$. In that case, SVD of the matrix $\fl{P}_{\ca{B}^{(\ell)},\s{N}}$ is $\lambda \|\fl{\bar{u}}_{\ca{B}^{(\ell)}}\|_2 \frac{\fl{\bar{u}}_{\ca{B}^{(\ell)}}}{\|\fl{\bar{u}}_{\ca{B}^{(\ell)}}\|_2}\fl{\bar{v}}$. For brevity, we will write $\widetilde{\fl{u}}=\frac{\fl{\bar{u}}_{\ca{B}^{(\ell)}}}{\|\fl{\bar{u}}_{\ca{B}^{(\ell)}}\|_2}$.
From our local incoherence assumption (we know that $\fl{\bar{u}}$ is $(\s{T}^{-1/2},\mu)$-locally incoherent), we know that 
\begin{align*}
    \|\widetilde{\fl{u}}\|_2 = 1, \; \|\widetilde{\fl{u}}\|_{\infty} \le \sqrt{\frac{\mu}{|\ca{B}^{(\ell)}|}} \text{ and }  \|\fl{\bar{v}}\|_{\infty} \le \sqrt{\frac{\mu}{\s{N}}}.
\end{align*}

By directly using the result in Lemma \ref{lem:min_acc} for the $r=1$ setting (rank $1$) restricted to users in $\ca{B}^{(\ell)}$ and items in $[\s{N}]$, we can recover an estimate $\widetilde{\fl{Q}}^{(\ell)}_{\ca{B}^{(\ell)},\s{N}}$ of $\fl{P}_{\ca{B}^{(\ell)},\s{N}}$ using $m_{\ell}=s_{\ell}\log (\s{MN}\delta^{-1}) (\s{N}p_{\ell}+\sqrt{\s{N}p_{\ell}\log \s{M}\delta^{-1}}) = O(s_{\ell}\log (\s{MN}\delta^{-1}) (\s{N}p_{\ell}\sqrt{\log \s{M}\delta^{-1}}))$ rounds (where $s_{\ell}$ and $p_{\ell}$ corresponds to the parameters $s,p$ in Lemma \ref{lem:min_acc}) such that 
\begin{align*}
    \|\widetilde{\fl{Q}}^{(\ell)}_{\ca{B}^{(\ell)},\s{N}}-\fl{P}_{\ca{B}^{(\ell)},\s{N}}\|_{\infty} \le O\Big(\frac{\sigma}{\sqrt{s_{\ell}d_2'}}\sqrt{\frac{\mu^3\log d_2'}{p_{\ell}}}\Big)
\end{align*}
with probability $1-O(\delta\log (\s{MN}\delta^{-1}))$ where $ d_2'=\min(\s{M}^{1-(2\zeta)^{-1}},\s{N})$. In other words, by denoting $q_{\ell}=p_{\ell}s_{\ell}$, we have that by using $m_{\ell}= O(q_{\ell}\s{N}\log ^2({\s{MNT}})))$
rounds, we get
\begin{align}\label{eq:unclustered}
    \|\widetilde{\fl{Q}}^{(\ell)}_{\ca{B}^{(\ell)},\s{N}}-\fl{P}_{\ca{B}^{(\ell)},\s{N}}\|_{\infty} \le O\Big(\frac{\sigma\sqrt{\mu^3 \log d_2'}}{\sqrt{q_{\ell}d_2'}}\Big)
\end{align}
with probability at least $1-\s{T}^{-4}$. Hence, if we need an estimate $\widetilde{\fl{Q}}^{(\ell)}_{\ca{B}^{(\ell)},\s{N}}$ such that $\|\widetilde{\fl{Q}}^{(\ell)}_{\ca{B}^{(\ell)},\s{N}}-\fl{P}_{\ca{B}^{(\ell)},\s{N}}\|_{\infty} \le \Delta_{\ell}$ for some input parameter $\Delta_{\ell}$, then $\frac{\sigma\sqrt{\mu^3 \log d_2'}}{\sqrt{q_{\ell}d_2'}} \le c'\Delta_{\ell}$ for some appropriate constant $c'$ implying that $q_{\ell}=\frac{\sigma^2\mu^3\log d_2'}{c'^2d_2'\Delta_{\ell}^2}$ and therefore, the total number of rounds is $O\Big(\max(1,\frac{\s{N}}{\s{M}^{1-(2\zeta)^{-1}}})\frac{\sigma^2\mu^3\log^2(\s{MNT})}{\Delta_{\ell}^2}\Big) $ with probability $1-\s{T}^{-4}$. 
By using the result in Lemma \ref{lem:min_acc}, we will need to use $s_{\ell},p_{\ell}$ such that $p_{\ell} \ge C\mu^2(d'_2)^{-1}\log ^3d'_2$ and $\frac{\sigma}{\sqrt{s_{\ell}}}=O\Big(\sqrt{\frac{p_{\ell}d'_2}{\mu^3\log d'_2}}\|\fl{P}\|_{\infty}\Big)$. The latter condition implies that $q_{\ell}=s_{\ell}p_{\ell}=\frac{\sigma^2\mu^3\log d_2'}{c_2^2d_2'\|\fl{P}\|_{\infty}^2}$ for some constant $c_2$. Putting everything together, we have that a sufficient value of $q_{\ell}$ is $\frac{\sigma^2\mu^3\log d_2'}{c_2^2d_2'\|\fl{P}\|_{\infty}^2}$ provided $\Delta_{\ell}\le \|\fl{P}\|_{\infty}$.

\paragraph{Estimation of matrix $\fl{P}_{\ca{M}^{(\ell,1)}, \ca{N}^{(\ell,1)}}$ (and similarly $\fl{P}_{\ca{M}^{(\ell,2)}, \ca{N}^{(\ell,2)}}$) in $\ell^{\s{th}}$ iteration:} In this case again, we will show guarantees only when $\left|\ca{M}^{(\ell,1)}\right| \ge \s{M}^{1-(2\zeta)^{-1}}$ (recall that $\s{M}^{\zeta}=\s{T}$). There are two sub-cases in this setting: 

\begin{enumerate}
    \item ($\s{N}\ge |\ca{N}^{(\ell,1)}|\ge \s{M}^{1-(2\zeta)^{-1}}$). In this case, as in the analysis for estimating $\fl{P}_{\ca{B}^{(\ell)},\s{N}}$, we can find an estimate $\widetilde{\fl{P}}^{(\ell,1)}\in \bb{R}^{\s{M}\times \s{N}}$ such that $\|\widetilde{\fl{P}}_{\ca{M}^{(\ell,1)}, \ca{N}^{(\ell,1)}}^{(\ell,1)}-\fl{P}_{\ca{M}^{(\ell,1)}, \ca{N}^{(\ell,1)}}\|\le \Delta_{\ell}$ for an input parameter $\Delta_{\ell}$ using 
    $O\Big(\max(1,\frac{\s{N}}{\s{M}^{1-(2\zeta)^{-1}}})\frac{\sigma^2\mu^3\log^2 (\s{MNT)}}{\Delta^2_{\ell}}\Big) $ rounds with probability $1-\s{T}^{-4}$ (by using Lemma \ref{lem:min_acc} for users in $\ca{M}^{(\ell,1)}$ and items in $\ca{N}^{(\ell,1)})$. In the above analysis, we used the fact that after clustering (see Lemma \ref{lem:incoherence}), we have $|\sqrt{\ca{N}^{(\ell,1)}}|\|\fl{v}_{\ca{N}^{(\ell,1)}}\|_{\infty}=O(\|\fl{v}_{\ca{N}^{(\ell,1})}\|_{2})$. In addition, we also used the fact that the $\fl{u}$ is $(\s{T}^{-1/2},\mu)$-locally incoherent.

    \item ($|\ca{N}^{(\ell,1)}|\le \s{M}^{1-(2\zeta)^{-1}}$) This is the more interesting case. Again, let us write the SVD of $\fl{P}_{\ca{M}^{(\ell,1)}, \ca{N}^{(\ell,1)}}=\widetilde{\lambda}\widetilde{\fl{u}}\widetilde{\fl{v}}$ with $\widetilde{\fl{u}}\in \bb{R}^{|\ca{M}^{(\ell,1)}|}, \widetilde{\fl{v}}\in \bb{R}^{|\ca{N}^{(\ell,1)}|}$. We know that $\widetilde{\fl{u}}$ is $\mu$-incoherent (since $|\ca{M}^{(\ell,1)}|\ge \s{M}^{1-(2\zeta)^{-1}}$) and moreover $\widetilde{\fl{v}}$ is $O(1)$-incoherent (see lemma \ref{lem:incoherence}). Again, by directly using the result in Lemma \ref{lem:min_acc} for the $r=1$ setting (rank $1$) for users in $\ca{M}^{(\ell,1)}$ and items in $\ca{N}^{(\ell,1)})$, we can recover an estimate $\widetilde{\fl{P}}^{(\ell,1)}_{\ca{M}^{(\ell,1)}, \ca{N}^{(\ell,1)}}$ of $\fl{P}_{\ca{M}^{(\ell,1)}, \ca{N}^{(\ell,1)}}$ (where $\widetilde{\fl{P}}^{(\ell,1)}\in \bb{R}^{\s{M}\times \s{N}}$) using $m_{\ell}= O(s_{\ell}\log (\s{MN}\delta^{-1}) (|\ca{N}^{(\ell,1)}|p_{\ell}\sqrt{\log \s{M}\delta^{-1}}))$ rounds (where $s_{\ell}$ and $p_{\ell}$ corresponds to the parameters $s,p$ in Lemma \ref{lem:min_acc}) such that 
\begin{align*}
    \|\widetilde{\fl{P}}^{(\ell,1)}_{\ca{M}^{(\ell,1)}, \ca{N}^{(\ell,1)}}-\fl{P}_{\ca{M}^{(\ell,1)}, \ca{N}^{(\ell,1)}}\|_{\infty} \le O\Big(\frac{\sigma}{\sqrt{s_{\ell}|\ca{N}^{(\ell,1)}|}}\sqrt{\frac{\mu^3\log |\ca{N}^{(\ell,1)}|}{p_{\ell}}}\Big)
\end{align*}
with probability $1-O(\delta\log \delta^{-1})$. As before, let us denote $q_{\ell}=p_{\ell}s_{\ell}$. We have that by using $m_{\ell}= O(q_{\ell}|\ca{N}^{(\ell,1)}|\log ^2\s{MN}\s{T}))$
rounds, we get
\begin{align}\label{eq:unclustered2}
    \|\widetilde{\fl{P}}^{(\ell,1)}_{\ca{M}^{(\ell,1)}, \ca{N}^{(\ell,1)}}-\fl{P}_{\ca{M}^{(\ell,1)}, \ca{N}^{(\ell,1)}}\|_{\infty} \le O\Big(\frac{\sigma\sqrt{\mu^3 \log |\ca{N}^{(\ell,1)}|}}{\sqrt{q_{\ell}|\ca{N}^{(\ell,1)}|}}\Big)
\end{align}
with probability at least $1-\s{T}^{-4}$. Hence, if we need an estimate $\widetilde{\fl{P}}^{(\ell,1)}_{\ca{M}^{(\ell,1)}, \ca{N}^{(\ell,1)}}$ such that $\|\widetilde{\fl{P}}^{(\ell,1)}_{\ca{M}^{(\ell,1)}, \ca{N}^{(\ell,1)}}-\fl{P}_{\ca{M}^{(\ell,1)}, \ca{N}^{(\ell,1)}}\|_{\infty} \le \Delta_{\ell}$ for some input parameter $\Delta_{\ell}$, then $\frac{\sigma\sqrt{\mu^3 \log |\ca{N}^{(\ell,1)}|}}{\sqrt{q_{\ell}|\ca{N}^{(\ell,1)}|}} \le c'\Delta_{\ell}$ for some appropriate constant $c'$ implying that a sufficient value of $q_{\ell}=\frac{\sigma^2\mu^3\log \s{N}}{c'^2|\ca{N}^{(\ell,1)}|\Delta_{\ell}^2}$ and therefore, the total number of rounds is $O\Big(\frac{\sigma^2\mu^3\log\s{N}\log^2 (\s{MNT)}}{\Delta_{\ell}^2}\Big) $ with probability $1-\s{T}^{-4}$. 
Again, by using the result in Lemma \ref{lem:min_acc}, we will need to use $s_{\ell},p_{\ell}$ such that $p_{\ell} \ge C\mu^2(|\ca{N}^{(\ell,1)}|)^{-1}\log ^3|\ca{N}^{(\ell,1)}|$ and $\frac{\sigma}{\sqrt{s_{\ell}}}=O\Big(\sqrt{\frac{p_{\ell}|\ca{N}^{(\ell,1)}|}{\mu^3\log |\ca{N}^{(\ell,1)}|}}\|\fl{P}\|_{\infty}\Big)$. Again, the latter condition implies that $q_{\ell}=s_{\ell}p_{\ell}=\frac{\sigma^2\mu^3\log |\ca{N}^{(\ell,1)}|}{c_2^2|\ca{N}^{(\ell,1)}|\|\fl{P}\|_{\infty}^2}$ for some constant $c_2$. Putting everything together, we have that a sufficient value of $q_{\ell}$ is $\frac{\sigma^2\mu^3\log |\ca{N}^{(\ell,1)}|}{c_2^2|\ca{N}^{(\ell,1)}|\Delta_{\ell}^2}$ provided $\Delta_{\ell}\le \|\fl{P}\|_{\infty}$.

\begin{rmk}\label{rmk:edge_case}
There is an edge case when $\left|\ca{N}^{(\ell,1)}\right|$ becomes so small that the condition $\left|\ca{N}^{(\ell,1)}\right|=\Omega(\mu \log \left|\ca{N}^{(\ell,1)}\right|)$ is not satisfied. Hence, we must have $\left|\ca{N}^{(\ell,1)}\right|=O(\mu)$.
If this becomes true, then in theory, a condition in Lemma \ref{lem:min_acc} becomes violated and low rank matrix completion will not work. However, an easy fix in this situation is to implement the standard Upper-Confidence Bound algorithm for each user in $\ca{M}^{(\ell,1)}$ separately until the end of rounds (from the phase $\ell$ when the condition becomes false) with the set of items $\ca{N}^{(\ell,1)}$ (and also for users in $\ca{B}^{(\ell)}$ that are inserted into $\ca{M}^{(\ell,1)}$ in subsequent phases). From Remark \ref{rmk:begin}, this step will only add a regret of $O(\sigma\sqrt{\left|\ca{N}^{(\ell,1)}\right|\s{T}\log \s{T}})=O(\sigma\sqrt{\mu\s{T}\log\s{T}})$ which is a significantly lower order term than the regret guarantee proved in Theorem \ref{thm:main1}.
\end{rmk}

Again, by using the result in Lemma \ref{lem:min_acc}, we will use $s_{\ell}$ to be the minimum positive integer satisfying $\frac{\sigma}{\sqrt{s_{\ell}}}=O\Big(\Delta_{\ell}\sqrt{\frac{|\ca{N}^{(\ell,1)}|}{\mu^3 \log |\ca{N}^{(\ell,1)}|}}\Big)$ and $\sqrt{p_{\ell}}=\frac{\sigma r}{\sqrt{s_{\ell} |\ca{N}^{(\ell,1)}|}}\frac{\sqrt{\mu^3 \log |\ca{N}^{(\ell,1)}|}}{c\Delta_{\ell}}$ for some appropriate constant $c>0$. 

\end{enumerate}

Now, we are ready to state and prove the main theorems:

\begin{thmu}[Restatement of Theorem \ref{thm:main1}]
Consider the rank-$1$ online matrix completion problem with $\s{T}$ rounds, $\s{M}$ users s.t. $\s{M}\ge \sqrt{\s{T}}$ and $\s{N}$ items. Denote $d_2=\min(\s{M},\s{N})$.
Let $\fl{R}^{(t)}_{u\rho_u(t)}$ be the reward in each round, defined as in \eqref{eq:obs}. Let $\sigma^2$ be the noise variance in rewards and let $\fl{P}\in \bb{R}^{\s{M}\times \s{N}}$ be the expected reward matrix with SVD decomposition $\fl{P}=\lambda\fl{\bar{u}}\fl{\bar{v}}^{\s{T}}$ such that $\fl{\bar{u}}$ is $(\s{T}^{-1/2},\mu)$-locally incoherent, $ \|\fl{\bar{v}}\|_{\infty}\le \sqrt{\mu/\s{N}}$, $d_2=\Omega(\mu \log d_2)$  
  and $|\fl{\bar{v}}_{j_{\min}}|=\Theta(|\fl{\bar{v}}_{j_{\max}}|)$.  Then, by suitably choosing parameters $\{\Delta_{\ell}\}_{\ell}$, positive integers $\{s_{(\ell,0)},s_{(\ell,1)},s_{(\ell,2)}\}_{\ell}$ and $1 \ge \{p_{(\ell,0)},p_{(\ell,1)},p_{(\ell,2)}\}_{\ell} \ge 0$ as described in Algorithm \ref{algo:phased_elim}, we can ensure a regret guarantee of $\s{Reg}(\s{T})=O(\sqrt{\s{T}}\|\fl{P}\|_{\infty}+\s{J}\sqrt{\s{TV}})$ where $\s{J}=O\Big(\log \Big(\frac{1}{\sqrt{\s{VT}^{-1}}}\min\Big(\|\fl{P}\|_{\infty},\frac{\sigma\sqrt{\mu}}{\log \s{N}}\Big)\Big)\Big)$ and $\s{V}=\Big(\max(1,\frac{\s{N}\sqrt{\s{T}}}{\s{M}})\sigma^2\mu^3\log^2(\s{MNT})\Big)$.
\end{thmu}

\begin{proof}[Proof of Theorem~\ref{thm:main1}]

 In the $\ell^{\s{th}}$ iteration, we set $\Delta_{\ell}=C'2^{-\ell}\min\Big(\|\fl{P}\|_{\infty},\frac{\sigma\sqrt{\mu}}{\log \s{N}}\Big)$ for some appropriate constant $C'>0$. In the $\ell^{\s{th}}$ iteration, our goal is to obtain estimates of the matrices $\fl{P}_{\ca{B}^{(\ell)},\s{N}}$ (using Lemma \ref{lem:min_acc} with parameters $(s_{(\ell,0)},p_{(\ell,0)})$ corresponding to $s,p$ in Lemma \ref{lem:min_acc}), $\fl{P}_{\ca{M}^{(\ell,1)},\ca{N}^{(\ell,1)}}$ (using Lemma \ref{lem:min_acc} with parameters $(s_{(\ell,1)},p_{(\ell,1)})$) and $\fl{P}_{\ca{M}^{(\ell,2)},\ca{N}^{(\ell,2)}}$ (using Lemma \ref{lem:min_acc} with parameters $(s_{(\ell,2)},p_{(\ell,2)})$ up to an error of $\Delta_{\ell}/2$ with high probability. 
 
 From Lemma \ref{lem:min_acc}, for any sub-matrix $\fl{P}_{\s{sub}}\in \bb{R}^{\s{M}'\times \s{N}'}$ of $\fl{P}_{\s{sub}}$, we can compute an estimate $\widehat{\fl{P}}_{\s{sub}}$ satisfying $\|\widehat{\fl{P}}_{\s{sub}}-\fl{P}_{\s{sub}}\|_{\infty} \le \Delta_{\ell}/2 $ with high probability by setting the parameters $s_{\ell},p_{\ell}$ such that 
\begin{align}
     \frac{\sigma r}{\sqrt{s_{\ell}d'_2}}\sqrt{\frac{\mu^3\log d'_2}{p_{\ell}}} = c\Delta_{\ell}
\end{align}
for some constant $c>0$ where $d'_2=\min(\s{M}',\s{N}')$. The largest possible error $\Delta_{\ell}$ that is possible to obtain by using the matrix completion technique is obtained by substituting $s=1$ and $p=C\mu^2(d'_2)^{-1}\log^3 d'_2$; we obtain that $\Delta_{\ell} \le \frac{\sigma r\sqrt{\mu}}{\log d'_2}$. Note that $d'_2 \le \s{N}$ (in both cases when $\s{M} \ge \s{N}$ and vice-versa) and therefore, if we choose $\Delta_{\ell} \le C'\min\Big(\frac{\sigma\sqrt{\mu}}{\log \s{N}}\Big)$ for some appropriate constant $C'>0$, then there exists $s_{\ell}$ and $p_{\ell}$ for which we can obtain $\|\widehat{\fl{P}}_{\s{sub}}-\fl{P}_{\s{sub}}\|_{\infty} \le \Delta_{\ell}/2 $ with high probability. Moreover, if $\Delta_{\ell}\le \|\fl{P}\|_{\infty}$, we can also ensure that $\frac{\sigma}{\sqrt{s_{\ell}}}=O\Big(\sqrt{\frac{p_{\ell}|\ca{N}^{(\ell,1)}|}{\mu^3\log |\ca{N}^{(\ell,1)}|}}\|\fl{P}\|_{\infty}\Big)$. Hence, we can ignore any lower bounds on $p_{\ell}$ and upper bounds on $\sigma$ as they are automatically satisfied due to the instantiation of $\Delta_{\ell}$.

Also from the analysis above (see equations \ref{eq:unclustered} and \ref{eq:unclustered2}), we know that by using a total number of rounds in the $\ell^{\s{th}}$ phase that is bounded from above by $O\Big(\max(1,\frac{\s{N}}{\s{M}^{1-(2\zeta)^{-1}}})\frac{\sigma^2\mu^3\log \s{N}\log^2(\s{MNT})}{\Delta_{\ell}^2}\Big)$ (see the analysis for the sufficient number of rounds for estimating $\fl{P}_{\ca{B}^{(\ell)},\s{N}}$ in the $\ell^{\s{th}}$ iteration)
, we have  with probability $1-O(\s{T}^{-4})$ (the event $\ca{E}^{(\ell)}_1$ for the $\ell^{\s{th}}$ iteration)
\begin{align}
    &\left|\left|\widetilde{\fl{Q}}^{(\ell)}_{\ca{B}^{(\ell)},[\s{N}]}-\fl{P}_{\ca{B}^{(\ell)},[\s{N}]}\right|\right|_{\infty} \le \frac{\Delta_{\ell}}{2}  \label{eq:ref11} \\
    &\left|\left|\widetilde{\fl{P}}^{(\ell,1)}_{\ca{M}^{(\ell,1)},\ca{N}^{(\ell,1)}}-\fl{P}_{\ca{M}^{(\ell,1)},\ca{N}^{(\ell,1)}}\right|\right|_{\infty}  \le \frac{\Delta_{\ell}}{2}  \label{eq:ref12}\\
    &\left|\left|\widetilde{\fl{P}}^{(\ell,2)}_{\ca{M}^{(\ell,2)},\ca{N}^{(\ell,2)}}-\fl{P}_{\ca{M}^{(\ell,2)},\ca{N}^{(\ell,2)}}\right|\right|_{\infty}  \le \frac{\Delta_{\ell}}{2}. \label{eq:ref13}
\end{align}
We condition on the events $\ca{E}_1^{(\ell)}$ being true for all $\ell$. The probability that there exists any $\ell$ such that the event $\ca{E}_1^{(\ell)}$ is false is $O(\s{T}^{-4})$; hence the probability that $\ca{E}_1^{(\ell)}$ is true for all $\ell$ is at least $1-O(\s{T}^{-3}$ ( the total number of iterations can be at most $\s{T}$). In case, one such event $\ca{E}_1^{(\ell)}$ is false, we can bound the regret trivially by $O(\s{T}^{-3}\|\fl{P}\|_{\infty})$. By an inductive argument (See Lemma \ref{lem1:third} and the first part of Corollary \ref{coro:crucial}), we know that the events $\ca{E}_2^{(\ell)}$ will be true for all $\ell$. 
  We will also denote the set of rounds in phase $\ell$ by $\ca{T}_{\ell}\subseteq [\s{T}]$ (therefore $\left|\ca{T}_{\ell}\right|=m_{\ell}$). Let us compute the  regret restricted to the rounds in $\ca{T}^{(\ell)}$ conditioned on the events $\ca{E}_1^{(\ell)},\ca{E}_2^{(\ell)}$ being true for all $\ell$ as follows:
\begin{align*}
   &\frac{m_{\ell}}{\s{M}}\sum_{u \in [\s{M}]}\f{\mu}_u^{\star}- \sum_{t\in \ca{T}^{(\ell)}}\frac{1}{\s{M}}\sum_{u\in[\s{M}]}\mathbf{P}_{u\rho_u(t)}^{(t)} \\
   &=\underbrace{\frac{m_{\ell}}{\s{M}}\sum_{u \in [\s{M}]: u \text{ not ignored}}\f{\mu}_u^{\star}- \sum_{t\in \ca{T}^{(\ell)}}\frac{1}{\s{M}}\sum_{u\in[\s{M}]: u \text{ not ignored}}\mathbf{P}_{u\rho_u(t)}^{(t)}}_{\text{ Term 1}}\\
   &+\underbrace{\frac{m_{\ell}}{\s{M}}\sum_{u \in [\s{M}]: u \text{ ignored}}\f{\mu}_u^{\star}- \sum_{t\in \ca{T}^{(\ell)}}\frac{1}{\s{M}}\sum_{u\in[\s{M}]: u \text{ ignored}}\mathbf{P}_{u\rho_u(t)}^{(t)}}_{\text{Term 2}} \\
\end{align*}  
The second term can be bounded from above by $O(m_{\ell}\s{M}^{1-(2\zeta)^{-1}}\|\fl{P}\|_{\infty}/\s{M})=O(m_{\ell}\s{M}^{-(2\zeta)^{-1}}\|\fl{P}\|_{\infty})$ (see second part of Corollary \ref{coro:crucial}).The first term can be bounded by using Corollary \ref{coro:crucial} conditioned on the events $\ca{E}_1^{(\ell)},\ca{E}_2^{(\ell)}$ being true for all $\ell$,  
\begin{align}
 (2a+1)\Delta_{\ell-1}m_{\ell}=O\Big(\max(1,\frac{\s{N}}{\s{M}^{1-(2\zeta)^{-1}}})\frac{\sigma^2\mu^3\log\s{N}\log^2(\s{MNT})}{\Delta_{\ell}}\Big).   
\end{align}

Putting everything together, we can now bound the total regret as follows (using $\Delta_{\ell-1}=2\Delta_{\ell}$):
 \begin{align*}
     &\s{Reg}(\s{T})=\frac{\s{T}}{\s{M}}\sum_{u\in [\s{M}]}(\mu_u^{\star}-\sum_{t\in [\s{T}]}\bb{E}\fl{R}_{u\rho_u(t)}^{(\s{t})}) \\
     &=\sum_{\ell}  O\Big(\Delta_{\ell}m_{\ell}\mid \ca{E}^{(\ell)}_1,\ca{E}^{(\ell)}_2 \text{ is true for all } \ell\Big)+O(\s{M}^{-(2\zeta)^{-1}}\|\fl{P}\|_{\infty}m_{\ell}\mid  \ca{E}^{(\ell)}_1,\ca{E}^{(\ell)}_2 \text{ is true for all } \ell)+O(\s{T}^{-3}\|\fl{P}_{\infty}\|) \\
     &\le  \sum_{\ell}  O\Big(\Delta_{\ell}m_{\ell}\mid \ca{E}^{(\ell)}_1,\ca{E}^{(\ell)}_2 \text{ is true for all } \ell\Big)+O(\sqrt{\s{T}}\|\fl{P}\|_{\infty}) 
 \end{align*}
 where we used the fact that $\sum_{\ell}m_{\ell} = \s{T}$ and $\s{T}^{\zeta}=\s{M}$. As stated in the theorem, let $\s{V}=\Big(\max(1,\frac{\s{N}}{\s{M}^{1-(2\zeta)^{-1}}})\sigma^2\mu^3\log\s{N}\log^2(\s{MNT})$.
 Moving on, we can decompose the first term into two parts: 
\begin{align*}
     &\s{Reg}(\s{T})\le O(\sqrt{\s{T}}\|\fl{P}\|_{\infty})+O\Big(\sum_{\ell: \Delta_{\ell}\le\Phi} \Delta_{\ell}m_{\ell}\mid \ca{E}^{(\ell)}_1,\ca{E}^{(\ell)}_2 \text{ is true for all } \ell\Big)\\
     &+O\Big(\sum_{\ell: \Delta_{\ell}>\Phi} \Delta_{\ell}\s{V}\Delta_{\ell}^{-2}\mid \ca{E}^{(\ell)}_1,\ca{E}^{(\ell)}_2 \text{ is true for all } \ell\Big) \\
     &\le O(\sqrt{\s{T}}\|\fl{P}\|_{\infty}) +\s{T}\Phi+O\Big(\sum_{\ell: \Delta_{\ell}>\Phi} \s{V}\Delta_{\ell}^{-1}\Big) \\
\end{align*}
Since we chose $\Delta_{\ell}=C'2^{-\ell}\min\Big(\|\fl{P}\|_{\infty},\frac{\sigma\sqrt{\mu}}{\log \s{N}}\Big)$ for some constant $C'>0$, the maximum number of phases $\ell$ for which $\Delta_{\ell}>\Phi$ can be bounded from above by $\s{J}=O\Big(\log \Big(\frac{1}{\Phi}\min\Big(\|\fl{P}\|_{\infty},\frac{\sigma\sqrt{\mu}}{\log \s{N}}\Big)\Big)\Big)$. Hence, we have
\begin{align*}
    \s{Reg}(\s{T}) &\le O(\sqrt{\s{T}}\|\fl{P}\|_{\infty})+O(\s{T}\Phi)+O\Big(\s{JV}\Phi^{-1}\Big) \\
     &= O(\sqrt{\s{T}}\|\fl{P}\|_{\infty})+O(\s{J}\sqrt{\s{TV}})
\end{align*}
where we substituted $\Phi=\sqrt{\s{V}\s{T}^{-1}}$ and hence $\s{J}=O\Big(\log \Big(\frac{1}{\sqrt{\s{VT}^{-1}}}\min\Big(\|\fl{P}\|_{\infty},\frac{\sigma\sqrt{\mu}}{\log \s{N}}\Big)\Big)\Big)$ in the final step. 
\end{proof}

\begin{thmu}[Restatement of Theorem \ref{thm:lb}]
Let $\fl{P}$ be a rank $1$ matrix such that $0 \le \fl{P}_{ij} \le 1$ for all $i,j\in [\s{M}]\times [\s{N}]$ and the noise variance $\sigma^2=1$. In that case, for problem P1, any algorithm will suffer a regret of at least $\Omega(\sqrt{\s{NT}\s{M}^{-1}})$.
\end{thmu}

\begin{proof}[Proof of Theorem \ref{thm:lb}]

Recall that the well-known multi-armed bandit (MAB) problem is a special case of our setting when the number of users $\s{M}=1$. We can cast the reward matrix as a $1\times \s{N}$ matrix $\fl{P}$ whose $i^{\s{th}}$ element is denoted by $\fl{P}_i$.
In that case, the expected regret $\s{Reg}_{\s{MAB}}(\s{T})$ over a time period of $\s{T}$ steps is $\s{T}\max_{j \in [\s{N}]}\fl{P}_{j}-\sum_{t\in [\s{T}]}\fl{P}_{\rho(t)}$ where $\rho(t)$ is the item recommended at time $t$. For the MAB problem, it is known that any algorithm must suffer a regret of at least $\Omega(\sqrt{\s{N T}})$ (see Theorem 15.2 in \cite{lattimore2020bandit}).
Consider our problem with $\s{M}$ users and $\s{N}$ items when all users $u \in [\s{M}]$ are identical and furthermore, in each round, the item recommended for the $i^{\s{th}}$ user can depend on the rewards obtained for users $1,2,\dots,i-1$ in that round and the rewards obtained in the previous rounds. In that case, any algorithm that achieves a regret $\s{Reg}(\s{T})$ in the repeated setting and will achieve a regret of $\s{M}^{-1}\s{Reg}_{\s{MAB}}(\s{MT})$ in the MAB problem and vice-versa. Hence, we have that $\s{Reg}(\s{T}) = \s{M}^{-1}\s{Reg}_{\s{MAB}}(\s{MT}) =\Omega(\sqrt{(\s{NT/M})})$.

\end{proof}

\section{Small number of users $\s{M}$ (Proof of Theorem \ref{thm:main2})}\label{app:repeated2}

\begin{algorithm*}[t]
\caption{\textsc{OCTAL Algorithm for small number of users $\s{M}$}   \label{algo:phased_elim2}}
\begin{algorithmic}[1]
\small
\REQUIRE Number of users $\s{M}$, items $\s{N}$, rounds $\s{T}$, noise $\sigma^2$, bound on the entry-wise magnitude of expected rewards $||\fl{P}||_{\infty}$, incoherence $\mu$.
\STATE  Set $\ca{M}^{(1,1)}=\ca{M}^{(1,2)}=\phi$ and $\ca{B}^{(1)}=[\s{M}]$. Set $\ca{N}^{(1,1)}=\ca{N}^{(1,2)}=\phi$. Set $f=O(\log (\s{MNT}))$ and suitable constants $a,c,C,C',C_{\lambda}>0$.

\FOR{$\ell=1,2,\dots,$}

\STATE Set $\Delta_{\ell}=C'2^{-\ell}\min\Big(\|\fl{P}\|_{\infty},\frac{\sigma\sqrt{\mu}}{\log \s{N}}\Big)$. 

\FOR{$k=1,2,\dots,f$}

\FOR{each pair of non-null sets  $(\ca{M}^{(\ell,1)}\cup \ca{B}^{(\ell)},\ca{N}^{(\ell,1)}),(\ca{M}^{(\ell,2)}\cup \ca{B}^{(\ell)},\ca{N}^{(\ell,2)})\subseteq [\s{M}]\times [\s{N}]$}

\STATE Denote $(\ca{T}^{(1)},\ca{T}^{(2)})$ to be the considered pair of sets and $i\in \{1,2\}$ to be its index.

\STATE Set $d_{2,i}=\min(|\ca{T}^{(1)}|,|\ca{T}^{(2)}|)$. Set $p_{\ell,i}=C\mu^2 d_{2,i}^{-1}\log^3 d_{2,i}$ and $s_{\ell,i}=\Big\lceil \Big(\frac{c\sigma  \sqrt{\mu}}{\Delta_{\ell}\log d_{2,i}}\Big)^2 \Big\rceil$.

\STATE For each tuple of indices $(u,v)\in \ca{T}^{(1)}\times \ca{T}^{(2)}$, independently set $\delta_{uv}=1$ with probability $p_{\ell,i}$ and $\delta_{uv}=0$ with probability $1-p_{\ell,i}$.

\STATE  Denote $\Omega^{(i)}=\{(u,v)\in \ca{T}^{(1)}\times \ca{T}^{(2)} \mid \delta_{uv}=1\}$ and
 $b_{\ell,i}=\max_{u \in \ca{U}} |v \in \ca{V}\mid (u,v) \in \Omega|$. Set total number of rounds to be $m_{\ell,i}=b_{\ell,i}s_{\ell,i}$.

\ENDFOR



\STATE For $i\in\{1,2\}$, compute $\widetilde{\fl{P}}^{(\ell,i,f)}=\textsc{Estimate}(m_{\ell,i},b_{\ell,i},\Omega^{(i)},|\ca{M}^{(\ell,i)}|,|\ca{N}^{(\ell,i)}|,\lambda=C_{\lambda}\sigma\sqrt{d_{2,i}p_{\ell,i}})$. \#
\textit{Algorithm \ref{algo:estimate} recommends items to every user in $\ca{M}^{(\ell,i)}\cup \ca{B}^{(\ell)}$ for $m_{\ell,i}$ rounds. Since users in $\ca{B}^{(\ell)}$ get recommended items for $m_{\ell,1}+m_{\ell,2}$ rounds, for $i\in \{1,2\}$ we can recommend arbitrary items in $\ca{N}^{(\ell,i)}$ for users in $\ca{M}^{(\ell,i)}$ for the additional $m_{(\ell,3-i)}$ rounds.}

\ENDFOR 

\STATE Compute  $\widetilde{\fl{P}}^{(\ell,i)}=$Entrywise Median$(\{\widetilde{\fl{P}}^{(\ell,i,k)}\}_{k=1}^{f})$ for $i\in \{1,2\}$.

\STATE Set $\ca{B}^{(\ell+1)}\subseteq \ca{B}^{(\ell)}$ to be the set 
\begin{align*}
   \Big\{u \in \ca{B}^{(\ell)}\mid \left|\max(\max_{t\in \ca{N}^{(\ell,1)}}\fl{\widetilde{P}}^{(\ell,1)}_{ut},\max_{t\in \ca{N}^{(\ell,2)}}\fl{\widetilde{P}}^{(\ell,2)}_{ut})
  -\min(\min_{t\in \ca{N}^{(\ell,1)}}\fl{\widetilde{P}}^{(\ell,1)}_{ut},\min_{t\in \ca{N}^{(\ell,2)}}\fl{\widetilde{P}}^{(\ell,2)}_{ut})\right|\le 2a\Delta_{\ell}\Big\}. 
\end{align*}

\STATE Set $\s{A}_u = \max(\max_{t\in \ca{N}^{(\ell,1)}}\fl{\widetilde{P}}^{(\ell,1)}_{ut},\max_{t\in \ca{N}^{(\ell,2)}}\fl{\widetilde{P}}^{(\ell,2)}_{ut})$. Compute $\ca{T}_{u}^{(\ell+1)}=\{j \in \ca{N}^{(\ell,1)}\}\mid \fl{\widetilde{P}}^{(\ell,1)}_{uj}+\Delta_{\ell}> \s{A}_u\}\cup \{j \in \ca{N}^{(\ell,2)}\}\mid \fl{\widetilde{P}}^{(\ell,2)}_{uj}+\Delta_{\ell}> \s{A}_u\}$ for all $u \in \ca{B}^{(\ell)}\setminus \ca{B}^{(\ell+1)}$.

\STATE For $i\in \{1,2\}$, for all users $u \in \ca{M}^{(\ell,i)}$, compute $\ca{T}_{u}^{(\ell+1)}=$ $\{j \in \ca{N}^{(\ell,i)}\mid \fl{\widetilde{P}}^{(\ell,i)}_{uj}+\Delta_{\ell}> \max_{t \in \ca{N}^{(\ell,i)}}\fl{\widetilde{P}}^{(\ell,i)}_{ut}\}$.


\STATE Set $v$ to be any user in $[\s{M}]\setminus \ca{B}^{(\ell+1)}$. Set $\ca{M}^{(\ell+1,1)}=\{u \in [\s{M}]\setminus \ca{B}^{(\ell+1)} \mid \ca{T}_{u}^{(\ell+1)}\cap \ca{T}_{v}^{(\ell+1)}\neq \phi\}$. Set $\ca{M}^{(\ell+1,2)}=[\s{M}]\setminus (\ca{B}^{(\ell+1)}\cup \ca{M}^{(\ell+1,1)})$.

\STATE Compute $\ca{N}^{(\ell+1,1)}=\bigcap_{u \in \ca{M}^{(\ell+1,1)}}\ca{T}_u^{(\ell+1)}$,  $\ \ \ \ca{N}^{(\ell+1,2)}=\bigcap_{u \in \ca{M}^{(\ell+1,2)}}\ca{T}_u^{(\ell+1)}$.



\ENDFOR

\end{algorithmic}
\end{algorithm*}

As in section \ref{app:repeated}, we will consider the reward matrix $\fl{P}=\lambda \fl{\bar{u}}\fl{\bar{v}}^{\s{T}}\in \bb{R}^{\s{M}\times \s{N}}$ to be a rank-$1$ matrix with $\|\fl{\bar{u}}\|_2=\|\fl{\bar{v}}\|_2=1$.
We make the following assumptions where we use the notations presented in Section \ref{app:repeated}:
\begin{enumerate}
    \item We assume that the vector $\bar{\fl{u}}$ is $(\min(\left|\ca{C}_1\right|,\left|\ca{C}_2\right|),\mu)$-incoherent. This is stricter than the first assumption presented in Section \ref{app:repeated}.
    \item $\beta=\max\Big(\left|\frac{\fl{\bar{v}}_{j_{\max}}}{\fl{\bar{v}}_{j_{\min}}}\right|,\left|\frac{\fl{\bar{v}}_{j_{\min}}}{\fl{\bar{v}}_{j_{\max}}}\right|\Big)$ for some positive constant $\beta>0$. Note that if we represent $\fl{P}=\fl{u}\fl{v}^{\s{T}}$ where $\fl{u}$ is the user embedding and $\fl{v}$ is the item embedding, then we have $\beta=\max\Big(\left|\frac{\fl{v}_{j_{\max}}}{\fl{v}_{j_{\min}}}\right|,\left|\frac{\fl{v}_{j_{\min}}}{\fl{v}_{j_{\max}}}\right|\Big)$.
\end{enumerate}

Here, we have a slight modification of the phased algorithm (see Algorithm \ref{algo:phased_elim2}).
In each phase $\ell$, we will maintain three groups of users and items $(\ca{B}^{(\ell)},\ca{N}^{(\ell,1)} \cup \ca{N}^{(\ell,2)})$, ($\ca{M}^{(\ell,1)},\ca{N}^{(\ell,1)}$)  and ($\ca{M}^{(\ell,2)}$,$\ca{N}^{(\ell,2)}$) such that $\ca{B}^{(\ell)}\cup \ca{M}^{(\ell,1)} \cup \ca{M}^{(\ell,2)} =[\s{M}]$ and $\ca{N}^{(\ell,1)},\ca{N}^{(\ell,2)} \subseteq [\s{N}]$. Here $\ca{B}^{(\ell)},\ca{M}^{(\ell,1)},\ca{M}^{(\ell,2)}$ are initialized by $[\s{M}],\phi,\phi$ respectively and 
$\ca{N}^{(\ell,1)},\ca{N}^{(\ell,2)}$ are initialized by $[\s{N}],\phi$ respectively.

Suppose we have a sequence of tuples $(m_{\ell},\Delta_{\ell})_{\ell}$ that is going to be characterized precisely later. In every phase indexed by $\ell$, we will compute two matrices $\widetilde{\fl{P}}^{(\ell,1)}, \widetilde{\fl{P}}^{(\ell,2)} \in \bb{R}^{\s{M}\times \s{N}}$ (using $m_{\ell}$ rounds for each) that correspond to  estimates of two relevant sub-matrices of $\fl{P}$ namely $\fl{P}_{\ca{M}^{(\ell,1)}\cup \ca{B}^{(\ell)},\ca{N}^{(\ell,1)}}$ and $\fl{P}_{\ca{M}^{(\ell,2)}\cup \ca{B}^{(\ell)},\ca{N}^{(\ell,2)}}$ respectively. We will define the event $\ca{E}_1^{(\ell)}$ when the following holds true 
\begin{align*}
    &\left|\left|\widetilde{\fl{P}}_{\ca{M}^{(\ell,1)}\cup \ca{B}^{(\ell)},\ca{N}^{(\ell,1)}}-\fl{P}_{\ca{M}^{(\ell,1)}\cup \ca{B}^{(\ell)},\ca{N}^{(\ell,1)}}\right|\right|_{\infty} \le \Delta_{\ell}/2 \\
    &\left|\left|\widetilde{\fl{P}}_{\ca{M}^{(\ell,2)}\cup \ca{B}^{(\ell)},\ca{N}^{(\ell,2)}}-\fl{P}_{\ca{M}^{(\ell,2)}\cup \ca{B}^{(\ell)},\ca{N}^{(\ell,2)}}\right|\right|_{\infty} \le \Delta_{\ell}/2
\end{align*}
by using a number of rounds $m_{\ell}$ that is bounded from above by $O\Big(\max(1,\frac{\s{N}}{\s{M}^{1-(2\zeta)^{-1}}})\frac{\sigma^2\mu^3\log^2(\s{MNT})}{\Delta_{\ell}}\Big)$.

Thus, the only difference with the analysis in Section \ref{app:repeated} is that for the users in $\ca{B}^{(\ell)}$, we consider the items in $\ca{N}^{(\ell,1)}\cup \ca{N}^{(\ell,2)}$ (which can be shown to comprise the best arms for both clusters); therefore we combine the users in $\ca{B}^{(\ell)}$ with users in $\ca{M}^{(\ell,1)}$ and users in $\ca{M}^{(\ell,2)}$ separately to obtain the matrix estimates. Since users in $\ca{B}^{(\ell)}$ need to recommended more items in each phases, we can recommend arbitrary items to users in $\ca{M}^{(\ell,1)}$ ($\ca{M}^{(\ell,1)}$) from $\ca{N}^{(\ell,1)}$ ($\ca{N}^{(\ell,1)}$ respectively.) 

At each phase indexed by $\ell$, we will compute a temporary set of active items $\ca{T}_u^{(\ell+1)}$ for all users $u\in [\s{N}]$. We update $\ca{B}^{(\ell+1)}$ as given in Step 14 of Algorithm \ref{algo:phased_elim2}. We update the sets of users $\ca{M}^{(\ell+1,1)}$, $\ca{M}^{(\ell+1,2)}$ as in Section \ref{app:repeated} (see Steps 17,18 in Algorithm \ref{algo:phased_elim2}). 
For any phase indexed by $\ell$, we define the event $\ca{E}_2^{(\ell)}$ such that for every user  $u\in  \ca{M}^{(\ell,1)}$, the following holds:
\begin{align*}
   & \left|\fl{P}_{ut}-\max_{t\in [\s{N}]}\fl{P}_{ut}\right| \le 2\Delta_{\ell-1} \text{  for all } t\in\ca{N}^{(\ell,1)} , u\in \ca{M}^{(\ell,1)}\\
    &\left|\fl{P}_{ut}-\max_{t\in [\s{N}]}\fl{P}_{ut}\right| \le 2\Delta_{\ell-1} \text{  for all } t\in\ca{N}^{(\ell,2)}, u\in \ca{M}^{(\ell,2)} \\
    &\left|\fl{P}_{ut}-\max_{t\in [\s{N}]}\fl{P}_{ut}\right| \le (2a+2)\Delta_{\ell-1} \text{  for all } t\in \ca{N}^{(\ell,1}\cup \ca{N}^{(\ell,2)}, u\in \ca{B}^{(\ell)}.
\end{align*}


As before, we define the events $\ca{E}_3^{(\ell)}$ which is true when
$\ca{M}^{(\ell,1)}\subseteq \ca{C}_1$, $j_{\max}\in \ca{N}^{(\ell,1)}$ and  $\ca{M}^{(\ell,2)}\subseteq \ca{C}_2$, $j_{\min}\in \ca{N}^{(\ell,2)}$. Finally, let  $\ca{E}_4^{(\ell)}$ be the event which is true when
$\ca{M}^{(\ell,2)}\subseteq \ca{C}_1$, $j_{\min}\in \ca{N}^{(\ell,1)}$ and  $\ca{M}^{(\ell,1)}\subseteq \ca{C}_2$, $j_{\max}\in \ca{N}^{(\ell,2)}$.
By a similar analysis as in Section \ref{app:repeated}, we can show that Lemmas \ref{lem1:first}, \ref{lem1:second}, \ref{lem1:third} and \ref{lem:fourth}, the first part of \ref{coro:crucial} (note that no users are \textit{ignored} in this analysis) and  \ref{lem:incoherence} all hold true with very minor modifications wherever required.

\paragraph{Estimation of matrix $\fl{P}_{\ca{M}^{(\ell,1)}\cup \ca{B}^{(\ell)}, \ca{N}^{(\ell,1)}}$ (and similarly $\fl{P}_{\ca{M}^{(\ell,2)}\cup \ca{B}^{(\ell)}, \ca{N}^{(\ell,2)}}$) in $\ell^{\s{th}}$ iteration:} In this case, notice that the size of $\left|\ca{M}^{(\ell,1)}\cup \ca{B}^{(\ell)}\right| \ge \min(|\ca{C}_1|,|\ca{C}_2|)$ since the combination of the two sets must comprise the users of the cluster $\ca{M}^{(\ell,1)}$ corresponds to. There are two sub-cases in this setting: 

\begin{enumerate}
    \item ($\s{N}\ge |\ca{N}^{(\ell,1)}|\ge \min(|\ca{C}_1|,|\ca{C}_2|)$). In this case, as in the analysis for estimating $\fl{P}_{\ca{B}^{(\ell)},\s{N}}$, we can find an estimate $\widetilde{\fl{P}}_{\ca{M}^{(\ell,1)}\cup \ca{B}^{(\ell)},\ca{N}^{(\ell,1)}}\in \bb{R}^{\s{M}\times \s{N}}$ such that $\left|\left|\widetilde{\fl{P}}_{\ca{M}^{(\ell,1)}\cup \ca{B}^{(\ell)},\ca{N}^{(\ell,1)}}-\fl{P}_{\ca{M}^{(\ell,1)}\cup \ca{B}^{(\ell)},\ca{N}^{(\ell,1)}}\right|\right|_{\infty} \le \Delta_{\ell}/2$ for an input parameter $\Delta_{\ell}$ using 
    $O\Big(\max(1,\frac{\s{N}}{\min(|\ca{C}_1|,|\ca{C}_2|)})\frac{\sigma^2\mu^3\log^2 (\s{MNT)}}{\Delta^2_{\ell}}\Big) $ rounds with probability $1-\s{T}^{-4}$ (by using Lemma \ref{lem:min_acc} for users in $\ca{M}^{(\ell,1)}$ and items in $\ca{N}^{(\ell,1)})$. In the above analysis, we used Lemma \ref{lem:incoherence} and the fact that the $\fl{u}$ is $(\s{M}^{-1}\min(|\ca{C}_1|,|\ca{C}_2|),\mu)$-locally incoherent.

    \item ($|\ca{N}^{(\ell,1)}|\le \min(|\ca{C}_1|,|\ca{C}_2|)$) 
    Again, by a similar analysis as in Section \ref{app:repeated} and by using Lemma \ref{lem:incoherence} and the fact that the $\fl{u}$ is $(\s{M}^{-1}\min(|\ca{C}_1|,|\ca{C}_2|),\mu)$-locally incoherent, the total number of sufficient rounds is  $O\Big(\frac{\sigma^2\mu^3\log\s{N}\log^2 (\s{MNT)}}{\Delta_{\ell}^2}\Big) $ with probability $1-\s{T}^{-4}$. 

\end{enumerate}

Now, we are ready to state and prove the main theorems:

\begin{thmu}\label{thm:main2}
Consider the rank-$1$ online matrix completion problem with $\s{M}$ users, $\s{N}$ items, $\s{T}$ recommendation rounds such that $\s{M}\ge \sqrt{\s{T}}$. Denote $d_2=\min(\s{M},\s{N})$.
Let $\fl{R}^{(t)}_{u\rho_u(t)}$ be the reward in each round, defined as in \eqref{eq:obs}. Let $\sigma^2$ be the noise variance in rewards and let $\fl{P}\in \bb{R}^{\s{M}\times \s{N}}$ be the expected reward matrix with SVD decomposition $\fl{P}=\lambda\fl{\bar{u}}\fl{\bar{v}}^{\s{T}}$ such that $\fl{\bar{u}}$ is $(\min(|\ca{C}_1|,|\ca{C}_2|),\mu)$-locally incoherent, $ \|\fl{\bar{v}}\|_{\infty}\le \sqrt{\mu/\s{N}}$, $d_2=\Omega(\mu\log d_2)$  
  and $|\fl{\bar{v}}_{j_{\min}}|=\Theta(|\fl{\bar{v}}_{j_{\max}}|)$.  Then, by suitably choosing parameters $\{\Delta_{\ell}\}_{\ell}$, positive integers $\{s_{(\ell,1)},s_{(\ell,2)}\}_{\ell}$ and $1 \ge \{p_{(\ell,1)},p_{(\ell,2)}\}_{\ell} \ge 0$ as described in Algorithm \ref{algo:phased_elim}, we can ensure a regret guarantee of $\s{Reg}(\s{T})=O(\sqrt{\s{T}}\|\fl{P}\|_{\infty}+\s{J}\sqrt{\s{TV}})$ where $\s{J}=O\Big(\log \Big(\frac{1}{\sqrt{\s{VT}^{-1}}}\min\Big(\|\fl{P}\|_{\infty},\frac{\sigma\sqrt{\mu}}{\log \s{N}}\Big)\Big)\Big)$ and $\s{V}=\Big(\max(1,\frac{\s{N}}{\min(|\ca{C}_1|,|\ca{C}_2|)})\sigma^2\mu^3\log^3(\s{MNT})\Big)$.
\end{thmu}

\begin{proof}[Proof of Theorem~\ref{thm:main2}]

The proof follows on similar lines as the proof in Theorem \ref{thm:main1}. The only change is in $m_{\ell}$ which is reflected in the result (more specifically in $\s{V}$). Also, note that in Algorithm \ref{algo:phased_elim2}, no users are \textit{ignored} and so we do not need to bound the regret for ignored users.

\end{proof}


\end{document}